\documentclass[1p]{elsarticle}

\usepackage[lined,vlined]{algorithm2e}
\usepackage{color}
\usepackage{graphicx}
\usepackage{psfrag}
\usepackage{verbatim}
\usepackage{multirow}
\usepackage{subfigure}
\usepackage{amsmath,amssymb}
\usepackage{amsthm}
\usepackage{float,multicol}
\usepackage{tweaklist}
\usepackage{url}

\usepackage{tikz}
\usepackage[utf8]{inputenc}
\usetikzlibrary{decorations,arrows,shapes}

\definecolor{darkgreen}{RGB}{0,128,0}

\pgfdeclarelayer{background}
\pgfsetlayers{background,main}

\renewcommand{\enumhook}{\setlength{\itemsep}{-1pt}}
\renewcommand{\itemhook}{\setlength{\itemsep}{-1pt}}

\SetKwFunction{Proj}{Project}
\SetKwFor{Procedure}{Procedure}{}{}%

\newcommand{\etal}{\textit{et al.\ }} 
\newcommand{\ie} {\textit{i.e.}\xspace } 
\newcommand{\gc}[1] {{\sc #1}} 
\newcommand{\opSty}[1] {\texttt{#1}}
\renewcommand{\FuncSty}[1]{\opSty{#1}}
\renewcommand{\ArgSty}[1]{\textrm{#1}}

\newcommand{\best}[1]{$^\dagger${#1}}
\newcommand{\procCall}[1]{\textit{\texttt{#1}} }

\newcommand{\cf}{W}
\newcommand{\smallcf}{\omega}

\newcommand{\CiteInLine}[1]{\citeauthor{#1} \citeyear{#1}}
\newcommand{\X}{\mathcal{X}}

\newcommand{\C}{\mathcal{W}}
\newcommand{\F}{\mathcal{F}}
\renewcommand{\L}{D}

\newcommand{\Q}{Q}

\newcommand{\ignore}[1]{}

\newtheorem{theorem}{Theorem}
\newtheorem{corollary}{Corollary}
\newtheorem{definition}{Definition}
\newtheorem{lemma}{Lemma}
\newtheorem{example}{Example}
\newtheorem{proposition}{Proposition}

\DeclareTextFontCommand{\textbfit}{%
  \fontseries\bfdefault 
  \itshape
}

\newcommand\myatop[2]{\genfrac{}{}{0pt}{}{#1}{#2}}

\journal{Artificial Intelligence}

\begin{document}

\begin{frontmatter}

\title{Tractability-preserving Transformations of Global Cost Functions (extended version) \tnoteref{thanks1}}
\tnotetext[thanks1]{This paper is an extended version of~\cite{Dec16}, DOI 	10.1016/j.artint.2016.06.005.}
 
\author[INRA]{David Allouche}
\author[LIRMM]{Christian Bessiere}
\author[GREYC]{Patrice Boizumault}
\author[INRA]{Simon de Givry}
\author[PG]{Patricia Gutierrez}
\author[CUHK]{Jimmy H.M.\ Lee\corref{cor1}}
\author[CUHK]{Ka Lun Leung}
\author[GREYC]{Samir Loudni}
\author[GREYC]{Jean-Philippe M\'etivier}
\author[INRA]{Thomas Schiex\corref{cor1}}
\author[CUHK]{Yi Wu}

\address[INRA]{MIAT, UR-875, INRA, F-31320 Castanet Tolosan, France}
\address[LIRMM]{CNRS, University of Montpellier, France}
\address[PG]{IIIA-CSIC, Universitat Autonoma de Barcelona, 08193 Bellaterra, Spain}
\address[CUHK]{Department of Computer Science and Engineering, The Chinese University of Hong Kong, Shatin, N.T., Hong Kong}
\address[GREYC]{GREYC, Universit\'e de Caen Basse-Normandie,6 Boulevard du Mar\'echal Juin, 14032 Caen cedex 5, France}
\cortext[cor1]{Corresponding author}


\begin{abstract}
Graphical model processing is a central problem in artificial
intelligence. The optimization of the combined cost of a network of local cost functions federates a variety of famous problems including CSP, SAT and Max-SAT but also optimization in stochastic variants such as Markov Random Fields and Bayesian networks. Exact solving methods for these problems typically include branch and bound and local inference-based bounds.
In this paper we are interested in understanding when and how dynamic programming based optimization can be used to efficiently enforce soft local consistencies on Global Cost Functions, defined as parameterized families of cost functions of unbounded arity. 
Enforcing local consistencies in cost function networks is performed by applying so-called Equivalence Preserving Transformations (EPTs) to the cost functions. These EPTs may transform global cost functions and make them intractable to optimize.
We identify as \emph{tractable projection-safe} those global cost functions whose optimization is and remains tractable after applying {the  EPTs used for enforcing arc consistency}. We also provide new classes of cost functions that are tractable projection-safe thanks to dynamic programming.
We show that dynamic programming can either be directly used inside filtering algorithms, defining polynomially {DAG-filterable} cost functions, or emulated by {arc consistency filtering on a Berge-acyclic network of bounded-arity cost functions}, defining Berge-acyclic network-decomposable cost functions. We give  examples of such cost functions and we provide a systematic way to define decompositions from existing decomposable global constraints.
%
These two approaches to {enforcing consistency in global cost functions} are then embedded in  a  solver for extensive experiments that confirm the feasibility and efficiency of our proposal.
\end{abstract}
\end{frontmatter}


\section{Introduction}

Cost Function Networks (CFNs) offer a simple and
general framework for modeling and solving over-constrained and
optimization problems.  They capture a variety of problems that range from CSP, SAT and Max-SAT to maximization of likelihood in stochastic variants such as Markov Random Fields or Bayesian networks. They have been applied to a variety of real problems, in resource allocation, bioinformatics or machine learning among others~\cite{CELAR99,MST2008,charpillet,ermon,schiex15,CPD-AIJ,cpaior16}.

Besides being equipped with an efficient branch and bound procedure augmented with powerful local consistency techniques, a practical CFN solver should have a good library of global cost functions to model the often complex scenarios in real-life applications.

Enforcing  local consistencies requires to
apply Equivalence Preserving Transformations (EPTs) such as cost
projection and extension~\cite{MT2004}. Most local consistencies require
to compute minima of the cost function to determine the amount of cost to project/extend. 
By applying these operations, local consistencies may reduce domains and, more importantly, tighten
a global lower bound on the criteria to optimize. This is crucial for branch and bound efficiency.
Global cost functions have unbounded arity, but may have a specific semantics that makes available
dedicated polynomial-time algorithms for minimization. However, when
local consistencies apply EPTs, they modify the cost function and may
break the properties that makes it polynomial-time minimizable. 
We say that a cost function is {\em tractable\/} if it can be minimized
  in polynomial time. 
The notion of {\em tractable projection-safety\/} captures precisely those
functions that remain tractable even after EPTs.

In this paper, we prove that any tractable global cost function remains
tractable after EPTs to/from the zero-arity cost function
($\cf_\varnothing$), and cannot remain tractable if arbitrary EPTs to/from
$r$-ary cost functions for $r \geq 2$ are allowed.  When $r = 1$, we show 
that the answer is indefinite. We describe a simple tractable global cost
function and show how it becomes intractable after projections/extensions 
to/from unary cost functions. We also show that flow-based 
projection-safe cost functions~\cite{LL2009} are positive
examples of tractable projection-safe cost functions. 

For $r=1$, we  introduce {{\em polynomially DAG-filterable}} global 
cost functions, which can be transformed into a filtering Directed Acyclic Graph with a 
polynomial number of simpler cost functions for (minimum) cost calculation.
Computing minima of such cost functions, using a polynomial time 
dynamic programming algorithm, is tractable and
remains tractable after projections/extensions. Thus, polynomially
{DAG-filterable} cost functions are tractable projection-safe.  Adding to
the existing repertoire of global cost functions,  cost function variants of existing global constraints such as
\gc{Among}, \gc{Regular}, \gc{Grammar}, and \gc{Max}/\gc{Min}, are
proved to be polynomially {DAG-filterable}.

To avoid the need to implement dedicated dynamic programming
algorithms, we also consider the possibility of directly using
decompositions of global cost functions into polynomial size networks of cost functions with bounded arities, usually ternary cost functions. 
We show how such {\em network-decompositions} 
can be derived from known global constraint decompositions and how
Berge-acyclicity allows soft local consistencies to emulate dynamic
programming in this case. We prove that Berge-acyclic network-decompositions
can also be used to directly build polynomial filtering DAGs.

To demonstrate the feasibility of these approaches, we implement and
embed various global cost functions using filtering DAG and network-decompositions 
in~\texttt{\small toulbar2}, an open source cost function networks solver.
We conduct experiments using different  benchmarks to
evaluate and to compare the performance of the DAG-based and network-based decomposition approaches. 

{The rest of the paper is organized as follows. 
Section \ref{background} contains the necessary background 
to understand our contributions. 
Section \ref{tractable-PS} analyses the tractability of enforcing 
local consistencies on global cost functions and characterizes 
the conditions for preserving tractability after applying EPTs. 
In Section \ref{sec:proj-constr} we define DAG-filtering and in Section \ref{example-DAG} we 
give an example of a polynomial {DAG-filterable} global cost function. 
Sections \ref{s:dec} and \ref{s:loc} present network-decomposability and the conditions 
for preserving the level of local consistency. 
Section \ref{comparison} shows the relation between network-decompositions and DAG-filtering. 
Section \ref{sec-expe} provides an experimental analysis of the two approaches 
on several classes of problems. 
Section \ref{conclusion} concludes the paper. 
}


\section{Background}
\label{background}

We give preliminaries on cost function networks 
and global cost functions.

\subsection{Cost Function Networks}
\label{bg:cfn}

A cost function network (CFN)
is {a special case of the
valued constraint satisfaction problem}~\cite{THG1995} with a specific cost
structure $([0,\ldots,\top], \oplus, \leq)$. We give the formal
definitions of the cost structure and CFN as follows.

\begin{definition}[{Cost Structure \cite{THG1995}}] The \emph{cost structure}
  $([0,\ldots,\top], \oplus, \leq)$ is a tuple defined as:
  \begin{itemize}
  \item{} $[0,\ldots,\top]$ is the interval of integers from $0$ to
    $\top$ ordered by the standard ordering $\leq$, where $\top$ is either a positive integer or $+\infty$.
  \item{} $\oplus$ is the addition operation defined as $a \oplus b =
    min(\top,a+b)$. We also define the subtraction $\ominus$ operator for any
    $a$ and $b$, where $a \geq b$, as:
    \[
    a \ominus b = \left\{ \begin{array}{ll}
        a-b, & \mbox{ if $a \neq \top$;} \\
        \top, & \mbox{ otherwise}
      \end{array} \right. 
    \]
  \end{itemize}
\end{definition}

Note that more general additive cost structures have also been used. 
Specifically, VAC and OSAC~\cite{Cooper2010} local
consistencies are defined using a structure using non-negative
rational instead of non-negative integer numbers. For ease of understanding, 
our discussion assumes integral costs. However, it can easily be generalized to rational costs.

\begin{definition}[{Cost Function Network \cite{Schiex00b}}] 
A \emph{Cost Function Network} (CFN) is a tuple $(\X,\C,\top)$,
  where: \vspace{-0.1cm}
  \begin{itemize}
  \item{} $\X$ is an ordered set of discrete domain variables $\{x_{1}, x_{2}, \ldots, x_{n}\}$. The domain of  $x_{i} \in \X$ being denoted as $D(x_i)$;
  \item{} $\C$ is a set of cost functions $\cf_S$ each with a scope
    $S = \{x_{s_1}, \ldots, x_{s_r}\} \subseteq \X$ that maps tuples $\ell \in \L^S$, where $\L^S = D(x_{s_1}) \times \cdots \times   D(x_{s_r})$, to $[0,\ldots,\top]$. 
\end{itemize}
\end{definition}

When the context is clear, we abuse notation by denoting an assignment 
of a set of variables $S \subseteq \X$ as a tuple $\ell =
(v_{s_1}, \ldots, v_{s_r}) \in \L^S$.  
The notation $\ell[x_{s_i}]$ denotes the value $v_{s_i}$ assigned 
to $x_{s_i}$ in $\ell$, and $\ell[S']$ denotes the tuple formed by 
projecting $\ell$ onto $S'\subseteq S$. 
%
%
Without loss of generality, we assume $\C = \{\cf_\varnothing\} \cup
\{\cf_i \mid x_i \in \X\} \cup \C^+$.  $\cf_\varnothing$ is a constant zero-arity cost function.  $\cf_i$ is a unary cost function associated with each $x_i \in \X$.  $\C^+$ is a set of cost functions $\cf_S$ with scope $S$ and $|S|\geq 2$.  If $\cf_\varnothing$ and $\{\cf_i\}$ are not defined, we assume $\cf_i(v) = 0$ for all $v
\in D(x_i)$ and $\cf_\varnothing = 0$. To simplify notation, we also
denote by $\cf_{s_1,s_2,\ldots,s_r}$ the cost function on variables $\{x_{s_1}, x_{s_2},\ldots, x_{s_r}\}$ when the context is clear.

\begin{definition}
  Given a CFN $(\X, \C, \top)$, the {\em cost\/} of a tuple $\ell
  \in \L^\X$ is defined as $cost(\ell) = \bigoplus_{\cf_S \in \C}
  \cf_S(\ell[S])$.  A tuple $\ell \in \L^\X$ is {\em feasible\/} if  $cost(\ell) < \top$, and it is an  {\em optimal solution\/} of the CFN if  $cost(\ell)$ is minimum among all tuples in $\L^\X$. 
\end{definition}

We  observe that a classical \emph{Constraint Network} is merely a CFN where all cost functions $\cf_S\in\C$ are such that $\forall \ell\in \L^S, \cf_S(\ell)\in\{0,\top\}$. 
The problem of the existence of a solution in a constraint network, called \emph{Constraint Satisfaction Problem (CSP)}, is NP-complete. 
Finding an  optimal solution to a CFN is thus above NP. 
Restrictions to Boolean variables and binary constraints are known to be APX-hard~\cite{papadimitriou1991}. In the terminology of stochastic graphical models, this problem is also equivalent to the Maximum A Posteriori (MAP/MRF) problem or the Maximum Probability Explanation (MPE) in Bayesian networks~\cite{cpaior16}. CFNs can be solved exactly with depth-first branch-and-bound search using $\cf_{\varnothing}$ as a lower bound. Search efficiency is enhanced by maintaining local consistencies that increase the lower
bound by redistributing costs among $\cf_S$, pushing costs into $\cf_{\varnothing}$ and $\cf_i$, and pruning values while preserving the equivalence of the problem (\ie, the cost of each tuple $\ell \in \L^\X$ is unchanged).

\subsection{Soft local consistencies and EPTs}

Different consistency notions have been defined. Examples include NC*~\cite{JT2004_AC}, (G)AC*~\cite{Schiex00b,JT2004_AC,MT2004,LL2009,LL2012asa},  FD(G)AC*~\cite{JT2004_AC,Larrosa2003,LL2009,LL2012asa}, (weak)  ED(G)AC*~\cite{GHZL2005,LL2010,LL2012asa}, VAC and OSAC~\cite{Cooper2010}. Enforcing such local consistencies requires  applying equivalence preserving transformations (EPTs) that shift costs between different scopes. 
The main EPT is defined below and described as Algorithm~\ref{l-ept}. This is a compact version of the projection 
and extension defined in \cite{MC2005}.

\begin{definition}[{EPTs \cite{MC2005}}]\label{r-ept}
  Given two cost functions $\cf_{S_1}$ and $\cf_{S_2}$, $S_2\subset
  S_1$, the EPT \Proj$(S_1,S_2,\ell,\alpha)$ shifts an amount of
  cost $\alpha$ between a tuple $\ell \in \L^{S_2}$ of $\cf_{S_2}$ and the cost function $\cf_{S_1}$. The direction of the shift is given by the
  sign of $\alpha$. The precondition guarantees that costs remain non  negative after the EPT has been applied.

 Denoting by $r=|S_2|$, the EPT is called {an $r$-EPT. It is }an $r$-projection when  $\alpha\geq 0$ and an $r$-extension when $\alpha<0$.
\end{definition}

\begin{algorithm}\label{l-ept}
  Precondition: $-\cf_{S_2}(\ell)\leq\alpha\leq \min_{\ell'\in \L^{S_1}, \ell'[S_2] = \ell}W_{S_1}(\ell')$\;
  \Procedure{\Proj{$S_1,S_2,\ell,\alpha$}}
  {
    $\cf_{S_2}(\ell) \gets \cf_{S_2}(\ell) \oplus \alpha$\;
    \ForEach{($\ell'\in \L^{S_1}$ such that $\ell'[S_2] = \ell$)}
    {$\cf_{S_1}(\ell') \gets \cf_{S_1}(\ell') \ominus \alpha$\;}
  }
  \caption{\small A cost shifting EPT used to enforce soft local
    consistencies. The $\oplus, \ominus$ operations are extended here
    to handle possibly negative costs as follows: for non-negative
    costs $\alpha, \beta$, we have $\alpha \ominus (-\beta) = \alpha
    \oplus \beta$ and for $\beta\leq \alpha$, $\alpha\oplus (-\beta) =
    \alpha \ominus\beta$.}
\end{algorithm}

It is now possible to introduce local consistency enforcing algorithms.

\begin{definition}[{Node Consistency \cite{JT2004_AC}}] 
  A variable $x_i$ is \emph{star node consistent} (NC*) if each value $v
  \in D(x_i)$ satisfies $\cf_i(v) \oplus \cf_\varnothing < \top$ and
  there exists a value $v' \in D(x_i)$ such that $\cf_i(v') = 0$. A  CFN is NC* iff all variables are NC*.
\end{definition}

Procedure \opSty{enforceNC*}() in Algorithm \ref{enforceNC} enforces
NC*, where \opSty{unaryProject}() applies EPTs that move unary costs
towards $\cf_\varnothing$ while keeping the solution unchanged, and
\opSty{pruneVar}($x_i$) removes infeasible values.

\begin{algorithm}
  \caption{Enforce NC* \label{enforceNC}}
  \SetKw{Func}{Function} \SetKw{Proc}{Procedure} \SetKw{False}{false}
  \SetKw{True}{true} \SetKwBlock{Start}{}{}
  
  \Proc{} \opSty{enforceNC*}() 
  \Start{
    \lnl{ncL1}
    \lForEach{$x_i \in \X$}{\opSty{unaryProject}($x_i$)\;}
    \lnl{ncL2}
    \lForEach{$x_i \in \X$}{\opSty{pruneVar}($x_i$)\;}
  }
  \Proc{} \opSty{unaryProject}($x_i$)
  \Start{
    \lnl{ncL3}
    $\alpha := \min\{\cf_i(v) \mid v \in D(x_i)\}$\;
    \lnl{ncL6}
    \Proj{$\{x_i\},\varnothing,(),\alpha$}\;
  }
	\Proc{} \opSty{pruneVar}($x_i$)
  \Start{	  
    \lnl{ncL10}
    \ForEach{$v \in D(x_i)$ s.t. $\cf_i(v) \oplus \cf_\varnothing = \top$} {
      \lnl{ncL12}
      $D(x_i) := D(x_i) \setminus \{v\}$\;
    }
  }
\end{algorithm}
  
\begin{definition}[{(Generalized) Arc Consistency \cite{MT2004,LL2009,LL2012asa}}] Given a CFN
 $P=(\X, \C, \top)$, a cost function $\cf_S \in \C^+$ and a  variable $x_i \in S$.
  \begin{itemize}
  \item{} A tuple $\ell \in \L^S$ is a \emph{simple support} for $v
    \in D(x_i)$ with respect to $\cf_S$ with $x_i \in S$ iff
    $\ell[x_i] = v$ and $\cf_S(\ell) = 0$.
  \item{} A variable $x_i \in S$ is \emph{star generalized arc  consistent (GAC*)} with respect to $\cf_S$ iff \vspace{-5pt}
    \begin{itemize}
    \item{} $x_i$ is NC*;
    \item{} each value $v_i \in D(x_i)$ has a simple support $\ell$
      with respect to $\cf_S$.
    \end{itemize}
  \item{} A CFN is {\em GAC*} iff all variables are GAC* with respect
    to all related non-unary cost functions.
  \end{itemize}
\end{definition}

To avoid exponential space complexity, the GAC* definition and the algorithm is slightly different from the one given by 
Cooper and Schiex~\cite{MT2004}, which also requires for every
tuple $\ell \in \L^S$, $\cf_S(\ell) = \top$ if $\cf_\varnothing
\oplus \bigoplus_{x_i \in S}\cf_i(\ell[x_i]) \oplus \cf_S(\ell) =
\top$.

The procedure \opSty{enforceGAC*}() in Algorithm~\ref{enforceGAC},
enforces GAC* on a single variable $x_i \in \X$ with respect to a cost function $\cf_S \in \C^+$,
where $x_i \in S$ in a CFN $(\X, \C, \top)$. The procedure first computes the minimum when $x_i = v$ 
for each $v \in D(x_i)$ at line \ref{gacL14}, then performs a $1$-projection from $\cf_S$ to $\cf_i$ at 
line \ref{gacL16}. Lines \ref{gacL19} and \ref{gacL20} enforce NC* on $x_i$.

\begin{algorithm}[htb]
  \caption{Enforcing GAC* for $x_i$ with respect to $\cf_S$ \label{enforceGAC} }
  \SetKw{Proc}{Procedure}     
  \SetKwBlock{Start}{}{}    
  \Proc{} \opSty{enforceGAC*}($\cf_S$, $x_i$)
  \Start{
    \lnl{gacL13}
    \ForEach{$v \in D(x_i)$} {
      \lnl{gacL14}
      $\alpha := \min\{\cf_S(\ell) \mid \ell \in \L^S \wedge \ell[x_i] = v\}$\;	
      \lnl{gacL16}
      \Proj{$S,\{x_i\},(v),\alpha$}
    }
    \lnl{gacL19}
    \opSty{unaryProject}($x_i$)\;
		\lnl{gacL20}
		\opSty{pruneVar}($x_i$)\;
  }
\end{algorithm}

Local consistency enforcement involves two types of operations: (1) finding the minimum cost returned by the cost functions among all (or part of the) tuples; (2) applying EPTs that shift costs to and from smaller-arity cost functions.

Minimum cost computation corresponds to line~\ref{ncL3} in Algorithm~\ref{enforceNC}, and line~\ref{gacL14} in Algorithm~\ref{enforceGAC}. For simplicity, we write $min\{\cf_S(\ell) \mid \ell \in \L^S\}$ as $min\{\cf_S\}$.

In practice, projections and extensions can be performed in constant time using the $\Delta$ data-structure introduced in Cooper and Schiex~\cite{MT2004}. For example, when we perform $1$-projections or $1$-extensions, instead of modifying the costs of all tuples, we store the projected and extended costs in $\Delta^{-}_{x_i,v}$ and $\Delta^{+}_{x_i,v}$ respectively. Whenever we compute the value of the cost function $\cf_S$ for a tuple $\ell$ with $\ell[x_i] = v$, we return $\cf_S(\ell) \ominus \Delta^{-}_{x_i,v} \oplus
\Delta^{+}_{x_i,v}$. The time complexity of enforcing one of the previous consistencies is thus entirely defined by the time complexity of computing the minimum of a cost function during the enforcing.  

\begin{proposition}\label{lem:supporttime}
   The procedure \opSty{enforceGAC*}$()$ in Algorithm \ref{enforceGAC}  
   requires $O(d \cdot f_{min})$ time, where $d$ is the maximum domain
   size and $f_{min}$ is the time complexity of minimizing $\cf_S$.
\end{proposition}
\begin{proof}
  Line \ref{gacL14} requires  $O(f_{min})$ time. We can replace the domain of $x_i$ by $\{v\}$, and run  the minimum computation to get the minimum cost. Projection at line  \ref{gacL16} can be performed in constant time. Thus, each iteration  requires $O(f_{min})$. Since the procedure iterates $d$ times, and  the procedures \opSty{unaryProject} and \opSty{pruneVar} requires $O(d)$, the overall  complexity is $O(d \cdot f_{min} + d) = O(d \cdot f_{min})$.
\end{proof}

In the general case, $f_{min}$ is in $O(d^r)$ where $r$ is the size of the scope and $d$ the maximum domain size. However, a \emph{global cost function} may have specialized algorithms which make the operation of finding minimum, and thus consistency enforcement, tractable.

\subsection{Global Constraints, Soft Global Constraints and Global Cost Functions}

\begin{definition}[Global Constraint~\cite{BCR2005,FPT2006}]
  A \emph{global constraint}, denoted by  \gc{GC}$(S, A_1, \ldots, A_t)$, is a family of hard constraints parameterized by a scope $S$, and possibly extra parameters $A_1,\ldots,A_t$.
\end{definition}

Examples of global constraints are \gc{AllDifferent}~\cite{JLL1978}, \gc{GCC} \cite{JCR1996}, \gc{Same} \cite{NIS2004}, \gc{Among}~\cite{BC1994}, \gc{Regular} \cite{GP2004}, \gc{Grammar} \cite{KS2010}, and \gc{Maximum}/\gc{Minimum} constraints \cite{NB2001}. Because of their unbounded scope, global constraints cannot be efficiently propagated by generic local consistency algorithms, which are exponential in the arity of the constraint.  Specific propagation algorithms are designed to achieve polynomial time complexity in the size of the input, \ie the scope, the domains and extra parameters.

To capture the idea of costs assigned to constraint violations, the notion of \emph{soft global constraint} has been introduced. 
This is a traditional global constraint with one extra variable representing the cost 
of the assignment w.r.t. to an existing global constraint. The  cost is given by a violation measure  function.

\begin{definition}[Soft Global Constraint \cite{TJC2001}]
  A \emph{soft global constraint}, denoted by  \gc{Soft\_GC}$^{\mu}(S \cup \{z\}, A_1, \ldots,A_t)$, is a family of hard constraints parameterized by 
  a violation measure $\mu$, 
  a scope $S$, 
  a cost variable $z$, and possibly extra parameters $A_1,\ldots,A_t$. The constraint is satisfied if and only if 
  $z=\mu(S, A_1, \ldots,A_t)$.
\end{definition}

Soft global constraints are used to introduce costs in the CSP framework, and therefore inside constraint programming solvers~\cite{PetitRB01}.
 It requires the introduction of extra cost variables and does not exploit the stronger propagation offered by some of the soft local consistencies. A possible alternative, when a sum of costs needs to be optimized, lies in the use of \emph{global cost functions}.

\begin{definition}[{Global Cost Function~\cite{MCT2009,LL2012asa}}]
\begin{sloppypar}
  A \emph{global cost function}, denoted  as \gc{W\_GCF}$(S, A_1, \ldots, A_t)$, is a family of cost  functions parameterized by a scope $S$ and possibly extra parameters $A_1, \ldots, A_t$.
	\end{sloppypar}
\end{definition}

For example, if $S$ is a set of variables with non-negative integer domains, it is easy to define the Global Cost Function \gc{W\_Sum}$(S) \equiv \bigoplus_{x_i\in S} \min(\top,x_i)$.

It is possible to derive a global cost function from an existing soft global constraint \gc{Soft\_GC}$^{\mu}(S \cup \{z\}, A_1, \ldots,A_t)$. In this case, we denote the corresponding global cost function as \gc{W\_GCF}$^\mu$. Its value for a tuple $\ell\in\L^S$ is equal to $\min(\top,\mu(\ell))$.  

For example, global cost functions \gc{W\_AllDifferent}$^{var}$/\gc{W\_AllDifferent}$^{dec}$~\cite{LL2009,LL2012asa} can be derived
from two different violation measures of \gc{AllDifferent}, namely variable-based and decomposition-based~\cite{TJC2001,WGL2006}, respectively.  Other examples include \gc{W\_GCC}$^{var}$ and \gc{W\_GCC}$^{val}$~\cite{LL2009,LL2012asa},
\gc{W\_Same}$^{var}$~\cite{LL2009,LL2012asa}, 
\gc{W\_SlidingSum}$^{var}$~\cite{LS2011}, \gc{W\_Regular}$^{var}$ and \gc{W\_Regular}$^{edit}$~\cite{BES12,LL2009,LL2012asa}, \gc{W\_EGCC}$^{var}$~\cite{LS2011},  \gc{W\_Disjunctive}$^{val}$ and \gc{W\_Cumulative}$^{val}$~\cite{LS2011,lee2014}.


\section{Tractable Projection-Safety}\label{tractable-PS}

All soft local consistencies are based on the use of EPTs, 
shifting costs between two scopes. 
The size of the smallest scope used in a EPT is called 
the \emph{order} ($r$) of the EPT. Such a EPT is called an $r$-EPT. It is directly related to the level of local consistency enforced: 
node consistency uses EPTs onto the empty scope ($r=0$), arc consistencies 
use unary scopes ($r=1$) whereas higher-order consistencies use larger 
scopes ($r \geq 2$)~\cite{MC2005}. In this section, we show that the
order of the {EPTs} directly impacts the tractability of global cost function minimization.

To be able to analyze complexities in global cost functions, 
we first  define the  decision problem associated with the optimization problem 
$\min\{\gc{W\_GCF}(S, A_1, \ldots, A_t)\}$. 

\begin{itemize}
\item[]{}{\sc IsBetterThan}(\gc{W\_GCF}$(S, A_1, \ldots, A_t), m)$
\item[]{$\mathbf{Instance.}$} A global cost function \gc{W\_GCF}, a scope $S$ with domains for the variables in $S$, values for the parameters $A_1, \ldots, A_t$, and a fixed integer $m$.
\item[]{$\mathbf{Question.}$} Does there exist a tuple $\ell \in
  \L^S$ such that \gc{W\_GCF}$(S, A_1, \ldots, A_t)(\ell) {<} m$?
\end{itemize}
\noindent
{We can then define the tractability of a global cost function. }

\begin{definition}
  A global cost function \gc{W\_GCF}$(S, A_1, \ldots, A_t)$ is said to be \emph{tractable} iff 
  the problem {\sc IsBetterThan}(\gc{W\_GCF}$(S, A_1, \ldots, A_t),m$) is in $P$.
\end{definition}

\noindent 
For a tractable global cost function $\cf_S =$\gc{W\_GCF}$(S, A_1, \ldots, A_t)$, the time complexity of computing $\min\{\cf_S\}$ is bounded above by a polynomial function in the size of the
input, including the scope, the corresponding domains, the other parameters of the global cost function, and $\log(m)$.

We introduce \emph{tractable $r$-projection-safety} 
global cost functions, which remain \emph{tractable} after applying 
$r$-EPTs.

\begin{definition}
  We say that a global cost function W\_GCF$(S,A_1,\ldots,A_t)$ is \emph{tractable $r$-projection-safe} iff:
  \begin{itemize}
  \item{} it is tractable and;
  \item{} any global cost functions that can be derived from W\_GCF$(S,A_1,\ldots,A_t)$ by a series of $r$-EPTs is also tractable.
\end{itemize}
\end{definition}

The tractability after $r$-EPTs depends on $r$. We divide the discussion of tractable $r$-projection-safety into three cases: $r=0$, $r\geq 2$ and $r = 1$. 
{In the following, given a tractable global cost function $\cf_S$, 
we  denote by  $\nabla_r(\cf_S)$ the global cost function resulting from  the application of an arbitrary  finite sequence of $r$-EPTs on $\cf_S$.}

\subsection{Tractability and $0$-EPTs}

When $r=0$, EPTs are performed to/from $\cf_\varnothing$. This kind of EPTs is used when enforcing Node Consistency (NC*)~\cite{JT2004_AC} but also in $\varnothing$-inverse consistency~\cite{MCT2009}, and strong $\varnothing$-inverse consistency~\cite{LL2009,LL2012asa}.

We show that if a global cost function is tractable, it remains tractable after applying such EPTs.
\begin{theorem}
  Every tractable global cost function is tractable $0$-projection-safe.
\end{theorem}
\begin{proof}
  Consider a tractable global cost function $\cf_S =\gc{W\_GCF}(S, A_1, \ldots, A_t)$. 
  Clearly, $\cf_S$ and $\nabla_0(\cf_S)$  only differ by a constant, \ie there exists $\alpha^-$ and
  $\alpha^+$, where $\alpha^-,\alpha^+ \in \{0,\ldots,\top\}$, such  that: \[
    \nabla_0(\cf_S)(\ell) = \cf_S(\ell) \oplus \alpha^+ \ominus
    \alpha^-, \textrm{ for all } \ell \in \L^S
    \]
    If $\cf_S(\ell) = \min\{\cf_S\}$ for some $\ell \in \L^S$, then   $\nabla_0(\cf_S)(\ell) = \min\{\nabla_0(\cf_S)\}$.  If $\cf_S$ is tractable, so is $\nabla_0(\cf_S)$.
\end{proof}

\subsection{Tractability and  {EPTs of order greater than 2}}

When $r \geq 2$, EPTs are performed to/from $r$-arity cost functions. This is required for enforcing higher order consistencies and is used in practice in ternary cost functions processing~\cite{MST2008} and complete $k$-consistency~\cite{MC2005}.

If arbitrary sequences of $r$-EPTs are allowed, we show that tractable global cost functions always become intractable after some sequence of $r$-EPT applications, where $r \geq 2$.

\begin{theorem}
  \label{thm:r-proj}
  Any tractable global cost function \gc{W\_GCF}$(S, A_1, \ldots, A_t)$  returning finite costs is not tractable $r$-projection-safe for $r \geq 2$, {unless ${P} = {NP}$.}   
\end{theorem}

\begin{proof}
{
Let us first define the binary constraint satisfaction problem  {\sc ArityTwoCSP} as follows. }
  \begin{itemize}
  \item[]{}{\sc ArityTwoCSP}$(\X,\C^h)$
  \item[]{$\mathbf{Instance.}$} A CSP instance $(\X,\C^h)$, where  every constraint $C^h_S \in \C^h$ involves two variables, \ie  $|S| = 2$. 
  \item[]{$\mathbf{Question.}$} Is the CSP $(\X,\C^h)$ satisfiable?
  \end{itemize}
\noindent	
{ 
{\sc ArityTwoCSP} is NP-hard 
as graph coloring 
can be solved through a direct modeling into {\sc ArityTwoCSP}.
}
%
%
We reduce the problem {\sc ArityTwoCSP}$(\X,\C^h)$  to the problem {\sc IsBetterThan}$(\nabla_2(\cf_\X), \top)$, where $\cf_\X =
  \mbox{\gc{W\_GC}}(\X,A_1,\ldots,A_t)$ is an arbitrary 
  global cost function using only finite costs.  
We first construct a CFN $(\X,\C \cup  \{\cf_\X\},\top)$. 
The upper  bound $\top$ is a sufficiently large integer such that $\top >
\cf_\X(\ell)$ for every $\ell \in \L^S$,  
which is always possible given that $\cf_\X$ remains finite.  
This technical restriction is not significant: if a
global cost function $\cf_S$ maps some tuples to infinity, we can
transform it to another cost function $\cf_S^\prime$ such that the
infinity costs are replaced by a sufficiently large integer $p$ such
that
$p \gg \max\{\cf_S(\ell) \mid \ell \in \L^S \wedge \cf_S(\ell) \neq
+\infty\}$.

  The cost functions $\cf_S \in \C \setminus \{\cf_\X\}$ are defined as follows:  \[
  \cf_S(\ell) = \left\{ \begin{array}{ll}
      0, & \mbox{ if $\ell$ is accepted by $C^h_S\in \C^h$}; \\
      \top, & \mbox{ otherwise} \\
    \end{array} \right. 
  \]
From the CFN, $\nabla_2$ can be defined as follows: for each  forbidden tuple $\ell[S]$ in each $C^h_S \in \C^h$, we add an  extension of $\top$ from $\cf_S$ to $\cf_\X$ with respect to  $\ell[S]$ into $\nabla_2$.  Under this construction,  $\nabla_2(\cf_\X)(\ell)$ can be represented as:
  \[
  \nabla_2(\cf_\X)(\ell) = \cf_\X(\ell) \oplus \bigoplus_{\cf_S \in
    \C} \cf_S(\ell[S])
  \]
{For a tuple $\ell \in \L^\X$, $\nabla_2(\cf_\X)(\ell) = \top$ iff 
$\ell$ is forbidden by some $C^h_S$ in $\C^h$. As a result  
 {\sc IsBetterThan}$(\nabla_2(\cf_\X), \top)$ is satisfiable iff 
 {\sc  ArityTwoCSP}$(\X,\C^h)$ is satisfiable. 
 As {\sc ArityTwoCSP} is NP-hard, 
 {\sc IsBetterThan}$(\nabla_2(\gc{W\_GC}), \top)$ is not polynomial, unless $P=NP$. 
 Hence,  $\nabla_2(\gc{W\_GC})$ is not tractable, 
 and then, $\gc{W\_GC}$ is not 
 tractable $2$-projection-safe, unless $P = NP$.}
\end{proof}

\subsection{Tractability and $1$-EPTs}

When $r=1$, $1$-EPTs cover $1$-projections and $1$-extensions, which 
are the backbone of the consistency algorithms of 
(G)AC*~\cite{JT2004_AC,LL2009,LL2012asa}, FD(G)AC* 
\cite{Larrosa2003, LL2009,LL2012asa}, (weak) ED(G)AC*~\cite{GHZL2005,LL2010,LL2012asa}, 
VAC, and OSAC \cite{Cooper2010}. In these cases, 
tractable cost functions are tractable
$1$-projection-safe only under special conditions.  
For example, Lee and Leung
define \emph{flow-based projection-safety} based on a 
flow-based global cost function.

\begin{definition}[{Flow-based \cite{LL2009,LL2012asa}}] 
A global cost function $\gc{W\_GCF}(S, A_1, \ldots, A_t)$  is \emph{flow-based} iff it can be represented as a flow network 
$G$ such that the minimum cost among all maximum flows between a fixed source and a fixed destination is equal to $\min\{\gc{W\_GCF}(S, A_1, \ldots, A_t)\}$.
\end{definition}

\begin{sloppypar}
\begin{definition}[{Flow-based projection safe \cite{LL2009,LL2012asa}}] A global cost function $\gc{W\_GCF}(S, A_1, \ldots, A_t)$ is flow-based projection-safe iff it is is flow-based, and is still flow-based following any sequence of $1$-projections and $1$-extensions.
\end{definition}
\end{sloppypar}

Lee and Leung~\cite{LL2009,LL2012asa} further propose sufficient conditions for tractable cost functions to be flow-based projection-safe. Flow-based projection-safety implies tractable $1$-projection-safety.  We state the result in the following theorem.
\begin{theorem}
  Any flow-based projection-safe global cost function is tractable $1$-projection-safe.
\end{theorem}
\begin{proof}
    Follows directly from the tractability of the minimum cost flow algorithm. 
\end{proof}

However, tractable cost functions are not necessarily tractable $1$-projection-safe. One example is \gc{W\_2SAT}, which is a global cost function derived from an instance of the polynomial {\sc 2SAT} problem.

\begin{definition} 
  Given a set of Boolean variables $S$, a set of binary clauses $F$,  and a positive integer $c$, the global cost function
  \gc{W\_2SAT}$(S, F, c)$ is defined as:
\begin{equation*}   
  \gc{W\_2SAT}(S, F, c)(\ell) = \left\{ \begin{array}{ll}
      0, & \mbox{ if $\ell$ satisfies $F$} \\
      c, & \text{ otherwise}
    \end{array} \right.     
\end{equation*}
\end{definition}

\noindent 
\gc{W\_2SAT} is tractable, because the {\sc 2SAT} problem is tractable \cite{KM1967}. However, it is not tractable $1$-projection-safe.
\begin{theorem}    
    \gc{W\_2SAT} is not tractable  $1$-projection-safe, {unless ${P}={NP}$.} 
\end{theorem}
\begin{proof}
{
Let us first define the {\sc WSAT-2-CNF} problem. }

  \begin{itemize}
    \itemsep=-3pt
  \item[]{}{\sc WSAT-2-CNF}
  \item[]{$\mathbf{Instance.}$} A 2-CNF formula $F$ (a set of binary clauses) 
  and a fixed integer $k$.
  \item[]{$\mathbf{Question.}$} Is there an assignment that satisfies all clauses in $F$ with at most $k$ variables set to $true$ ?
	  \end{itemize}
{\sc WSAT-2-CNF} was shown NP-hard in  \cite[page 69]{FG2006}. 
We reduce it 
to the problem {\sc IsBetterThan}$(\nabla_1(\gc{W\_2SAT}), \top)$.

We construct a particular sequence of $1$-projections and/or $1$-extensions $\nabla_1$ such that the {\sc WSAT-2-CNF} instance  
can be solved using $\cf_\X =
          \gc{W\_2SAT}(\X,F,k+1)$ from the Boolean CFN $N=(\X, \C \cup\{\cf_{\X}\}, k+1)$.  $\C$ only contains
unary cost functions $\cf_i$, which are defined as follows:	  \[
          \cf_i(v) = \left\{ \begin{array}{ll}
              1, & \mbox{ if $v = true$}; \\
              0, & \mbox{ otherwise} \\
            \end{array} \right. 
          \]
Based on $N$, we construct $\nabla_1$ as follows: for each  variable $x_i \in \X$, we add an extension of $1$ from   $\cf_i$ to $\cf_\X$ with respect to the value $true$ into  $\nabla_1$. As a result, a tuple $\ell$ with    $\nabla_1(\cf_\X)(\ell) = k' \leq k$ contains exactly $k'$ variables set to $true$ (because every $x_i=true$  incurs a cost of $1$) and also satisfies $F$ (or it would have cost $k+1 = \top$). 
Thus, the {\sc WSAT-2-CNF} instance with threshold $k$ is 
satisfiable iff 
{\sc IsBetterThan}$(\nabla_1(\cf_\X),k+1)$ is satisfiable. 
As {\sc WSAT-2-CNF} is NP-hard, 
{\sc IsBetterThan}$(\nabla_1(\gc{W\_2SAT}),k+1)$ is not polynomial, unless $P=NP$. 
Hence,  $\nabla_1(\gc{W\_2SAT})$ is not tractable, 
and then, $\gc{W\_2SAT}$ is not 
 tractable $1$-projection-safe, unless $P = NP$.
\end{proof}

When the context is clear, we use tractable projection-safety, projection and extension to refer to tractable $1$-projection-safety,
$1$-projection and $1$-extension respectively hereafter.


\section{Polynomial DAG-Filtering}
\label{sec:proj-constr}

Beyond flow-based global cost functions~\cite{LL2009,LL2012asa}, we introduce now an additional class of tractable projection-safe cost functions based on dynamic programming algorithms. As mentioned by Dasgupta \etal\cite{DAG2007}, every dynamic programming algorithm has an underlying DAG structure.
\begin{definition}[DAG]
  A \emph{directed acylic graph (DAG)} $T = (V,E)$, where $V$ is a set  of vertices (or nodes) and $E \subseteq V \times V$ is a set of directed edges,
  is a directed graph with no directed cycles, and:  \begin{itemize}
  \item{} An edge $(u,v) \in E$ points from $u$ to $v$, where $u$ is the \emph{parent} of $v$, and $v$ is the \emph{child} of
    $u$;
  \item{} A \emph{root} of a DAG is a vertex with zero in-degree;
  \item{} A \emph{leaf} of a DAG is a vertex with zero out-degree;
  \item{} An \emph{internal vertex} of a DAG is any vertex which is not a leaf;
  \end{itemize}
\end{definition}

\noindent
We now introduce the \emph{DAG filterability} of a global cost function.

\begin{definition}[DAG-filter]
  \label{def:recurdecompose}
  A \emph{DAG-filter} for a cost function $\cf_S$ is a DAG $T = (V,E)$ such that:
  \begin{itemize}
  \item $T$ is connected;
  \item $V = \{\smallcf_{S_i}\}_i$ is a set of cost function vertices each with a scope $S_i$, among which vertex $\cf_S$ is the root of $T$;
  \item Each internal vertex $\smallcf_{S_i}$ in $V$ is associated with an aggregation function $f_i$ that maps a multiset of costs 
	$\{\alpha_j \mid \alpha_j \in [0\ldots\top]\}$ to $[0\ldots\top]$ and is based on an associative and commutative binary operator;
  \item For every internal $\smallcf_{S_i}\in V$,
    \begin{itemize}
    \item the scope of $\smallcf_{S_i}$ is composed from its children's scopes:   
      \[
      S_i = \bigcup_{(\smallcf_{S_i}, \smallcf_{S_j}) \in E} S_j
      \]
    \item $\smallcf_{S_i}$ is the aggregation of its children:  \[
      \smallcf_{S_i}(\ell) = f_i(\{\smallcf_{S_j}(\ell[S_j]) \mid
      (\smallcf_{S_i}, \smallcf_{S_j}) \in E\});
      \]
    \item{} $\min$ is distributive over $f_i$: \[
      \min\{\smallcf_{S_i}\} = f_i(\{\min\{\smallcf_{S_j}\} \mid
      (\smallcf_{S_i}, \smallcf_{S_j}) \in E\}).
      \]
    \end{itemize}  
  \end{itemize}
\end{definition}
When a cost function $\cf_S$ has a DAG-filter $T$,  we say that  $\cf_S$ is DAG-filterable by $T$. Note that any cost function $\cf_S$ has a trivial DAG filter which is composed of a single vertex that defines $\cf_S$ as a cost table  (with size exponential in the arity $|S|$).

In the general case, a DAG-filter (recursively) transforms a cost function into cost functions with smaller scopes until it reaches the ones at the leaves of a DAG, which may be trivial to solve. The (minimum) costs can then be aggregated using the $f_i$ functions at each internal vertex to get the resultant (minimum) cost, through dynamic programming. However, further properties on DAG-filters are required to allow for projections and extensions to operate on the DAG structure.

\begin{definition}[Safe DAG-filter]
  \label{def:sd}
A DAG-filter $T = (V,E)$ for a cost function $\cf_S$  is \emph{safe} iff:
  \begin{itemize}
  \item{} projection and extension are distributive over $f_i$, \ie   for a variable $x \in S$, a cost $\alpha$ and a tuple $\ell \in    \L^S$,
    \begin{itemize}
    \item{} $\smallcf_{S_i}(\ell[S_i]) \oplus \nu_{x,S_i}(\alpha) =
      f_i(\{\smallcf_{S_k}(\ell[S_k]) \oplus \nu_{x,S_k}(\alpha) \mid
      (\smallcf_{S_i}, \smallcf_{S_k}) \in E\})$, and;
    \item{} $\smallcf_{S_i}(\ell[S_i]) \ominus \nu_{x,S_i}(\alpha) =
      f_i(\{\smallcf_{S_k}(\ell[S_k]) \ominus \nu_{x,S_k}(\alpha) \mid
      (\smallcf_{S_i}, \smallcf_{S_k}) \in E\})$,
    \end{itemize}
    where the function $\nu$ is defined as:
   \[
    \nu_{x,S_j}(\alpha) = \left\{ \begin{array}{ll}
        \alpha, & \mbox{ if $x \in S_j$,} \\
        0, & \mbox{ otherwise.} \\
      \end{array} \right.
    \]
  \end{itemize}
\end{definition}

The requirement of a distributive $f_i$ with respect to projection and extension at each vertex in $T$ implies that the structure of the DAG is unchanged after projections and extensions.  Both operations can be distributed down to the leaves. We formally state this as the following theorem. Given a variable $x$, with a value $a\in D(x)$, and a  cost function $\cf_S$, we denote as $\cf'_S$
 the cost function obtained by the application of \Proj{$S,\{x\},(v),\alpha$} on $\cf_S$ if $x\in S$ or  $\cf_S$ otherwise.

 \begin{theorem}
   \label{thm:projOnSafeDecomp}
For a cost function $\cf_S$ with a safe DAG-filter $T=(V,E)$, $W'_S$ has a safe DAG-filter $T^{\prime}=(V^{\prime},E^{\prime})$,  where each $\smallcf_{S_i} \in V^{\prime}$ is defined as:
    \[
    \smallcf'_{S_i} = \left\{ \begin{array}{ll}
				\smallcf_{S_i} \ominus \nu_{x,S_k}(\alpha), & \mbox{  if $\smallcf_{S_i}$ is a leaf of $T$,} \\
				\smallcf_{S_i}, & \mbox{ otherwise.} \\        
      \end{array} \right.
    \] 
and $(\smallcf'_{S_i}, \smallcf'_{S_k}) \in E^{\prime}$ iff $(\smallcf_{S_i},\smallcf_{S_k}) \in E$, \ie $T'$ is isomorphic to $T$. 
Moreover, both $\smallcf'_{S_i} \in V'$ and $\smallcf_{S_i} \in V$ are associated with the same aggregation function $f_i$.  
\end{theorem}
\begin{proof}
  Follows directly from Definition~\ref{def:sd}.
\end{proof}

Two common choices for $f_i$ are $\oplus$ and $\min$, with which distributivity depends on how scopes intersect. In the following, we show that the global cost function is safely DAG-filterable if the internal vertices that are associated with $\oplus$ have children with non-overlapping scopes, and those associated with $\min$ have children with identical scopes.

\begin{proposition}
  \label{thm:oplussafe}
  Any DAG-filter  $T=(V,E)$ for a cost function $\cf_S$ such that
  \begin{itemize}
  \item{} each  $\smallcf_{S_i} \in V$ is associated with the aggregation function $f_i = \bigoplus$;
  
   \item{} for any distinct $\smallcf_{S_j},\smallcf_{S_k} \in V$, which are children of $\smallcf_{S_i}$,  $S_j \cap S_k = \varnothing$.
  \end{itemize}
is safe.
\end{proposition}
 
\begin{proof}	
  We need to show that $\min$, projection and extension are  distributive over $\oplus$.  Since the scopes of the cost  functions do not overlap, $\min$ is distributive over $\oplus$.  We  further show the distributivity with respect to projection  ($\ominus$), while extension ($\oplus$) is similar.  We consider an  internal vertex $\smallcf_{S_i} \in V$.  Given a variable $x \in S_i$,  a cost $\alpha$, and a tuple $\ell \in \L^S$, since the scopes of  the cost functions $\{\smallcf_{S_k} \mid (\smallcf_{S_i},
  \smallcf_{S_k}) \in E\}$ are disjoint, there must exist exactly
  one cost function $\smallcf_{S_j}$ such that $x \in S_j$, \ie: \begin{eqnarray*}
    \smallcf_{S_i}(\ell) \ominus \alpha 
    &=& (\smallcf_{S_j}(\ell[S_j]) \ominus \alpha) \oplus \bigoplus_{k \neq j \wedge (\smallcf_{S_i}, \smallcf_{S_k}) \in E} \smallcf_{S_k}(\ell[S_k]) \\
    &=& \displaystyle \bigoplus_{(\smallcf_{S_i}, \smallcf_{S_k}) \in E} (\smallcf_{S_k}(\ell[S_k]) \ominus \nu_{x,S_k}(\alpha))
  \end{eqnarray*}
  The result follows.
\end{proof}

\begin{proposition}
  \label{thm:minsafe}
  Any DAG-filter  $T=(V,E)$ for a cost function $\cf_S$ such that
  \begin{itemize}
  \item{} each  $\smallcf_{S_i} \in V$ is associated with the aggregation function $f_i = \min$; 
   \item{} $\forall\smallcf_{S_j} \in V$, which are children of $\smallcf_{S_i}$, $S_j = S_i$.
   \end{itemize}
is safe.
\end{proposition}
\begin{proof}
  Since the scopes are completely overlapping,  
  \begin{eqnarray*}
    \min\{\smallcf_{S_i}\} 
    &=& \min_{\ell \in \L^{S_i}}\{\min_{(\smallcf_{S_i}, \smallcf_{S_k}) \in E} \{\smallcf_{S_k}(\ell)\}\} \\
    &=& \min_{(\smallcf_{S_i}, \smallcf_{S_k}) \in E}\{\min_{\ell \in \L^{S_k}} \{\smallcf_{S_k}(\ell)\}\} \\
    &=& f_i(\{\min\{\smallcf_{S_k}\} \mid (\smallcf_{S_i}, \smallcf_{S_k}) \in E\})       
  \end{eqnarray*}
It is trivial to see that projection and extension are  distributive over $f_i$. The result follows.
\end{proof}


We are now ready to define polynomial DAG-filterability of global cost functions. As safe DAG-filters can be exponential in size, we need to  restrict to {safe DAG-filters} of polynomial size by restricting the size of the DAG  to be polynomial and by bounding the arity of the cost functions at the leaves of the DAG.

\begin{definition}[Polynomial DAG-filterability]
  \label{def:polydecom}
  A global cost function $\gc{W\_GCF}(S,A_1,\ldots,A_t)$ is \emph{polynomially DAG-filterable} iff
  \begin{enumerate}
  \item any instance $\cf_S$ of $\gc{W\_GCF}(S,A_1,\ldots,A_t)$ has a safe DAG-filter $T=(V,E)$
  \item{} where $|V|$ is polynomial in the size of the input parameters of $\gc{W\_GCF}(S,A_1,\ldots,A_t)$;
  \item each leaf in $V$ is a unary cost function, and
  \item each aggregation function $f_i$ associated with each internal vertex is polynomial-time computable.
\end{enumerate}
\end{definition}

Dynamic programming can compute the minimum of a polynomially DAG-filterable cost function in a tractable way. Projections and extensions to/from such cost functions can also be distributed to the leaves in $T$. Thus, polynomially DAG-filterable global cost functions are tractable and also tractable projection-safe, as stated below.

\begin{theorem}
    \label{lem:pd-poly}
    A polynomially DAG-filterable global cost function $\gc{W\_GCF}(S,A_1,\ldots,A_t)$ is tractable.
  \end{theorem}
\begin{proof}
Let $\cf_S$ be any instance of $\gc{W\_GCF}(S,A_1,\ldots,A_t)$, and $T = (V,E)$ be a safe DAG-filter for $\cf_S$. 
  Algorithm \ref{algo:min} can be applied to compute $\min\{\cf_S\}$.  The algorithm uses a bottom-up memoization approach.  Algorithm~\ref{algo:min} first sorts $V$ topologically at line \ref{minL1}.  After sorting, all the leaves will be grouped at the end of the sorted sequence, which is then processed in the reversed order at line~\ref{minL2}.  If the vertex is a leaf, the minimum is computed and stored in the table \opSty{Min} at line \ref{minL4}. Otherwise, its minimum is computed by aggregating $\{\opSty{Min}[\smallcf_{S_k}] \mid  (\smallcf_{S_i},\smallcf_{S_k}) \in E\}$, which have been already computed, by the function $f_i$ at line \ref{minL6}. Line \ref{minL7} returns the minimum of the root node.
	
The computation is tractable. Leaves being unary cost functions, line~\ref{minL4} is in $O(d)$, where $d$ is the  maximum domain size. For other vertices, line \ref{minL6}  calls  $f_i$, which is assumed to be polynomial time. The result follows.
\end{proof}

\begin{algorithm}[htb]
\caption{Computing $\min\{\cf_S\}$  \label{algo:min} }
\SetKw{Proc}{Procedure}
\SetKw{Func}{Function}
\SetKw{False}{false}
\SetKw{True}{true}
\SetKwBlock{Start}{}{}
{\small
  \Func{} \procCall{Minimum}($\cf_S$)
  \Start{
    \lnl{minL0}
    Form the corresponding filtering DAG $T = (V,E)$\;
    \lnl{minL1}
    Topologically sort $V$\;
    \lnl{minL2}
    \ForEach{$\smallcf_{S_i} \in V$ in reverse topological order}{
      \lnl{minL3}
      \eIf{$\smallcf_{S_i}$ is a leaf of $T$}{	     
        \lnl{minL4}
        \ArgSty{Min}$[\smallcf_{S_i}] := \procCall{min}\{\smallcf_{S_i}\}$ \;
      }
      {
        \lnl{minL6}
        \ArgSty{Min}$[\smallcf_{S_i}] := f_i(\{\ArgSty{Min}[\smallcf_{S_k}] \mid (\smallcf_{S_i},\smallcf_{S_k}) \in E\})$ \;
      }
    } 	 
    \lnl{minL7}
    \Return{\ArgSty{Min}$[\cf_S]$}\;
  }
}
\end{algorithm}

Note that Algorithm~\ref{algo:min} computes the minimum from scratch
each time it is called. In practice, querying the minimum of cost
function $\cf_S$ when $x_i$ is assigned to $v$ for different values $v$ can be done more efficiently with some pre-processing. 
We define $Min^+[\smallcf_{S_j},x_i,v]$ that stores $\min\{\smallcf_{S_j}(\ell) \mid x_i \in S_j \wedge \ell[x_i] = v\}$. $Min^+[\smallcf_{S_j},x_i,v]$ can be computed similarly to Algorithm~\ref{algo:min} by using the equation:
\[
Min^+[\smallcf_{S_j},x_i,v] = 
\left \{ 
\arraycolsep=0pt%
\begin{array}{ll}
\smallcf_{S_j}(v), & \mbox{ if $\smallcf_{S_j}$ is a leaf of $T$ and $S_j = \{x_i\}$} \\
\min\{\smallcf_{S_j}\}, & \mbox{ if $\smallcf_{S_j}$ is a leaf of $T$ and $S_j \neq \{x_i\}$} \\
f_j(\{Min^+[\smallcf_{S_k}, x_i, v] ) \mid~ & (\smallcf_{S_i},\smallcf_{S_k}) \in E\}), \mbox{ otherwise}
\end{array}\right .
\] 
\noindent
Whenever we have to compute the minimum for $x_i=v$, we simply return $Min^+[\cf_S,x_i,v]$. Computing $Min^+[\cf_S,x_i,v]$ is equivalent to running Algorithm~\ref{algo:min} $nd$ times, where $n$ is the number of variables and $d$ the maximum domain size. However, this can be reduced by incremental computations exploiting the global constraint semantics, as illustrated on the \gc{W\_Grammar}$^{var}$ global cost function in Section~\ref{example-DAG}.

We now show that a polynomially DAG-filterable cost function is tractable projection-safe. The following lemma will be useful. For a  variable $x \in S$ and a value $v \in D(x)$, we denote as $\gc{W'\_GCF}(S,A_1,\ldots,A_t)$ the cost function obtained by applying \Proj{$S,\{x\},(v),\alpha$} to a global cost function $\gc{W\_GCF}(S,A_1,\ldots,A_t)$.

\begin{lemma}
  \label{lemma:pd-safe}
  If a global cost function $\gc{W\_GCF}(S,A_1,\ldots,A_t)$ is
  polynomially DAG-filterable, $\gc{W'\_GCF}(S,A_1,\ldots,A_t)$ is
  polynomially DAG-filterable.
\end{lemma}

\begin{proof}
Suppose $\gc{W\_GCF}(S,A_1,\ldots,A_t)$ is polynomially DAG-filterable. Then any instance $\cf_S$ of it has a safe filtering DAG $T = (V,E)$. By Theorem~\ref{thm:projOnSafeDecomp}, we know that $\cf'_S$, the corresponding instance of $\gc{W'\_GCF}(S,A_1,\ldots,A_t)$, has a safe DAG filter $T^{\prime}$, which is isomorphic to $T$, has polynomial size, and polynomial-time computable $f_i$ associated with each internal vertex. The leaves of $T^{\prime}$ only differ from those of $T$ by a constant. The result follows.        
\end{proof}

\begin{theorem}
  \label{thm:pd-safe}
  A polynomially DAG-filterable global cost function $\gc{W\_GCF}(S,A_1,\ldots,A_t)$ is  tractable projection-safe.
\end{theorem}
\begin{proof}
  Follows directly from Theorem \ref{lem:pd-poly} and Lemma  \ref{lemma:pd-safe}.
\end{proof}

As shown by Theorem~\ref{thm:pd-safe}, a polynomially DAG-filterable cost function $\cf_S$ remains
polynomially DAG-filterable after projection or extension. Algorithm~\ref{pDAGProject} shows how the 
projection is performed from $\cf_S$ and $\cf_i$, where $x_i  \in S$. Lines \ref{dagprjln2} to \ref{dagprjln4}
modify the leaves of the filtering DAG, as suggested by Theorem~\ref{thm:projOnSafeDecomp}. 

Lines \ref{dagprjln4} to \ref{dagprjln8} in Algorithm \ref{pDAGProject} show how incrementality can be achieved.
If $\cf_i$ or $D(x_i)$,  $x_i \in S$, are changed we update the entry $Min[\smallcf_{S_i}]$ 
at line \ref{dagprjln4}, which corresponds to the leaf $\smallcf_{S_i}$, where $x_i \in S_i$. The change propagates upwards in 
lines \ref{dagprjln7} and \ref{dagprjln8}, updating all entries related to the leaf $\smallcf_{S_i}$. 
The table $FW$ can be updated similarly.

\begin{algorithm}[htb]
\caption{Projection from a polynomially DAG-filterable global cost function\label{pDAGProject}}
\SetKw{Proc}{Procedure}
\SetKw{Func}{Function}
\SetKw{False}{false}
\SetKw{True}{true}
\SetKwBlock{Start}{}{}
Precondition: $\cf_S$ is polynomially DAG-filterable with the filtering DAG $T=(V,E)$\;
  \Proc{} \procCall{Project}($S$, $\{x_i\}$, $(v)$, $\alpha$)
  \Start{  
		\lnl{dagprjln1}
			 $\cf_{i}(v)$ := $\cf_{i}(v) \oplus \alpha$ \;
	\lnl{dagprjln2}	
    \ForEach{$\smallcf_{S_j} \in V$ such that $S_j = \{x_i\}$ and $\smallcf_{S_j}$ is a leaf of $T$}{      
		\lnl{dagprjln3}
        $\smallcf_{S_j}(v)$ := $\smallcf_{S_j}(v) \ominus \alpha$ \;      
			\lnl{dagprjln4}			  
						\ArgSty{Min}$[\smallcf_{S_j}] := \procCall{min}\{\smallcf_{S_j}\}$ \;
		}	
		\lnl{dagprjln5}
    Topologically sort $V$\;    
		\lnl{dagprjln6}
    \ForEach{$\smallcf_{S_j} \in V$ in reverse topological order}{      
		\lnl{dagprjln7}
      \If{$\smallcf_{S_j}$ is not a leaf and $x_i \in S_j$}{	                          
			\lnl{dagprjln8}
        \ArgSty{Min}$[\smallcf_{S_j}] := f_i(\{\ArgSty{Min}[\smallcf_{S_k}] \mid (\smallcf_{S_j},\smallcf_{S_k}) \in E\})$ \;
      }
  }
	}
\end{algorithm}


The time complexity of enforcing GAC* on a polynomially DAG-filterable
global cost function heavily depends on preprocessing, as stated in
the following corollary.

\begin{corollary}  
\label{Thm:pDAGProject}
	If the time complexity for pre-computing the table $Min^+$ for 
	a polynomially DAG-filterable cost function $\cf_S$ 
	is $O(K(n,d))$, where $K$ is a function of $n = |S|$ and maximum domain size $d$, then enforcing GAC* on a variable $x_i \in S$ with respect to $\cf_S$ requires $O(K(n,d)+d)$ time.
\end{corollary}
\begin{proof}
  Computing the minimum of $\cf_S$ when $x_i = v$, where $x_i \in S$
  and $v \in D(x_i)$, requires only constant time by looking up from $Min^+$. 
	By Proposition \ref{lem:supporttime}, the time
  complexity is $O(K(n,d) + d)$ time.
\end{proof}

We have presented a new class of tractable projection-safe global cost
functions. Algorithm~\ref{algo:min} gives an efficient algorithm to
compute the minimum cost.  In the next Section, we give an example of
such a global cost function. More examples can be found in the
associated technical report~\cite{TR-Lee2014}.


\section{A Polynomially DAG-filterable Global Cost Function}
\label{example-DAG}

In the following, we show that \gc{W\_Grammar}$^{var}$, \gc{W\_Among}$^{var}$,\gc{W\_Regular}$^{var}$, 
\gc{W\_Max}, and \gc{W\_Min} are {polynomially 
DAG-filterable} using the results from the previous section. 

\subsection{The \gc{W\_Grammar}$^{var}$ Cost Function}

\gc{W\_Grammar}$^{var}$ is the cost function variant of the softened version of the hard global constraint \gc{Grammar}~\cite{KS2010} defined based on a context-free language.
\begin{definition}
  A context-free language $L(G)$ is represented by a
  \emph{context-free grammar} $G=(\Sigma, N, P, A_0)$, where:
  \begin{itemize}
  \item{} $\Sigma$ is a set of terminals;
  \item{} $N$ is a set of non-terminals;
  \item{} $P$ is a set of production rules from $N$ to $(\Sigma \cup
    N)^{*}$, where $*$ is the Kleene star, and;
  \item{} $A_0 \in N$ is a starting symbol. 
  \end{itemize}
  A string $\tau$ belongs to $L(G)$, written as $\tau \in L(G)$ iff $\tau$ can be derived from $G$.
\end{definition}

\noindent
Without loss of generality, we assume that (1) the context-free language $L(G)$ does not contain cycles, and (2) the strings are always of fixed length, representing values in tuples.

Assume $S = \{x_1, \ldots, x_n\}$. We define $\tau_{\ell}$ to be a string formed by a tuple $\ell \in
\L^S$, where the $i^{th}$ character of $\tau_{\ell}$ is $\ell[x_i]$.
The hard constraint \gc{grammar}$(S,G)$ authorizes a tuple $\ell \in
\L^S$ if $\tau_{\ell} \in L(G)$~\cite{KS2010}. Using the violation measure $\mathit{var}$ by Katsirelos \etal~\cite{KNW2011}, the \gc{W\_Grammar}$^{var}$ cost function is defined as follows.

\begin{sloppypar}
  \begin{definition}[\gc{W\_Grammar$^{var}$} \cite{KNW2011}] Given a context-free grammar   $G=(\Sigma, N, P, A_0)$. \emph{\gc{W\_Grammar}$^{var}(S,G)$}  returns $\min\{H(\tau_\ell, \tau_i) \mid \tau_i \in L(G)\}$ for  each tuple $\ell \in \L^S$, where $H(\tau_1, \tau_2)$ returns the Hamming  distance between $\tau_1$ and $\tau_2$.
\end{definition}
\end{sloppypar}

\begin{example}
\label{wgrammarex1}
Consider $S = \{x_1,x_2,x_3,x_4\}$, where $D(x_i) = \{a,b,c\}$ for $i=1\ldots 4$.  Given the grammar $G=(\{a,b,c\}, \{A_0,A,B,C\}, P, S)$ with the following production rules.
\begin{eqnarray*}
	A_0 &\rightarrow& AA \\
    A &\rightarrow& a \mid AA \mid BC \\
    B &\rightarrow& b \mid BB\\    
    C &\rightarrow& c \mid CC
\end{eqnarray*}
\noindent
The  cost returned by \gc{W\_Grammar}$^{var}(S, G)(\ell)$ is $1$ if $\ell = (c,a,b,c)$. The assignment of $x_1$ needs to be changed so that $L(M)$ accepts the corresponding string $aabc$.
\end{example}

\begin{theorem}
  \label{grammar}
  \gc{W\_Grammar}$^{var}(S,G)$ is a {polynomially DAG-filterable} and thus tractable projection-safe global cost function.
\end{theorem}
\begin{proof}
  We adopt the dynamic programming approach similar to the modified  CYK parser~\cite{KNW2011}. Without loss of generality, we assume $G$  is in Chomsky normal form, \ie\  {each production rule} always has the form  $A \rightarrow \alpha$ or $A \rightarrow BC$, where $A \in N$, $B,C \in N \setminus \{A_0\}$ and $\alpha \in \Sigma$.
	
  Define $\smallcf^A_{S_{i,j}} = \mbox{\gc{W\_Grammar}}^{var}(S_{i,j},
  G_A)$, where $i \leq j$, $S_{i,j} = \{x_i \ldots x_j\} \subseteq S$,  and $G_A = (\Sigma, N, P, A)$ for $A \in N$. By definition, \[
  \mbox{\gc{W\_Grammar}}^{var}(S,G)(\ell) =
  \smallcf_{S_{1,n}}^{A_0}(\ell)
  \]
  The base cases $\smallcf^A_{S_{i,i}}$ is defined as follows. Define $\Sigma_A = \{\alpha \mid A \rightarrow \alpha\}$ to be the set of terminals that can be yielded from $A$. 
  \begin{equation}
	\label{eq:geq1}
  \smallcf^A_{S_{i,i}}(\ell) =  \left\{ \begin{array}{ll}
      \displaystyle\min \{U^{\alpha}_i(\ell[x_i]) \mid (A \rightarrow \alpha) \in P\}, & \mbox{ if $\Sigma_A \neq \varnothing$} \\				
      \top, & 
      \mbox{ otherwise}\\																																		
    \end{array}\right. \\		
  \end{equation}

  The unary cost function $U^{\alpha}_i(\ell[x_i])$ is defined as follows.
  \begin{equation}
	\label{eq:geq1.5}
     U^{\alpha}_i(v) = \left\{ \begin{array}{ll}
      0, & \mbox{ if $v = \alpha$;} \\
      1, & \mbox{ otherwise} \\
    \end{array} \right.
  \end{equation}
  \noindent
  Other cost functions $\smallcf^A_{S_{i,j}}$, where $i < j$, are defined as follows. Let $N_A = \{(B,C) \mid A \rightarrow BC\}$  be the set of pairs of non-terminals that are yielded from $A$. 
  \begin{equation}
	\label{eq:geq2}	
     \smallcf^A_{S_{i,j}}(\ell) = \left\{ \begin{array}{ll}
      \displaystyle \min_{k = i,\ldots j-1} \{\smallcf^B_{S_{i,k}}(\ell[S_{i,k}]) \oplus \smallcf^C_{S_{k+1,j}}(\ell[S_{k+1,j}]) \mid (A \rightarrow BC) \in P \}, & \mbox{ if $N_A \neq \varnothing$ }\\
      \top, &   
      \mbox{ otherwise}\\
    \end{array}\right.   
  \end{equation}
  \end{proof}

The associated filtering DAG $(V,E)$ is illustrated in Figure~\ref{grammarDAG} on Example \ref{wgrammarex1}. In Figure~\ref{grammarDAG}, leaves are indicated by double circles, corresponding to the unary cost function in equation \ref{eq:geq1.5}.
Vertices with $\min$ or $\oplus$ aggregators are indicated by rectangles and circles respectively, corresponding to cost functions  
$\smallcf^A_{S_{i,j}}$ in equation~\ref{eq:geq2} if $i \neq j$, or equation~\ref{eq:geq1} otherwise. As shown in Figure~\ref{grammarDAG}, the root node \gc{W\_Grammar} is first split by the production rule $A_0 \rightarrow AA$. One of its children $\omega^{A}_{S_{1,1}}$ leads to the leaf $U^{a}_1$ according to the production rule $A \rightarrow a$. The DAG uses only $\oplus$ or $\min$ as aggregations and they satisfy the preconditions that allow to apply propositions~\ref{thm:oplussafe} and~\ref{thm:minsafe}. The cost function is therefore safely DAG-filterable. 
Moreover, the corresponding DAG $(V,E)$ has size $|V| = O(|P| \cdot |S|^3)$ polynomial in the size of the input. The leaves are unary functions  $\{U^{\alpha}_i\}$ and by Theorem \ref{thm:pd-safe}, the result follows.

\begin{figure}[htp]
\resizebox{1.0\textwidth}{!}{
\begin{tikzpicture}[>=latex',line join=bevel,]
\node (others4) at (284bp,90bp) [draw,draw=none] {$...$};
  \node (U4c) at (420bp,156bp) [draw,double,double distance=1pt,circle] {$U^c_4$};
  \node (sumw11Bw22C) at (284bp,61bp) [draw,circle] {$$};
  \node (U2b) at (556bp,197bp) [draw,double,double distance=1pt,circle] {$U^b_2$};
  \node (others3) at (284bp,148bp) [draw,draw=none] {$...$};
  \node (sumw12Aw34A) at (158bp,141bp) [draw,circle] {$$};
  \node (w24A) at (216bp,184bp) [draw,rectangle] {$\omega^{A}_{S_{2,4}}$};
  \node (sumw23Bw44C) at (284bp,177bp) [draw,circle] {$$};
  \node (sumw22Bw33B) at (420bp,197bp) [draw,circle] {$$};
  \node (w11B) at (352bp,19bp) [draw,rectangle] {$\omega^{B}_{S_{1,1}}$};
  \node (w11A) at (216bp,236bp) [draw,rectangle] {$\omega^{A}_{S_{1,1}}$};
  \node (w22B) at (488bp,197bp) [draw,rectangle] {$\omega^{B}_{S_{2,2}}$};
  \node (others2) at (284bp,206bp) [draw,draw=none] {$...$};
  \node (w34A) at (216bp,141bp) [draw,rectangle] {$\omega^{A}_{S_{3,4}}$};
  \node (U2c) at (420bp,66bp) [draw,double,double distance=1pt,circle] {$U^c_2$};
  \node (sumw33Bw44C) at (284bp,119bp) [draw,circle] {$$};
  \node (U3b) at (556bp,129bp) [draw,double,double distance=1pt,circle] {$U^b_3$};
  \node (U1a) at (284bp,244bp) [draw,double,double distance=1pt,circle] {$U^a_1$};
  \node (U1b) at (420bp,16bp) [draw,double,double distance=1pt,circle] {$U^b_1$};
  \node (w22C) at (352bp,63bp) [draw,rectangle] {$\omega^{C}_{S_{2,2}}$};
  \node (w44C) at (352bp,156bp) [draw,rectangle] {$\omega^{C}_{S_{4,4}}$};
  \node (w33B) at (488bp,129bp) [draw,rectangle] {$\omega^{B}_{S_{3,3}}$};
  \node (others) at (158bp,112bp) [draw,draw=none] {$...$};
  \node (sumw24Aw11A) at (158bp,184bp) [draw,circle] {$$};
  \node (wgrammar) at (57bp,141bp) [draw,rectangle] {$\omega^{A_0}_{S_{1,4}}=W\_Grammar^{var}$};
  \node (w23B) at (352bp,195bp) [draw,rectangle] {$\omega^{B}_{S_{2,3}}$};
  \node (w12A) at (216bp,90bp) [draw,rectangle] {$\omega^{A}_{S_{1,2}}$};
  \draw [->] (w22C) ..controls (373.49bp,63.931bp) and (383.12bp,64.369bp)  .. (U2c);
  \draw [->] (sumw22Bw33B) ..controls (428.94bp,189.13bp) and (433.79bp,184.29bp)  .. (438bp,180bp) .. controls (448.85bp,168.93bp) and (460.92bp,156.37bp)  .. (w33B);
  \draw [->] (w12A) ..controls (240.41bp,90bp) and (255bp,90bp)  .. (others4);
  \draw [->] (sumw33Bw44C) ..controls (317.29bp,120.6bp) and (417.61bp,125.56bp)  .. (w33B);
  \draw [->] (w24A) ..controls (240.52bp,191.84bp) and (255.31bp,196.77bp)  .. (others2);
  \draw [->] (sumw33Bw44C) ..controls (299.09bp,126.9bp) and (315.54bp,136.12bp)  .. (w44C);
  \draw [->] (wgrammar) ..controls (123.77bp,141bp) and (133.05bp,141bp)  .. (sumw12Aw34A);
  \draw [->] (w23B) ..controls (376.52bp,195.71bp) and (391.31bp,196.16bp)  .. (sumw22Bw33B);
  \draw [->] (sumw23Bw44C) ..controls (299.45bp,180.94bp) and (314.94bp,185.17bp)  .. (w23B);
  \draw [->] (sumw24Aw11A) ..controls (171.85bp,184bp) and (182.13bp,184bp)  .. (w24A);
  \draw [->] (w34A) ..controls (240.41bp,143.48bp) and (255bp,145.03bp)  .. (others3);
  \draw [->] (sumw24Aw11A) ..controls (171bp,195.14bp) and (185.08bp,208.22bp)  .. (w11A);
  \draw [->] (w33B) ..controls (509.49bp,129bp) and (519.12bp,129bp)  .. (U3b);
  \draw [->] (w24A) ..controls (240.52bp,181.51bp) and (255.31bp,179.94bp)  .. (sumw23Bw44C);
  \draw [->] (wgrammar) ..controls (100.49bp,159.46bp) and (126.14bp,170.59bp)  .. (sumw24Aw11A);
  \draw [->] (sumw23Bw44C) ..controls (299.45bp,172.4bp) and (314.94bp,167.47bp)  .. (w44C);
  \draw [->] (sumw12Aw34A) ..controls (171.85bp,141bp) and (182.13bp,141bp)  .. (w34A);
  \draw [->] (wgrammar) ..controls (109.34bp,125.96bp) and (127.74bp,120.57bp)  .. (others);
  \draw [->] (w22B) ..controls (509.49bp,197bp) and (519.12bp,197bp)  .. (U2b);
  \draw [->] (sumw22Bw33B) ..controls (435.37bp,197bp) and (450.63bp,197bp)  .. (w22B);
  \draw [->] (w34A) ..controls (240.75bp,133.08bp) and (255.94bp,128.02bp)  .. (sumw33Bw44C);
  \draw [->] (sumw11Bw22C) ..controls (299.37bp,61.435bp) and (314.63bp,61.898bp)  .. (w22C);
  \draw [->] (w44C) ..controls (373.49bp,156bp) and (383.12bp,156bp)  .. (U4c);
  \draw [->] (w11B) ..controls (373.49bp,18.069bp) and (383.12bp,17.631bp)  .. (U1b);
  \draw [->] (w11A) ..controls (237.38bp,238.47bp) and (246.84bp,239.62bp)  .. (U1a);
  \draw [->] (sumw12Aw34A) ..controls (170.91bp,130.15bp) and (184.75bp,117.55bp)  .. (w12A);
  \draw [->] (sumw11Bw22C) ..controls (298.55bp,52.374bp) and (315.54bp,41.566bp)  .. (w11B);
  \draw [->] (w12A) ..controls (240.97bp,79.467bp) and (256.58bp,72.607bp)  .. (sumw11Bw22C);
\end{tikzpicture}
}
\caption{The DAG corresponding to \gc{W\_Grammar}$^{var}$ \label{grammarDAG} }
\end{figure}
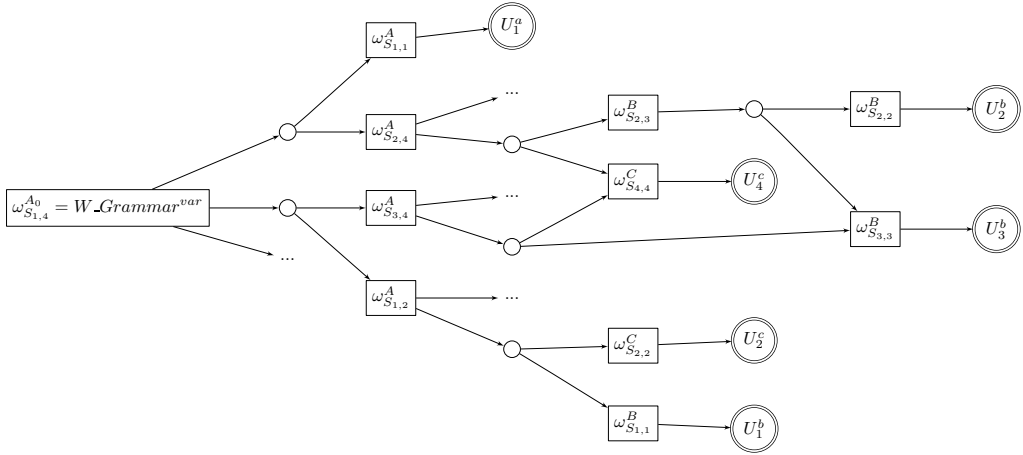

Note that Theorem \ref{grammar} also gives a proof that
\gc{W\_Regular}$^{var}$ is tractable projection-safe. Indeed, a finite
state automaton, defining a regular language, can be transformed into
a grammar with the number of non-terminals and production rules
polynomial in the number of states in the automaton. Then,
\gc{W\_Among}$^{var}$ is also tractable projection-safe since the
tuples satisfying an \gc{Among} global constraint can be represented
using a compact finite state counting automaton~\cite{beldiceanu2005}.

\begin{algorithm}[htb]

  \caption{Finding the minimum of \gc{W\_Grammar}$^{var}$}
  \label{algo:grammar}
  \SetKwBlock{Start}{}{}
  \SetKw{Func}{Function}
  \SetKw{Proc}{Procedure}    
  
  \Func{} \opSty{GrammarMin}($S,G$)
  \Start{      
    \lnl{grammarL1}
    \For{$i := 1$ \KwTo $n$}{
      \lnl{grammarL2}
      \For{$c \in \Sigma$}{
        \lnl{grammarL6}	
        $u[i,c] := \min\{U^{c}_i\}$\;           	           	           
      }
    }
			  \lnl{grammarL10}
    \For{$i := 1$ \KwTo $n$}{
      \lnl{grammarL11}
      \lForEach{$A \in N$}{$f[i,i,A] := \top$\;}
      \lnl{grammarL11.5}
      \ForEach{$(A, a)$ such that $(A \mapsto a) \in P$}{
        \lnl{grammarL12}
        $f[i,i,A] = \min\{f[i,i,A], u[i,a]\}$ \;
      }
     }
			\lnl{grammarL12.5}
    \Return{\opSty{GrammarPartialMin}($S,G,1$)}\;        
    }		
		\Func{} \opSty{GrammarPartialMin}($S,G,start$)
		\Start{        
				\lnl{grammarL13}
				\For{$len := 2$ \KwTo $n$}{
					\lnl{grammarL14}
					\For{$i := start$ \KwTo $n-len+1$}{
					        \lnl{grammarL15}
									$j := i+len-1$ \;            
												\lnl{grammarL15.5}
												\lForEach{$A \in N$}{$f[i,j,A] := \top$\;}
												\lnl{grammarL16}
												\ForEach{$(A, A_1, A_2)$ such that $(A \mapsto A_1A_2) \in P$}{
																	\lnl{grammarL17}
																	\For{$k := i$ \KwTo $j-1$} {
																					\lnl{grammarL17.5}
																					$f[i,j,A] := \min\{f[i,j,A], f[i,k,A_1] \oplus f[k+1,j,A_2]\}$ \;
																	}
								  }
					}
				}					
				\lnl{grammarL18}
        \Return{$f[1,n,A_0]$}\;          
		}
			
\end{algorithm}

\noindent Function \opSty{GrammarMin} in Algorithm~\ref{algo:grammar} computes the minimum of \gc{W\_Grammar}$^{var}(S,G)$. We first compute the minimum of the unary cost functions in the table $u[i,c]$ at lines \ref{grammarL1} to \ref{grammarL6}. The table $f$ of size $n \times n \times |N|$ is filled up in two separate for-loops: one at line \ref{grammarL10} according to the equation \ref{eq:geq1}, and another one at line \ref{grammarL17} for the equation \ref{eq:geq2}. The result is returned at line \ref{grammarL12.5}.

\begin{theorem}
  \label{grammarMinTime}
  The function \opSty{GrammarMin} in Algorithm~\ref{algo:grammar} computes the minimum of the global cost function \gc{W\_Grammar}$^{var}(S,G=(\Sigma,N,P,A_0))$ in time $O(nd \cdot |\Sigma| + n^3 \cdot |P|)$,
where $n = |S|$ and $d$ is the maximum domain size. 
\end{theorem}
\begin{proof}
  Lines \ref{grammarL1} to \ref{grammarL6} take $O(nd \cdot |\Sigma|)$.  The first
  for-loop at lines \ref{grammarL10} to \ref{grammarL12} requires $O(n \cdot |P|)$, while the second one at lines \ref{grammarL13} to
  \ref{grammarL17} requires $O(n^3 \cdot |P|)$. The overall time  complexity is $O(nd\cdot |\Sigma| + n \cdot |P| + n^3 \cdot |P|) = O(nd\cdot |\Sigma| + n^3
  \cdot |P|)$.
\end{proof}

As for incrementality, Algorithm~\ref{algo:grammarIncremental} gives the pre-processing performed on top of Algorithm~\ref{algo:grammar}, based on the weighted CYK propagator used in Katsirelos \etal~\cite{KNW2011}.
We compute the table $f$ at line \ref{grammarIncL1} using Algorithm~\ref{algo:grammar}.
Then we compute the table $F$ at lines \ref{grammarIncL6} to \ref{grammarIncL14} using the top-down approach.  For each production $A \mapsto A_1A_2$, lines \ref{grammarIncL12} and \ref{grammarIncL14} compute the maximum possible costs from their neighbors. An additional table $marked[i,j,A]$ is used to record whether the symbol $A$ is accessible when deriving sub-strings at positions $i$ to $j$ in $G$. Each time we need to compute the minimum for $x_i = v$, we just return $\min\{U^{\alpha}_i(v) \ominus F[i,i,A] \oplus f[0,n - 1,A_0]  \mid (A \mapsto v) \in P \wedge marked[i,i,A]\}$, or $\top$ if such production does not exist.

\begin{algorithm}[htb]

  \caption{Pre-computation for \gc{W\_Grammar}$^{var}$}
  \label{algo:grammarIncremental}
  
  \SetKwBlock{Start}{}{}
  \SetKw{Func}{Function}
  \SetKw{Proc}{Procedure}    
  \SetKw{KwDownTo}{down to}
	\SetKw{False}{false}
\SetKw{True}{true}
    
		\Proc{} \opSty{GrammarPreCompute}($S,G$)
  \Start{     
	\lnl{grammarIncL1}
    $F[1,n,A_0]$ := \opSty{GrammarMin}($S,G$)\;
		\lnl{grammarIncL2}
		\For{$i := 1$ \KwTo $n$}{      
		\lnl{grammarIncL3}
      \For{$j := i$ \KwTo $n$}{ 
			\lnl{grammarIncL3.1}
						\ForEach{$A \in N$}{  
						\lnl{grammarIncL3.2}
						      $F[i,j,A] := -\top$\; 
									\lnl{grammarIncL4}
									$marked[i,j,A]$ := \False\;
						 }
				}
    }		
		\lnl{grammarIncL5}
    $marked[1,n,A_0]$ := \True\;		
		\lnl{grammarIncL6}
    \For{$len := n$ \KwDownTo $2$}{		
		\lnl{grammarIncL7}
      \For{$i := 1$ \KwTo $n-len+1$}{			
			\lnl{grammarIncL8}			
          $j := i+len-1$ \;					
					\lnl{grammarIncL9}
          \ForEach{$(A, A_1, A_2)$ such that $(A \mapsto A_1A_2) \in P \wedge marked[i,j,A]$} {					          
					\lnl{grammarIncL10}
               \For{$k := i$ \KwTo $j$}{	
							\lnl{grammarIncL11}
               		$marked[i,k,A_1]$ := \True\;
									\lnl{grammarIncL12}
                    $F[i,k,A_1]$ := \opSty{max}($F[i,k,A_1], F[i,j,A] \ominus f[k + 1,j,A_2]$)\; 
                  \lnl{grammarIncL13}  
                    $marked[k + 1,j,A_2]$ := \True\;
										\lnl{grammarIncL14}  
                    $F[k + 1,j,A_2]$ := \opSty{max}($F[k + 1,j,A_2], F[i,j,A] \ominus f[i,k,A_1]$)\;
               }                 
     }
		}
		}
  }
\end{algorithm}

\begin{corollary}\label{grammarTime} 
  Given $\cf_S = \mbox{\gc{W\_Grammar}}^{var}(S,G=(\Sigma,N,P,A_0))$. 
  Enforcing GAC* on a variable $x_i \in S$ with respect to \gc{W\_Grammar} requires 
	$O(nd \cdot |\Sigma| + n^3 \cdot |P|)$
time, where $n = |S|$ and $d$ is the maximum domain size.    
\end{corollary}
\begin{proof}  
	Using a similar argument to that in the proof of Theorem~\ref{grammarMinTime}, Algorithm~\ref{algo:grammarIncremental} 
	 requires $O(nd \cdot |\Sigma| + n^3 \cdot |P|)$ time.
	The result follows directly from Corollary~\ref{Thm:pDAGProject} and Theorem~\ref{grammarMinTime}.     
\end{proof}


Algorithm~\ref{algo:GrammarProject} shows how projection is performed between $\gc{W\_Grammar}^{var}$ and $\cf_p$, and how incrementally can be achieved.
Line \ref{gprjln3} modifies the leaves $U^{c}_p$ for each $c \in \Sigma$, while lines \ref{gprjln4} and \ref{gprjln5} 
update the corresponding entries in the tables $u$ and $f$ respectively. 
The change is propagated up in $f$ at line \ref{gprjln6}, corresponding to derivation of sub-strings with positions from $p$ to the end in $G$.

\begin{algorithm}[ht]
\caption{Projection from $\mbox{\gc{W\_Grammar}}^{var}(S,G=(\Sigma,N,P,A_0))$ \label{algo:GrammarProject} }
\SetKw{Proc}{Procedure}
\SetKw{Func}{Function}
\SetKw{False}{false}
\SetKw{True}{true}
\SetKwBlock{Start}{}{}
  \Proc{} \procCall{GrammarProject}($S$, $\{x_p\}$, $(v)$, $\alpha$)
  \Start{  
		\lnl{gprjln1}
			 $\cf_{p}(v)$ := $\cf_{p}(v) \oplus \alpha$ \;
			\lnl{gprjln2}
			 \For{$c \in \Sigma$}{        
			\lnl{gprjln3}
        $U^{c}_p(v) := U^{c}_p(v) \ominus \alpha$\;           	           	           
				\lnl{gprjln4}
				$u[p,c] := \min\{U^{c}_p\}$\;  
      }
			  \lnl{gprjln5}
      \lForEach{$(A, a)$ such that $(A \mapsto a) \in P$}{        
        $f[p,p,A] = \min\{f[p,p,A], u[i,a]\}$ \;
      }			    	
	    \lnl{gprjln6}
			\procCall{GrammarPartialMin}($S, G,p$)\;     		
			\lnl{gprjln7}
			\procCall{GrammarPreCompute}($S, G$)\;     		
  }	
\end{algorithm}

\subsection{The \gc{W\_Among}$^{var}$ Cost Function}

\gc{W\_Among}$^{var}$ is the cost function variant of the softened version of \gc{Among} using the corresponding variable-based violation measure~\cite{ACNA2008}.

\begin{definition}\cite{ACNA2008} 
  Given a set of values $V$, a lower bound $lb$ and an upper bound  $ub$ such that $0 \leq lb \leq ub \leq |S|$.  \emph{\gc{W\_Among}$^{var}(S, lb, ub, V)$} returns $max\{0,
  lb-t(\ell,V), t(\ell,V)-ub\}$, where $t(\ell, V) = |\{i \mid
  \ell[x_i] \in V\}|$ for each tuple $\ell \in \L(S)$.
\end{definition}

\begin{example}
  \label{wamongex1}
  Consider $S = \{x_1,x_2,x_3\}$, where $D(x_1) = D(x_2) = D(x_3) = \{a,b,c,d\}$.  
  The cost returned by \gc{W\_Among}$^{var}(S, 1, 2, \{a,b\})(\ell)$ is:
  \begin{itemize}
  \item{} $0$ if $\ell = (a,b,c,d)$;
  \item{} $1$ if $\ell = (c,d,c,d)$;
  \item{} $2$ if $\ell = (a,b,a,b)$;
  \end{itemize}	 
\end{example}

\begin{theorem}
  \label{among}
  \gc{W\_Among}$^{var}(S,lb, ub, V)$ is polynomially DAG-filterable and  thus tractable projection-safe.
\end{theorem}
\begin{proof}
We first define two base cases $U^V_i$ and $\overline{U}^V_i$. The  function $U^V_i$ is the cost function on $x_i$ defined as:  \[
  U^V_i(v) = \left\{ \begin{array}{ll}
      0, & \mbox{ if $v \in V$;} \\
      1, & \mbox{ otherwise} \\
    \end{array} \right.
  \]
  \noindent
  and $\overline{U}^V_i(v) = 1-U^V_i$ is its negation.
  
We construct \gc{W\_Among}$^{var}$ based on $U^V_i$ and $\overline{U}^V_i$.  Define $\smallcf^j_{S_i} =
  \mbox{\gc{W\_Among}}^{var}(S_i,j, j, V)$, where $S_i = \{x_1,
  \ldots, x_i\} \subseteq S$. By definition, $S_i = S_{i-1} \cup
  \{x_i\}$ and $S_0 = \varnothing$. \gc{W\_Among}$^{var}(S, lb, ub,
  V)$ can be represented by the sub-cost functions $\smallcf^j_{S_i}$  as: \[
  \mbox{\gc{W\_Among}}^{var}(S, lb, ub, V)(\ell) = \min_{lb \leq j \leq
    ub}\{\smallcf^j_{S_{n}}(\ell)\}
  \]
  \noindent
  and each $\smallcf^j_{S_i}$ can be represented as:
\[\begin{array}{llll}
    \smallcf^j_{S_0}(\ell) &=& j \\
    \smallcf^0_{S_i}(\ell) &=& \smallcf^0_{S_{i-1}}(\ell[S_{i-1}]) \oplus \overline{U}^V_i(\ell[x_i]) & \mbox{ for $i > 0$}\\
    \smallcf^j_{S_i}(\ell) &=& \min \left\{ \begin{array}{l}
        \smallcf^{j-1}_{S_{i-1}}(\ell[S_{i-1}]) \oplus U^V_i(\ell[x_i]) \\
        \smallcf^{j}_{S_{i-1}}(\ell[S_{i-1}]) \oplus \overline{U}^{V}_i(\ell[x_i]) \\
    	\end{array}\right. &  \mbox{ for $j > 0$ and $i>0$}\\ 
    \end{array}
    \]
    The equations form a DAG $(V,E)$, as illustrated in Figure~\ref{amongDAG} using Example \ref{wamongex1}. In Figure \ref{amongDAG}, leaves are indicated by double-lined  circles. Vertices with $\min$ or $\oplus$ aggregators are    indicated by rectangles and circles respectively. The DAG has a  number of vertices $|V| = O(ub \cdot n) = O(n^2)$ and uses only using $\oplus$ or $\min$ as aggregations with proper scopes properties. 
		By propositions~\ref{thm:oplussafe} and~\ref{thm:minsafe}, \gc{W\_Among}$^{var}$ is safely DAG-filterable. 		
		Moreover, each leaf of the DAG is a  unary cost functions. By Theorem \ref{thm:pd-safe}, the result  follows 
	
    \begin{figure}[htp]
      \centering \psfrag{wamongvar}[cc][cc]{\tiny{$W\_Among^{var}$}}
      \psfrag{w2s3}[cc][cc]{\tiny{$\omega^2_{S_3}$}}
      \psfrag{w1s3}[cc][cc]{\tiny{$\omega^1_{S_3}$}}
      \psfrag{w1s2}[cc][cc]{\tiny{$\omega^1_{S_2}$}}
      \psfrag{w0s2}[cc][cc]{\tiny{$\omega^0_{S_2}$}}
      \psfrag{w1s1}[cc][cc]{\tiny{$\omega^1_{S_1}$}}
      \psfrag{w1s0}[cc][cc]{\tiny{$\omega^1_{S_0}$}}
      \psfrag{w0s1}[cc][cc]{\tiny{$\omega^0_{S_1}$}}
      \psfrag{w0s0}[cc][cc]{\tiny{$\omega^0_{S_0}$}}
      \psfrag{uv3}[cc][cc]{\tiny{$U^V_3$}}
      \psfrag{baruv3}[cc][cc]{\tiny{$\overline{U}^V_3$}}
      \psfrag{uv2}[cc][cc]{\tiny{$U^V_2$}}
      \psfrag{baruv2}[cc][cc]{\tiny{$\overline{U}^V_2$}}
      \psfrag{uv1}[cc][cc]{\tiny{$U^V_1$}}
      \psfrag{baruv1}[cc][cc]{\tiny{$\overline{U}^V_1$}}
	\includegraphics[width=0.7\textwidth]{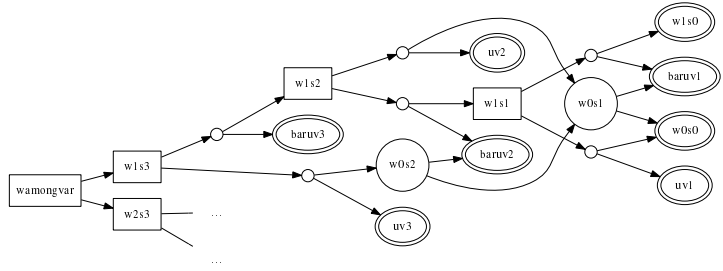}
	\caption{The DAG-filter corresponding to \gc{W\_Among}$^{var}$ \label{amongDAG} }
	\end{figure}
\end{proof}

Function \opSty{AmongMin} in Algorithm \ref{algo:among} computes the minimum of the \gc{W\_Among}$^{var}(S, lb, ub, V)$ cost function according to Theorem \ref{among}. Lines \ref{amongL1} to \ref{amongL7} compute the minimum costs returned by each additional unary cost functions and store at the arrays $\overline{u}$ and $u$. Lines \ref{amongL8} to \ref{amongL13} builds up the results by filling the table $f$ of size $n \times ub$ according to the formulation stated in the proof of Theorem~\ref{among}, and return
the result at line \ref{amongL14}. The complexity is stated in Theorem \ref{amongMinTime} as follows.

\begin{algorithm}[htb]
  \caption{Finding the minimum of \gc{W\_Among}$^{var}$}
  \label{algo:among}
  \SetKwBlock{Start}{}{}
  \SetKw{Func}{Function}
  \SetKw{Proc}{Procedure}    
  
  \Func{} \opSty{AmongMin}($S,lb, ub, V$)
  \Start{	 
    \lnl{amongL1}						    		    		
    \For{$i = 1$ \KwTo $n$}{    		  
      \lnl{amongL2}	
      $\overline{u}[i] := \min\{\overline{U}^V_i\}$\;    		  
      \lnl{amongL7}	
      $u[i] := \min\{U^V_i\}$\;    		  
    }    
    \lnl{amongL8}						
    \lFor{$j = 0$ \KwTo $ub$}{
      $f[0,j] := j$ \;
    }
    \lnl{amongL10}						
    \For{$i = 1$ \KwTo $n$}{
      \lnl{amongL11}						
      $f[i,0] := f[i-1,0] \oplus \overline{u}[i]$ \;
      \lnl{amongL12}						
      \For{$j = 1$ \KwTo $ub$}{
        \lnl{amongL13}						 	
        $f[i,j] := \min\{f[i-1,j-1] \oplus u[i],
        f[i-1,j] \oplus \overline{u}[i]\}$ \;
      }
    }
    \lnl{amongL14}						 	
    \Return{$\min_{lb \leq j \leq ub}\{f[n,j]\}$} \;
  }
\end{algorithm}

\begin{theorem}\label{amongMinTime}
  Function \opSty{AmongMin} in Algorithm \ref{algo:among} computes the  minimum of \gc{W\_Among}$^{var}(S,lb, ub, V)$ and requires  $O(n(n+d))$, where $n = |S|$ and $d$ is the maximum domain size.
\end{theorem}
\begin{proof}
  Lines \ref{amongL1} to \ref{amongL7} in Algorithm \ref{algo:among}  take $O(nd)$. Lines \ref{amongL8} to \ref{amongL14} requires $O(n
  \cdot ub)$. Since $ub$ is bounded by $n$, the result follows.
\end{proof}

\begin{corollary}  \label{amongTime}
Given $\cf_S = \mbox{\gc{W\_Among}}^{var}(S,lb, ub, V)$. Enforcing GAC* on a variable $x_i \in S$ with respect to $\cf_S$ requires $O(nd(n+d))$.
\end{corollary}
\begin{proof}  
	Follow directly from  Corollary~\ref{Thm:pDAGProject} and Theorem~\ref{amongMinTime} .
  \end{proof}

\subsection{The \gc{W\_Regular}$^{var}$ Cost Function}

\gc{W\_Regular}$^{var}$ is the cost function variant of the softened version of the hard constraint \gc{Regular}~\cite{GP2004} based on a regular language.

\begin{definition}
  A regular language $L(M)$ is represented by \emph{a deterministic finite state  automaton} (DFA) $M=(Q,\Sigma, \delta, q_0, F)$, where:
  \begin{itemize}
  \item{} $Q$ is a set of states;
  \item{} $\Sigma$ is a set of characters;
  \item{} The transition function $\delta$ is defined as: $\delta: Q
    \times \Sigma \mapsto Q$;
  \item{} $q_0 \in Q$ is the initial state, and;
  \item{} $F \subseteq Q$ is the set of final states.
\end{itemize}
A string $\tau$ lies in $L(M)$, written as $\tau \in L(M)$, iff $\tau$ can lead the transitions from $q_0$ to $q_f \in F$ in $M$
\end{definition}

The hard constraint \gc{Regular}$(S,M)$ authorizes a tuple $\ell \in
\L(S)$ if $\tau_{\ell} \in L(M)$, where $\tau_{\ell}$ is the string formed from $\ell$ \cite{GP2004}. A \gc{W\_Regular}$^{var}$ cost function is defined as follows, derived from the variable-based violation measure given by \CiteInLine{NMT2004} and \CiteInLine{WGL2006}.

\begin{definition} \cite{NMT2004,WGL2006}\label{def:regular}
  Given a DFA $M=(Q,\Sigma, \delta, q_0,
  F)$. The cost function \emph{\gc{W\_Regular}$^{var}(S,M)$} returns  $\min\{H(\tau_{\ell},\tau_i) \mid \tau_i \in L(M)\}$ for each tuple  $\ell \in \L(S)$, where $H(\tau_1, \tau_2)$ returns the Hamming  distance between $\tau_1$ and $\tau_2$.
\end{definition}

\begin{figure}[htp]
\centering
\psfrag{a}{\small{$a$}}
\psfrag{b}{\small{$b$}}
\psfrag{q0}{\small{$q_0$}}
\psfrag{q1}{\small{$q_1$}}
\psfrag{q2}{\small{$q_2$}}
\includegraphics[width=0.3\textwidth]{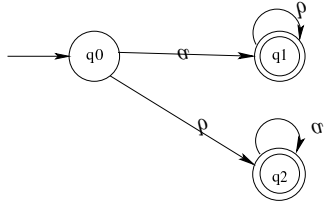}
\caption{The graphical representation of a DFA. \label{DFA}}
\end{figure}

\begin{example}
  \label{wregularex1}
  Consider $S = \{x_1,x_2,x_3\}$, where $D(x_1) = \{a\}$ and $D(x_2) =
  D(x_3) = \{a,b\}$.  Given the DFA $M$ shown in Figure \ref{DFA}. The  cost returned by \gc{W\_Regular}$^{var}(S, M)(\ell)$ is $1$ if $\ell
  = (a,b,a)$. The assignment of $x_3$ need changed in the tuple  $(a,b,a)$ so that $L(M)$ accepts the corresponding string $aba$.
\end{example}

\begin{theorem}
  \label{regular}
  \gc{W\_Regular}$^{var}(S,M)$ is polynomially DAG-filterable and thus  tractable projection-safe.
\end{theorem}
\begin{proof}

  \gc{W\_Regular}$^{var}$ can be represented as a  DAG~\cite{NMT2004,WGL2006,LL2012asa}, which directly gives a  polynomial DAG-filter. In the following, we reuse the symbols $S_i$ and $U^V_i(v)$, which are defined in the proof of  Theorem \ref{among}.
	
  Define $\smallcf^j_{S_i}$ to be the cost function  \gc{W\_Regular}$^{var}(S_i,M_j)$, where $M_j$ is the DFA $(Q,
  \Sigma, \delta, q_0, \{q_j\}))$. \gc{W\_Regular}$^{var}(S,M)$ can be  represented as: \[
  \mbox{\gc{W\_Regular}}^{var}(S,M)(\ell) = \min_{q_j \in
    F}\{\smallcf^j_{S_n}(\ell)\}
  \]
  The base cases $\smallcf^j_{S_0}$ are defined as:	 \[
  \smallcf^j_{S_0}(\ell) = \left\{ \begin{array}{ll}
      0, & \mbox{ if $j=0$} \\
      \top, & \mbox{ otherwise} \\			
    \end{array}\right. 
  \]
  Other sub-cost functions $\smallcf^j_{S_i}$, where $i > 0$, are
  defined as follows. Define $\delta_{q_j} = \{(q_i,v) \mid
  \delta(q_i,v) = q_j\}$. If $\delta_{q_j} = \varnothing$, no  transition can lead to $q_j$. \[
  \smallcf^j_{S_i}(\ell) = \left\{ \begin{array}{ll}
      \displaystyle \min_{\delta(q_k,v) = q_j}\{\smallcf^k_{S_{i-1}}(\ell[S_{i-1}]) \oplus U^{\{v\}}_i(\ell[x_i])\}, & \mbox{ if $\delta_{q_j}  \neq \varnothing$} \\
      \top, & \mbox{ otherwise} \\
    \end{array}\right. 
  \]
  \ignore{
    \[\begin{array}{lll}
      \smallcf^j_{S_0}(\ell) &=& \left\{ \begin{array}{ll}
          0, & \mbox{ if $j=0$} \\
          \top, & \mbox{ otherwise} \\			
        \end{array}\right. \\	
      \smallcf^j_{S_i}(\ell) &=& \left\{ \begin{array}{ll}
          \top, & \mbox{ if there does not exist $q_k$ and $v$ such that $\delta(q_k,v) = q_j$} \\		
          \min & 
          \{\smallcf^k_{S_{i-1}}(\ell[S_{i-1}]) \oplus U^{\{v\}}_i(\ell[x_i]) \mid \delta(q_k,v) = q_j\}, \mbox{ otherwise} \\
        \end{array}\right. \\		
    \end{array}
    \]
  }
		
  The corresponding DAG is shown in Figure \ref{regularDAG}, based on  Example \ref{wregularex1}. Again, the same notation as in Figure~\ref{amongDAG} is used. The DAG has a number of vertices $|V| =
  O(|S| \cdot |Q|)$, and its leaves are unary functions. The DAG-filter is thus polynomial, and, 
	by Theorem \ref{thm:pd-safe}, the result follows.
\end{proof}
	
Together with Theorem $6.11$ by \CiteInLine{LL2012asa} which showed that this global cost function was flow-based tractable, this result gives another reasoning for its tractability. Theorem \ref{regular} also gives another proof of the tractable projection-safety of \gc{W\_Among}$^{var}$ \cite{ACNA2008}. The \gc{Among} global constraint \cite{BC1994} can be modeled by the \gc{Regular} global constraint \cite{GP2004}, which the size of the corresponding DFA is polynomial in the size of the scope.

\begin{figure}[htp]
  \centering
  \psfrag{wregularvar}[cc][cc]{\tiny{$W\_Regular^{var}$}}
  \psfrag{w2s3}[cc][cc]{\tiny{$\omega^2_{S_3}$}}
  \psfrag{w1s3}[cc][cc]{\tiny{$\omega^1_{S_3}$}}
  \psfrag{w1s2}[cc][cc]{\tiny{$\omega^1_{S_2}$}}
  \psfrag{w0s2}[cc][cc]{\tiny{$\omega^0_{S_2}$}}
  \psfrag{w1s1}[cc][cc]{\tiny{$\omega^1_{S_1}$}}
  \psfrag{w1s0}[cc][cc]{\tiny{$\omega^1_{S_0}$}}
  \psfrag{w0s1}[cc][cc]{\tiny{$\omega^0_{S_1}$}}
  \psfrag{w0s0}[cc][cc]{\tiny{$\omega^0_{S_0}$}}
  \psfrag{ua3}[cc][cc]{\tiny{$U^{\{a\}}_3$}}
  \psfrag{ub3}[cc][cc]{\tiny{$U^{\{b\}}_3$}}
  \psfrag{ua2}[cc][cc]{\tiny{$U^{\{a\}}_2$}}
  \psfrag{ub2}[cc][cc]{\tiny{$U^{\{b\}}_2$}}
  \psfrag{ua1}[cc][cc]{\tiny{$U^{\{a\}}_1$}}
  \psfrag{ub1}[cc][cc]{\tiny{$U^{\{b\}}_1$}}		
  \includegraphics[width=0.7\textwidth]{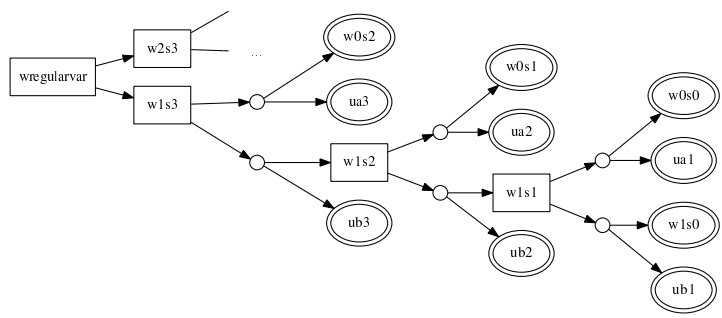}
  \caption{The DAG-filter corresponding to \gc{W\_Regular}$^{var}$ \label{regularDAG} }
\end{figure}

Function \opSty{RegularMin} in Algorithm~\ref{algo:regular} computes the minimum of a \gc{W\_Regular}$^{var}(S,M)$ cost function. The algorithm first initializes the table $u$ by assigning $\min\{U_i^c\}$ to $u[i,c]$ at lines \ref{regularL1} and \ref{regularL2}.  Lines \ref{regularL7} to \ref{regularL12} fills up the table $f$ of the size $n \times |Q|$.  Each entry $f[i,j]$ in $f$ holds the value
$\min\{\smallcf^j_{S_i}\}$, which is computed according to the formulation stated in Theorem \ref{regular}, and returns the result at
line \ref{regularL13}.

The time complexity of \opSty{RegularMin} in Algorithm~\ref{algo:regular} can be stated as follows.
\begin{theorem}\label{regularMinTime}
  Function \opSty{RegularMin} in Algorithm \ref{algo:regular} computes  the minimum of \gc{W\_Regular}$^{var}(S,M)$, and it requires $O(nd
  \cdot |Q|)$, where $n$ and $d$ are defined in Theorem  \ref{amongMinTime}.
\end{theorem}
\begin{proof}
  Lines \ref{regularL1} and \ref{regularL2} in  Algorithm~\ref{algo:regular} requires $O(n \cdot |\Sigma|)$.   Because $|\Sigma|$ is bounded by $d$, the time complexity is $O(n
  \cdot d)$. Lines \ref{regularL7} to \ref{regularL12} require $O(nd
  \cdot |Q|)$, according to the table size. Line \ref{regularL13}  requires $O(|Q|)$. The overall time complexity is $O(nd + nd \cdot
  |Q| + |Q|) = O(nd \cdot |Q|)$.
\end{proof}

\begin{algorithm}
  \caption{Finding the minimum of \gc{W\_Regular}$^{var}$}
  \label{algo:regular}
  \SetKwBlock{Start}{}{}
  \SetKw{Func}{Function}
  \SetKw{Proc}{Procedure}    
  
  \Func{} \opSty{RegularMin}($S,M$)
  \Start{	      
    \lnl{regularL1}			
    \For{$i := 1$ \KwTo $n$}{
      \lnl{regularL2}			
      \lFor{$c \in \Sigma$}{                	
        $u[i,c] := \min\{U_i^{\{c\}}\}$ \;           
      }
    }
    \lnl{regularL7}			
    $f[0,0] := 0$ \;
    \lnl{regularL8}			
    \lFor{$q_j \in Q \setminus \{q_0\}$}{
      $f[0,j] := \top$ \;
    }
    \lnl{regularL9}			
    \For{$i := 1$ \KwTo $n$}{
      \lnl{regularL10}
      $f[i,j] := \top$\;
      \lnl{regularL11}	
      \ForEach{$(q_k,q_j,c)$ such that $\delta(q_k,c) = q_j$}{
        \lnl{regularL12}	        
        $f[i,j] = \min\{f[i,j], f[i,k] \oplus u[i,c]$\}\;        
      }
    }
    \lnl{regularL13}			
    \Return{$\min_{q_j \in F}\{f[n,j]\}$} \;
  }
\end{algorithm}

The time complexity is polynomial in the size of the scope $S$ (with associated domains of size at most $d$), and the number of states in the finite automaton $M$. We state the time complexity of enforcing GAC* with respect to \gc{W\_Regular}$^{var}$ as follows.
\begin{corollary}\label{regularTime}	  
	Given $\cf_S = \mbox{\gc{W\_Regular}}^{var}(S,M)$. Enforcing GAC* on a variable $x_i \in S$ with respect to $\cf_S$ requires $O(nd^2 \cdot |Q|)$, where $n$ and $d$ is defined in Theorem  \ref{amongMinTime}.
\end{corollary}
\begin{proof}
	Follow directly from Corollary~\ref{Thm:pDAGProject} and Theorem \ref{regularMinTime}. 
\end{proof}

\subsection{The \gc{W\_Max}/\gc{W\_Min} Cost Functions}

\begin{definition}
\label{def:maxweight}
Given a function $f(x_i,v)$ that maps every variable-value pair $(x_i,v)$, where $v \in D(x_i)$, to a cost in $\{0 \ldots \top\}$.
\begin{itemize}
\item{} The \emph{\gc{W\_Max}$(S,f)$}$(\ell)$, where $\ell \in \L(S)$,  returns $\max \{f(x_i, \ell[x_i]) \mid x_i \in S\}$;
\item{} The \emph{\gc{W\_Min}$(S,f)$}$(\ell)$, where $\ell \in \L(S)$, returns $\min \{f(x_i, \ell[x_i]) \mid x_i \in S\}$. 
\end{itemize}
\end{definition}

Note that the \gc{W\_Max} and \gc{W\_Min} cost functions are not a direct generalization of any global constraints. Therefore, their name does not follow the traditional format. However, they can be used to model the \gc{Maximum} and \gc{Minimum} hard constraints \cite{NB2001}. For examples, the \gc{Maximum}$(x_{max}, S)$ can be represented as $x_{max} = \mbox{\gc{W\_Max}}(S,f)$, where $f(x_i,v) =
v$.

\begin{example}
  \label{wmaxex1}
  Consider $S = \{x_1, x_2, x_3\}$, where $D(x_1) = \{1,3\}$, $D(x_2)  = \{2,4\}$, and $D(x_3) = \{2,3\}$.  Given $f(x_i,v) =
  3\times v$, the cost of the tuple $(1,2,3)$ given by \emph{\gc{W\_Max}$(S,f)$} is $9$, while that of  $(3,4,2)$ is $12$.
\end{example}

\begin{theorem}
  \label{maxweight}
  \gc{W\_Max}$(S,f)$ and \gc{W\_Min}$(S,f)$ are polynomially DAG-filterable, and thus tractable projection-safe.
\end{theorem}

\begin{proof}
  Definition \ref{def:maxweight} does not lead to a DAG-filter that is tractable-safe. We give another DAG-filter of \gc{W\_Max}$(S,f)$ which satisfies  polynomial DAG-filterability, while that of \gc{W\_Min}$(S,f)$ is similar. For the ease of explanation, we arrange all possible outputs of $f$ in a non-decreasing sequence
	$A =[\alpha_0, \alpha_1, \ldots, \alpha_k]$, where $\alpha_i \leq \alpha_j$ iff $i \leq j$. The value $\alpha_0 = 0$ is added into the sequence as a base case.

	We define three families of unary cost functions $\{H_i^u\}$, $\{G_j^{\alpha}\}$ and $\{F_j^\alpha\}$. Cost functions $\{H_i^u \mid x_i \in S \wedge u 
    \in D(x_i)\}$ are unary functions on $x_i \in S$ defined as \begin{equation*}	 
        H_i^u(v) = \left\{\begin{array}{ll}
            f(x_i,v), & \mbox{ if $v=u$} \\
            \top, & \mbox{ if $v \neq u$}
        \end{array}\right.
    \end{equation*}
    \noindent
    
    The unary cost functions $F_j^{\alpha_{k-1}}$ are unary cost functions on $x_j\in S$ defined as:   \begin{equation*}	 
        F_j^{\alpha_k}(u) = \left\{\begin{array}{ll}
            0, & \mbox{ if $\alpha_k = f(x_j,u)$} \\
            \top, & \mbox{ otherwise}
        \end{array}\right.
    \end{equation*}

	Cost functions $\{G_j^{\alpha_k} \mid x_j \in S \wedge \alpha_k \in A\}$  are unary functions on $x_j \in S$, defined recursively as:    \begin{eqnarray*}
    	G_j^{\alpha_0}(u) &=& \top \\
      G_j^{\alpha_k}(u) &=& \left\{\begin{array}{ll}
      		  \top, & \mbox{ if  $\alpha_k > f(x_j,u)$ and $\forall v, f(x_i,v)\neq \alpha_k$} \\
            G_j^{\alpha_{k-1}}(u), & \mbox{ if  $\alpha_k  \leq f(x_j,u)$ and $\forall v, f(x_i,v)\neq \alpha_k$} \\
            \min\{ G_j^{\alpha_{k-1}}(u), F_j^{\alpha_k}(u) \}& \mbox{ if  $\alpha_k \leq f(x_j,u)$ and $\exists v, f(x_i,v) = \alpha_k$} \\
        \end{array}\right. 
    \end{eqnarray*}    
    
	They give a polynomial DAG-filter for \gc{W\_Max} as follows.	\begin{equation}
	\label{eqn:max}
			\mbox{\gc{W\_Max}}(S,f)(\ell) = \displaystyle\min_{\alpha_k \in A \wedge \alpha_k = f(x_i,v)} \{H_i^v(\ell[x_i]) \oplus 
										  \displaystyle\bigoplus_{x_j \in S \setminus \{x_i\}}G_j^{\alpha_k}(\ell[x_j])\}
    \end{equation}
	\noindent
    $H_i^v$ represents the choice of the maximum cost component in the  tuple, while $G_j^{\alpha}$ represents the  choice of each component other than the one with the  maximum weight.
    
    \begin{figure}[htp]
	\centering
	\psfrag{wmax}[cc][cc]{\tiny{$W\_Max$}}
	\psfrag{Hd2}[cc][cc]{\small{$H^d_2$}}
  \psfrag{Hb2}[cc][cc]{\small{$H^b_2$}}
  \psfrag{Ha1}[cc][cc]{\small{$H^a_1$}}
  
  \psfrag{G2d1}[cc][cc]{\small{$G^{\alpha_6}_1$}}
  \psfrag{G2d2}[cc][cc]{\small{$G^{\alpha_6}_2$}}
  \psfrag{G2d3}[cc][cc]{\small{$G^{\alpha_6}_3$}}
  \psfrag{F2d1}[cc][cc]{\small{$F^{\alpha_6}_1$}}
  \psfrag{F2d2}[cc][cc]{\small{$F^{\alpha_6}_2$}}
  \psfrag{F2d3}[cc][cc]{\small{$F^{\alpha_6}_3$}}
  
  \psfrag{G2b1}[cc][cc]{\small{$G^{\alpha_2}_1$}}
  \psfrag{G2b2}[cc][cc]{\small{$G^{\alpha_2}_2$}}
  \psfrag{G2b3}[cc][cc]{\small{$G^{\alpha_2}_3$}}
  \psfrag{F2b1}[cc][cc]{\small{$F^{\alpha_2}_1$}}
  \psfrag{F2b2}[cc][cc]{\small{$F^{\alpha_2}_2$}}
  \psfrag{F2b3}[cc][cc]{\small{$F^{\alpha_2}_3$}}
  
  \psfrag{G1a1}[cc][cc]{\small{$G^{\alpha_1}_1$}}
  \psfrag{G1a2}[cc][cc]{\small{$G^{\alpha_1}_2$}}
  \psfrag{G1a3}[cc][cc]{\small{$G^{\alpha_1}_3$}}
  \psfrag{F1a1}[cc][cc]{\small{$F^{\alpha_1}_1$}}
  \psfrag{F1a2}[cc][cc]{\small{$F^{\alpha_1}_2$}}
  \psfrag{F1a3}[cc][cc]{\small{$F^{\alpha_1}_3$}}
  
  \psfrag{Ga01}[cc][cc]{\small{$G^{\alpha_0}_1$}}
  \psfrag{Ga02}[cc][cc]{\small{$G^{\alpha_0}_2$}}
  \psfrag{Ga03}[cc][cc]{\small{$G^{\alpha_0}_3$}}

	\includegraphics[width=0.7\textwidth]{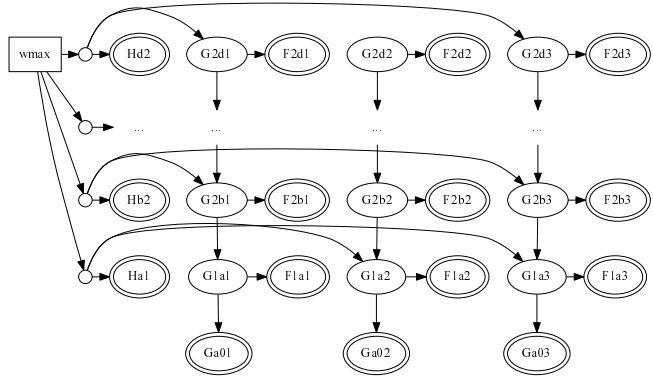}
	\caption{The DAG-filter corresponding to \gc{W\_Max} \label{maxDAG} }
	\end{figure} 
        
    The corresponding DAG $(V,E)$ as shown in Figure \ref{maxDAG} based on Example \ref{wmaxex1}. The notation is the same as Figure \ref{amongDAG}. The DAG contains $|V|$ vertices, where $|V| = O(nd \cdot n^2d) = O(n^3d^2)$.     
		By propositions~\ref{thm:oplussafe} and~\ref{thm:minsafe}, the decomposition is safely DAG-filterable. Moreover, the leaves are unary cost
		functions. The DAG-filter is polynomial and, by Theorem \ref{thm:pd-safe}, the result follows.					
    \end{proof}

\begin{example}
  \label{wmaxex2}
  Following Example \ref{wmaxex1}, the sequence $A$ is defined as:
	\begin{eqnarray*}
		A &=& [\alpha_0, \alpha_1,\alpha_2,\alpha_3,\alpha_4,\alpha_5, \alpha_6] \\
			&=& [0, f(x_1,1),f(x_2,2),f(x_3,2),f(x_1,3),f(x_3,3), f(x_2,4)] \\
		  &=& [0, 3,6,6,9,9,12]
	\end{eqnarray*}
	\noindent
	The DAG-filter for \gc{W\_Max} can be represented as:	
	\[
	\mbox{\gc{W\_Max}}(S,c)(\ell) = \min \left\{ \begin{array}{l} 		
		H_2^4(\ell[x_2]) \oplus G^{\alpha_6}_1(\ell[x_1]) \oplus G^{\alpha_6}_3(\ell[x_3]), \\		
		H_3^3(\ell[x_3]) \oplus G^{\alpha_5}_1(\ell[x_1]) \oplus G^{\alpha_5}_2(\ell[x_2]), \\
		H_1^3(\ell[x_1]) \oplus G^{\alpha_4}_2(\ell[x_2]) \oplus G^{\alpha_4}_3(\ell[x_3]), \\
		H_3^2(\ell[x_3]) \oplus G^{\alpha_3}_1(\ell[x_1]) \oplus G^{\alpha_3}_2(\ell[x_2]), \\
		H_2^2(\ell[x_2]) \oplus G^{\alpha_2}_1(\ell[x_1]) \oplus G^{\alpha_2}_3(\ell[x_3]), \\
		H_1^1(\ell[x_1]) \oplus G^{\alpha_1}_2(\ell[x_2]) \oplus G^{\alpha_1}_3(\ell[x_3]), \\
		\end{array} \right\}
	\]	
	Assume $\ell = (1,2,3)$. We first compute the values of $\{H_i^u\}$ and $\{G^\alpha_j\}$, 
	incrementally starting from $\alpha_0$. The results are shown in Table \ref{tab:values}. 
	\begin{table}[htb]
    \caption{Computing the values of $\{H_i^u\}$ and $\{G^\alpha_j\}$}
    \label{tab:values}
	\begin{center}
    \begin{tabular}{|c|c|c|c|c|}
    		\hline
         $\alpha_j$ & $H_i^u$ & $G^{\alpha_j}_1$ & $G^{\alpha_j}_2$ & $G^{\alpha_j}_3$\\  
         \hline
         $\alpha_0               $ & $-$ & $\top$ & $\top$ & $\top$\\                    
         \hline
         $\alpha_1 = f(x_1,1) = 3$ & $3$ & $0$ & $\top$ & $\top$\\            
         \hline
         $\alpha_2 = f(x_2,2) = 6$ & $6$ & $0$ & $0$ & $\top$\\            
         \hline
         $\alpha_3 = f(x_3,2) = 6$ & $\top$ & $0$ & $0$ & $0$\\            
         \hline
         $\alpha_4 = f(x_1,3) = 9$ & $\top$ & $0$ & $0$ & $0$\\            
         \hline
         $\alpha_5 = f(x_3,3) = 9$ & $9$ & $0$ & $0$ & $0$\\            
         \hline
         $\alpha_6 = f(x_2,4) = 12$ & $\top$ & $0$ & $0$ & $0$\\            
        \hline        
    \end{tabular}
\end{center}
\end{table}
	
	\noindent
	\gc{W\_Max}($S$,$c$)($\ell$) can be computed using Table \ref{tab:values}, which gives the cost $9$.	
	\[
	\mbox{\gc{W\_Max}}(S,c)(\ell) = \min \left\{ \begin{array}{l} 		
		\top \oplus 0 \oplus 0, \\		
		\top  \oplus 0 \oplus 0, \\		
		9  \oplus 0 \oplus 0, \\
		\top  \oplus 0 \oplus 0, \\
		6  \oplus 0 \oplus \top, \\
		3  \oplus \top \oplus \top, \\		
		\end{array} \right\} = 9
	\]
\end{example}

Function \opSty{WMaxMin} in Algorithm \ref{algo:maximum} computes the minimum of a \gc{W\_Max}$(S,c)$ cost function, based on Equation \ref{eqn:max}. The one for \gc{W\_Min}$(S,c)$ is similar. The for-loop at line \ref{maxL3} tried every possible variable-value pair $(x_i,a)$ in the non-decreasing order of $f(x_i,a)$. At each iteration, it first computes the minimum among all tuple $\ell$ which $\ell[x_i] = v$ and it is the maximum weighted component in the tuple in line \ref{maxL5}, and update the global minimum in line \ref{maxL6}. The variables $\{g[x_i]\}$ is then updated in line \ref{maxL4}. They store the current minimum of $\{G_i^\alpha\}$, which is used for compute the minimum among tuples with $\ell[x_i]$ is not the maximum weighted component.

\begin{algorithm}
  \caption{Finding the minimum of \gc{W\_Max}}
  \label{algo:maximum}
  
  \SetKwBlock{Start}{}{}
  \SetKw{Func}{Function}
  \SetKw{Proc}{Procedure}    
  
  \Func{} \opSty{WMaxMin}($S,f$)
  \Start{    
    \lnl{maxL1}
    \lFor{$i := 1$ \KwTo $n$}{
      $g[x_i] := \top$ \;        
    }    	
    \lnl{maxL2}
    $\ArgSty{curMin} := \top$ \;
    \lnl{maxL2.1}
    $A := \{(x_i,v) \mid x_i \in S \wedge v \in D(x_i)\}$\;
    \lnl{maxL2.2}
    sort $A$ in the nondecreasing order of $f(x_i,v)$\;
    \lnl{maxL3}
    \ForEach{$(x_i,v)$ according to the sorted list $A$}{
      \lnl{maxL3.5}
      $\alpha := f(x_i,v)$\;    		
      \lnl{maxL5}
      $\ArgSty{curCost} := H^{v}_i(v) \oplus \bigoplus_{j = 1 \ldots n, j \neq i}g[x_j] $\;
      \lnl{maxL6}
      $\ArgSty{curMin} := \min\{\ArgSty{curMin}, \ArgSty{curCost}\}$\;        	        	
      \lnl{maxL4}      	        	      	        	      	  
      $g[x_i] := \min\{g[x_i], G^{\alpha}_i(v)\}$\;       	
    } 
    \lnl{maxL7}   	
    \Return{$\ArgSty{curMin}$} \;
  }
\end{algorithm}	
	
The time complexity is given by the theorem below.
\begin{theorem}\label{thm:maxmintime}
  Function \opSty{WMaxMin} in Algorithm \ref{algo:maximum} computes  the minimum of \mbox{\gc{W\_Max}}$(S,f)$, and it requires  $O(nd\cdot\log(nd))$, where $n$ and $d$ are defined in Theorem  \ref{amongMinTime}.
\end{theorem}

\begin{proof}
  \noindent
  Line \ref{maxL2.1} takes $O(nd \cdot \log(nd))$ to sort. The  for-loop at line \ref{maxL3} iterates $nd$ times. All operations in
  the iteration requires $O(1)$ except line \ref{maxL5}. As it is,  line \ref{maxL5} requires $O(n)$. By using special data structure
  like segment trees \cite{JL1977}, the time complexity can be reduced  to $O(\log(n))$.  The overall complexity becomes $O(nd \cdot \log(nd) +
  nd \cdot \log(n)) = O(nd \cdot \log(nd))$.
\end{proof}

\begin{corollary}
  \label{minmaxTime}
  Given $\cf_S = \mbox{\gc{W\_Max}}(S,f)$. 
	Enforcing GAC* on a variable $x_i \in S$ with respect to $\cf_S$ requires $O(nd^2\cdot\log(nd))$, where $n$ and $d$ is defined in Theorem  \ref{amongMinTime}.
\end{corollary}
\begin{proof}
	Follow directly from Corollary~\ref{Thm:pDAGProject} and Theorem  \ref{thm:maxmintime}.
\end{proof}

In this section, we have seen how the minimum of a polynomially DAG-filterable global cost function can be computed efficiently, leading  to efficient soft local consistency enforcement. However, each newly implemented cost function requires to build a corresponding DAG structure with a dedicated dynamic
programming algorithm.

In the next section, we show that, in some cases, it is also possible to
avoid this by directly decomposing a global cost functions into a CFN in such a way that local consistency enforcement will emulate dynamic programming, avoiding the need for dedicated enforcement algorithms.
\section{Decomposing Global Cost Functions into CFNs}
\label{s:dec}

In  CSPs, some global constraints can be efficiently represented by a logically equivalent subnetwork of constraints of bounded arities~\cite{besvanHCP03,bessiereHandbook06}, and are said to be decomposable. Similarly, we will show that some global cost functions can be encoded as a sum of bounded arity cost functions. The definition below applies to any cost function, including constraints,  
extending the definition in~\cite{besvanHCP03} and \cite{bessiereHandbook06}.

\begin{definition}\label{def-netdecomp}
 ~~For a given integer $p$, a \emph{$p$-network-decomposition}  of a global cost function  $\gc{W\_GCF}(S,A_1,\ldots,A_k)$ is a polynomial transformation $\delta_p$ that returns a CFN  $\delta_p(S,A_1,\ldots,A_k) = (S\cup E,\F,\top)$, where $S\cap E=\varnothing$, such that $\forall \cf_{T}\in \F,|T|\leq p$ and $\forall \ell \in \L^S,  \gc{W\_GCF}(S,A_1,\ldots,A_k)(\ell) = \min_{\ell'\in \L^{S\cup E},  \ell'[S] = \ell} \bigoplus_{\cf_{S_i} \in \F} \cf_{S_i}(\ell'[S_i])$.
\end{definition}

Definition~\ref{def-netdecomp} above allows for the use of extra variables $E$, which do not appear in the original cost function scope and are eliminated by minimization. We assume, without loss of generality, that every extra variable $x\in E$ is involved in at least two cost functions in the decomposition.\footnote{Otherwise, such a variable  can be removed by variable elimination: remove $x$ from $E$ and  replace the $\cf_{T}$ involving $x$ by the cost function $\min_x \cf_{T}$ on  $T\setminus\{x\}$. This preserves the Berge-acyclicity of the network if it exists.}  Clearly, if $\gc{W\_GCF}(S,A_1,\ldots,A_k)$ appears in a CFN $P=(\X,\C,\top)$ and decomposes into $(S\cup E,\F,\top)$, the optimal solutions of $P$ can directly be obtained by projecting the optimal solutions of the CFN $P'=(\X\cup E, \C\setminus\{\gc{W\_GCF}(S,A_1,\ldots,A_k)\} \cup \F,\top)$ on $\X$.

\subsection{Building network-decomposable global cost functions}

A global cost function can be shown to be network-decomposable by exhibiting a bounded arity network decomposition of the global cost function. 
There is a simple way of deriving network-decomposable cost functions from known decomposable global constraints. The process goes directly from a known decomposable global constraint to a network-decomposable global cost function and does not require to use an intermediate soft global constraint with an associated violation measure $\mu$. Instead, the global cost function will use any {relaxation} of the decomposed global constraint.

{We say that the cost function $W_S$ is a \emph{relaxation} 
of $W'_S$ if for all $\ell \in \L^S, W_S(\ell) \leq W'_S(\ell)$. 
We then  write $W_S\leq W'_S$.  
} 
%
From a network-decomposable global constraint, it is possible to 
define an associated network-decomposable global cost function by 
relaxing every constraint in the decomposition. 

\begin{theorem}\label{dec-relax}
  Let $\gc{GC}(S,A_1,\ldots,A_k)$ be a global constraint that $p$-network decomposes into a classical constraint network $(S\cup E,\F,\top)$ and $f_\theta$ be a function parameterized by $\theta$ that maps every $C_T\in \F$ to a cost function $f_\theta(C_{T})$ such  that $f_\theta(C_{T}) \leq C_T$. The global cost function 
  $$\gc{W\_GCF}(S,A_1,...,A_k,f_\theta)(\ell)  = \min_{\myatop{\ell'\in \L^{S\cup E}}{\ell'[S] = \ell}} \bigoplus_{C_{T} \in \F}
  f_\theta(C_{T})(\ell'[T])$$ is a relaxation  of $\gc{GC}(S,A_1,\ldots,A_n)$, and is  $p$-network-decomposable by construction.
\end{theorem}

\begin{proof} 
  Since $(S\cup E,\F)$ is a network-decomposition
  of  $\gc{GC}(S,A_1,...,A_k)$, for any tuple $\ell\in \L^S$, $\gc{GC}(S,A_1,...,A_k)(\ell) = 0$
  if and only if   $\min_{\ell'\in \L^{S\cup E}, \ell'[S] = \ell} 
  \bigoplus_{C_T \in \F} C_T(\ell'[T]) = 0$. Let $\ell'\in \L^{S\cup E}$ be a
  tuple where this minimum is reached. This implies that $\forall C_{T}\in \F$,
  $C_{T}(\ell'[T]) = 0$. Since $f_\theta(C_T)\leq C_{T}$,  $f_\theta(C_{T})(\ell'[T])=0$. Therefore  $\bigoplus_{C_{T}
  \in \F} f_\theta(C_{T})(\ell'[T]) = 0$ and $\gc{W\_GCF}(S,A_1,\ldots,A_k,
  f_\theta)(\ell) = 0$. Moreover, the global cost function is $p$-network-decomposable by
  construction.
\end{proof}


Theorem~\ref{dec-relax} allows to immediately derive a long list of network decomposable global cost functions from existing network decompositions of global constraints such as \textsc{AllDifferent}, \textsc{Regular}~\cite{GP2004}, \textsc{Among} and \textsc{Stretch}~\cite{bessiere2007reformulating}. The parameterization through $f_\theta$ also allows a lot of flexibility. 

\begin{example}
  Consider the softened variant \gc{W\_AllDifferent}$^{dec}(S)$  of the global constraint \gc{AllDifferent}$(S)$ constraint using the  \emph{decomposition} violation measure where the cost of an  assignment is the number of pairs of variables taking the same   value~\cite{TJC2001}. It is well known that \gc{AllDifferent} decomposes into a set of  $\frac{n.(n-1)}{2}$ binary difference constraints.  Similarly, the   \gc{W\_AllDifferent}$^{dec}(S)$ cost function can be decomposed into a set of $\frac{n.(n-1)}{2}$ soft difference cost functions. A soft   difference cost function takes cost 1 iff the two involved variables   have the same value and 0 otherwise. In these cases, no extra variable is required. 

{\gc{AllDifferent} can be softened in a different way. 
Take an arbitrary graph $G=(V,E)$ over $V$, and consider the 
violation measure where the cost of an  assignment is the number 
of pairs of variables in $E$ taking the same   value. 
This gives rise to a global cost function \gc{W\_AllDifferent}$^{f_G}(V)$ 
that allows  a zero cost assignment if and only if $G$ is colorable, 
which is an NP-hard problem. }
Enforcing any soft arc consistency on that single global cost function will be intractable as well since it requires to compute the minimum of the cost function. Instead,  enforcing soft arc consistencies 
on the network-decomposition into binary cost functions will obviously be polynomial but will achieve a lower level of filtering.
\end{example}

\section{Local Consistency and Network-Decompositions}
\label{s:loc}

As we have seen with the \gc{W\_AllDifferent}$(V,f_G)$ global cost function, the use of network-decompositions instead of a monolithic variant has both advantages and drawbacks.  Thanks to local reasoning, a decomposition may be filtered more efficiently, but this may hinder the level of filtering achieved. 
{In CSP, it was observed that the structure of the decomposition 
has an impact on  the level of consistency achieved when filtering 
the decomposition. }

{Before going further, we give some extra definitions that are 
useful to characterize structure of decompositions. 
The hypergraph $(X,E)$ of a CFN $(\X,\C,\top)$ has one vertex per 
variable $x_i\in \X$ and one hyperedge for every scope $S$ such 
that $\exists \cf_S\in \C$. 
The incidence graph of a hypergraph $(X,E)$ is a bipartite  graph 
$G=(X\cup  E,E_H)$ where $\{x_i,e_j\}\in E_H$ iff $x_i\in X,e_j\in E$ 
and $x_i$  belongs to the hyperedge $e_j$.  
A hypergraph $(X,E)$ is Berge-acyclic iff its incidence graph is acyclic.}

In  CSP, it is known that if the decomposition is Berge-acyclic, 
then enforcing GAC on the decomposition enforces GAC on the global 
constraint itself~\cite{Beeri83}. 
We now show that a similar result can be obtained for cost functions
using either a variant of Directional Arc Consistency or Virtual Arc
Consistency (VAC), whose definitions are given in the two subsections
below.

\subsection{Berge-acyclicity and directional arc consistency}

In this section, we will show that enforcing directional arc
consistency on a Berge-acyclic network-decomposition of a cost
function or on the original global cost function yields the same cost
distribution on the last variable and therefore the same lower bound
(obtained by node consistency
) provided a correct variable ordering is used.

Directional Arc Consistency has been originally defined on binary
networks. We define Terminal DAC (or T-DAC) which generalizes
Directional Arc Consistency~\cite{CooperFCSP} by removing the
requirement of having binary scopes.

\begin{definition}[T-DAC] Given a CFN $N=(\X, \C, \top)$ a total order
  $\prec$ over variables:
  \begin{itemize}
  \item{} For a cost function $\cf_S \in \C^+$, a tuple $\ell\in \L^S$
    is a full support for a value $a\in D(x_i)$ of $x_i\in S$ iff
    $\cf_S(\ell)\bigoplus_{x_j\in S, j\neq i} \cf_j(\ell[x_j]) = 0$.
  \item{} A variable $x_i \in S$ is \emph{star directional arc
      consistent (DAC*)} for $\cf_S$ iff \vspace{-5pt}
    \begin{itemize}
    \item{} $x_i$ is NC*;
    \item{} each value $v_i \in D(x_i)$ has a full support $\ell$
      for $\cf_S$.
    \end{itemize}
  \item{} $N$ is \emph{Terminal Directional Arc Consistent} (T-DAC)
    w.r.t. the order $\prec$ iff for all cost functions
    $\cf_S\in \C^+$, the minimum variable in $S$ is DAC* for $\cf_s$.
  \end{itemize}
\end{definition}

To enforce T-DAC on a cost function $\cf_S$, it suffices to first shift the cost of every unary cost function $\cf_i, i \in S$ inside $\cf_S$ by applying \Proj{$S,\{x_i\},(a),-\cf_i(a)$} for every value $a\in D_i$.  Let $x_j$ be the minimum variable in $S$ according to $\prec$, one can then apply \Proj{$S,\{x_j\},(b),\alpha$} for every value $b\in D(x_j)$ with $\alpha = \min_{\ell\in \L^S, \ell[x_j] =b} \cf_S(\ell)$. Let $\ell$ be a tuple where this minimum is reached. Then either $\alpha = \top$ and the value will be deleted, or $\ell$ is a full support for $b\in D(x_j)$:
$ \cf_S(\ell)\bigoplus_{x_i\in S, i\neq j} \cf_i(\ell[x_i]) = 0$. This support can only be broken if for some unary cost function $\cf_i, i \in S, i\neq j$, $\cf_i(a)$ increases for some value $a\in D(x_i)$. Since $j$ is minimum, $i\succ j$.

To enforce T-DAC on a CFN $(\X,\C,\top)$, one can simply sort $\C$
 in a \emph{decreasing} order of the minimum variable in the scope
of each cost function,
and apply the previous process on each cost function, successively. When a cost function $\cf_S$ is processed, all the cost functions whose minimum variable is larger than the minimum variable of $S$ have already been processed, which guarantees that none of the established full supports will be broken in the future. Enforcing T-DAC is therefore in $O(ed^r)$ in time, where $e=|\C|$ and $r=\max_{\cf_S\in \C}|S|$. 
 
\begin{theorem}\label{theo-dac}
  If a global cost function \gc{W\_GCF}$(S,A_1,\ldots,A_k)$ decomposes
  into a Berge-acyclic CFN $N=(S\cup E, \F)$, there exists an ordering
  on $S\cup E$ such that the unary cost function $\cf_{x_{i_n}}$ on
  the last variable $x_{i_n}$ of $S$ produced by enforcing T-DAC on
  the sub-network
  $(S, \{\gc{W\_GCF}(S,A_1,\ldots,A_k)\}\cup\{\cf_{x_i}\}_{x_i\in S})$
  is identical to the unary cost function $\cf'_{x_{i_n}}$ produced by
  enforcing T-DAC on the decomposition
  $N=(S \cup E, \F\cup\{\cf_{x_i}\}_{x_i\in S})$.
\end{theorem}

\begin{proof}
  Consider the decomposed network $N$ and $I_N=(S\cup E\cup \F,E_I)$  its incidence graph. As $N$ is Berge-acyclic we know that $I_N$ is a tree whose vertices are the variables and the cost functions of $N$.  We root  $I_N$ in a variable of $S$. The neighbors (parent and children, if any)  of cost functions $\cf_T$ are the variables in $T$. The neighbors of
  a variable $x_i$ are the cost functions involving $x_i$. Consider any   topological ordering of the vertices of $I_N$. This ordering   induces a variable ordering $(x_{i_1},\ldots,x_{i_n}),x_{i_n} \in S$ which is   used to enforce T-DAC on $N$. Notice that for any cost function  $\cf_T\in \F$, the parent variable of $\cf_T$ in $I_N$ appears after all  the other variables of $T$.

  Consider a value $a\in D(x_{i_n})$ of the root. Since NC* is enforced, $\cf_{x_{i_n}}(a) < \top$.  Let $\cf_T$ be any child of $x_{i_n}$ and  $\ell$ a full support of value $a$ on $\cf_{T}$. We have $\cf_{x_{i_n}}(a)  = \cf_{T}(\ell) \bigoplus_{x_i\in T}\cf_{x_i}(\ell[x_i])$, which proves that $\cf_{T}(\ell)
  = 0$ and $\forall x_ i \in T, i\neq i_n, \cf_{x_i}(\ell[x_i]) = 0$. $I_N$ being a  tree, we can inductively apply the same argument on all the  descendants of $x_{i_n}$ until leaves are reached, proving that the  assignment $(x_{i_n}=a)$ can be extended to a complete assignment with  cost $\cf_{x_{i_n}}(a)$ in $N$. In both cases, $\cf_{x_{i_n}}(a)$ is the cost  of an optimal extension of $(x_{i_n} =a)$ in $N$.

  Suppose now that we enforce T-DAC using the previous variable  ordering on the undecomposed sub-network $(S,  \{\gc{W\_GCF}(S,A_1,\ldots,A_k)\}\cup\{\cf_{x_i}\}_{x_i\in S})$.  Let $\ell$ be a full support of value $a\in D(x_{i_n})$ on \gc{W\_GCF}$(S,A_1,\ldots,A_k)$. By definition, $\cf_{x_{i_n}}(a) = \gc{W\_GCF}(S,A_1,\ldots,A_k)(\ell)  \bigoplus_{x_i\in S}\cf_{x_i}(\ell[x_i])$ which proves that $\cf_{x_{i_n}}(a)$ is the  cost of an optimal extension of $(x_{i_n} =a)$ on $(S,  \{\gc{W\_GCF}(S,A_1,\ldots,A_k)\}\cup\{\cf_{x_i}\}_{x_i\in S})$. By definition of   decomposition, and since $x_{i_n}\not\in E$, this is equal to the cost  of an optimal extension of $(x_{i_n}=a)$ in $N$.
\end{proof}

T-DAC has therefore enough power to handle Berge-acyclic network-decompositions without losing any filtering strength, provided a correct order is used for applying EPTs. In this case, T-DAC emulates a simple form of dynamic programming on the network-decomposition.

\begin{example}\label{ex:wregular}
Consider the \gc{Regular} $(\{x_1,\ldots,x_n\},M)$ global
constraint, defined by a (not necessarily deterministic) finite automaton $M = (Q,\Sigma,\delta,q_0,F)$, where $Q$ is a set of states, $\Sigma$ the
emission alphabet, $\delta$ a transition function from $\Sigma\times
Q\rightarrow 2^Q$, $q_0$ the initial state and $F$ the set of final
states. As shown in~\cite{comicsECAI08}, this constraint decomposes into a constraint network
$(\{x_1,\ldots,x_n\}\cup\{Q_0,\ldots,Q_n\},C)$ where the extra variables $Q_i$ have $Q$ as their domain. The set of constraints $C$ in the network decomposition contains two unary constraints restricting $Q_0$ to $\{q_0\}$ and $Q_n$ to $F$ and a sequence of identical ternary constraints $c_{\{Q_i,x_{i+1},Q_{i+1}\}}$ each of which authorizes a triple
$(q,s,q')$ iff $q'\in\delta(q,s)$, thus capturing $\delta$. A relaxation of this decomposition may relax each of these constraints. The unary constraints on $Q_0$ and $Q_n$ would be replaced by unary cost functions $\lambda_{Q_0}$ and $\rho_{Q_n}$ stating the cost for using every state as either an initial or final state while the ternary constraints would be relaxed to ternary cost functions $\sigma_{\{Q_i,x_{i+1},Q_{i+1}\}}$ stating the cost for using any $(q,s,q')$ transition. 

This relaxation precisely corresponds to the use of a weighted
automaton $M_W=(Q,\Sigma,\lambda,\sigma,\rho)$ where every transition, starting and finishing state has an associated, possibly intolerable, cost defined by the cost functions $\lambda,\sigma$ and $\rho$~\cite{DBLP:conf/mfcs/CulikK93}.
The cost of an assignment in the decomposition is equal, by definition, to
the cost of an optimal parse of the assignment by the weighted
automaton. This defines a \gc{W\_Regular}$(S,M_W)$ global cost
function which is 
parameterized by a weighted automaton. As shown in~\cite{KNW2011}, a weighted automaton can encode the Hamming and Edit distances to the language of a classical automaton. 
We observe that the hypergraph of the decomposition of 
\gc{W\_Regular} is Berge-acyclic. Thus, contrary to 
the \gc{AllDifferent} example, where decomposition was hindering filtering,  
T-DAC on the \gc{W\_Regular} network-decomposition 
achieves T-DAC on the original cost function. 
\end{example}

It should be pointed out that T-DAC is closely related to
mini-buckets~\cite{dechter97minibuckets} and Theorem~\ref{theo-dac}
can easily be adapted to this scheme. Mini-buckets perform a weakened
form of variable elimination: when a variable $x$ is eliminated, the
cost functions linking $x$ to the remaining variables are partitioned
into sets containing at most $i$ variables in their scopes and at most
$m$ functions (with arity $>1$). If we compute mini-buckets using the same variable ordering, with $m=1$ and unbounded $i$, 
we will obtain the same unary costs as T-DAC on the root variable $r$, with the same time and space complexity. 
Mini-buckets can be used along two main recipes:
precomputed (static) mini-buckets do not require update during search
but restrict search to one static variable ordering; dynamic
mini-buckets allow for dynamic variable ordering (DVO) but suffer from
a lack of incrementality. Soft local consistencies, being based on
EPTs, always yield equivalent problems, providing incrementality
during search and are compatible with DVO.  

\subsection{Berge-acyclicity and virtual arc consistency}

Virtual Arc Consistency offers a simple and direct link between CSPs
and CFNs which allows to directly lift CSP properties to CFNs, under
simple conditions. 

\begin{definition}[VAC~\cite{Cooper2010}] 
{Given a CFN $N=(\X,\C,\top)$, 
we define the  constraint network Bool$(N)$ as the CSP with 
the same set $\X$ of variables with the same domains,  
and which contains, for each cost function $\cf_S\in W, |S|>0$, 
a constraint  $c_S$ with the same scope, which exactly forbids all 
tuples $\ell\in \L^S$ such that $\cf_S(\ell) \neq 0$. 
A CFN $N$ is said to be Virtual Arc Consistent (VAC) 
iff the arc consistent closure of the constraint network $Bool(N)$ is non empty.}
\end{definition}

\begin{theorem}\label{theo-vac}
  If a global cost function \gc{W\_GCF}$(S,A_1,\ldots,A_k)$ decomposes
  into a Berge-acyclic CFN $N=(S\cup E, \F,\top)$ then enforcing VAC
  on either $(S\cup E, \F\cup\{\cf_{x_i}\}_{x_i\in S},\top)$ or on
  $(S, \{\gc{W\_GCF}(S,A_1,\ldots,A_k)\}\cup\{\cf_{x_i}\}_{i\in
    S},\top)$ yields the same lower bound $\cf_\varnothing$.
\end{theorem}

\begin{proof} 
  Enforcing VAC on the CFN
  $N=(S\cup E, \F\cup\{\cf_{x_i}\}_{x_i\in S},\top)$ does not modify
  the set of scopes as it only performs 1-EPTs (See
  Definition~\ref{r-ept}).  Hence it yields an equivalent problem $N'$
  such that $Bool(N')$ has the same hypergraph as $Bool(N)$. Since $N$
  has a Berge acyclic structure, this is also the case for $Bool(N)$
  and $Bool(N')$. Now, Berge-acyclicity is a situation where arc
  consistency is a decision procedure.
  We can directly make use of Proposition 10.5 of~\cite{Cooper2010},
  which states that if a CFN $N$ is VAC and $Bool(N)$ is in a class of
  CSPs for which arc consistency is a decision procedure, $N$ has an
  optimal solution of cost $w_\varnothing$.

  Similarly, the network
  $Q=(S, \{\gc{W\_GCF}(S,A_1,\ldots,A_k)\}\cup\{\cf_{x_i}\}_{x_i\in
    T},\top)$
  contains just one cost function with arity strictly above 1 and
  $Bool(Q)$ will be decided by arc consistency. Enforcing VAC will
  therefore provide a CFN which also has an optimal solution of cost
  $\cf_\varnothing$. Finally, the networks $N$ and $Q$ have the same
  optimal cost by definition of a decomposition.
\end{proof}

Given that VAC is both stronger and more expensive to enforce than
DAC*, the added value of this theorem, compared to
theorem~\ref{theo-dac}, is that it does not rely on a variable
ordering. Such order always exists but it is specific to each global
cost function. Theorem~\ref{theo-vac} becomes interesting when a
problem contains several global cost functions with intersecting
scopes, for which theorem~\ref{theo-dac} may produce inconsistent
orders.

\section{Relation between DAG-filterability and Network-Decompositions}
\label{comparison}

In this section, we show that Berge-acyclic network-decomposable global
cost functions are also {polynomially DAG-filterable}. 

\begin{theorem}\label{theo-decompo}
  Let $\gc{W\_GCF}(S,A_1,\ldots,A_k)$ be a network-decomposable global
  cost function that decomposes into a CFN $(S\cup E,\F,\top)$ with a
  Berge-acyclic hypergraph. Then \gc{W\_GCF}$(S,A_1,\ldots,A_k)$ is
  polynomially DAG-filterable.
\end{theorem}

\begin{proof}
  We consider the incidence graph of the Berge-acyclic hypergraph of
  the CFN $(S\cup E,\F,\top)$ and choose a root for it in the original
  variables $S$, defining a rooted tree denoted as $I$. This root
  orients the tree $I$ with leaves being variables in $S$ and $E$. In
  the rest of the proof, we denote by $I(x_i)$ the subtree of $I$
  rooted in $x_i\in S\cup E$. Abusively, when the context is clear,
  $I(x_i)$ will also be used to denote the set of all variables in the
  subtree.
  
  The proof is constructive. We will transform $I$ into a filtering DAG (actually
  a tree) of nodes that computes the correct cost $\min_{\ell'\in
    \L^{S\cup E}, \ell'[S] = \ell} \bigoplus_{\cf_{T} \in \F}
  \cf_{T}(\ell'[T])$ and satisfies all the required properties of
  polynomial DAG-filters. To achieve this, we need to
  guarantee that the aggregation function $f_i=\oplus$ is always used
  on cost functions of disjoint scopes, that $f_i=\min$ is always
  applied on identically scoped functions and that sizes remain
  polynomial.

  We will be using three types of DAG nodes. A first type of node will
  be associated with every cost function $\cf_{T} \in\F$ in the
  network-decomposition. Each cost function appears in $I$ with a
  parent variable $x_i$ and a set of children variables among which
  some may be leaf variables. By the assumption that extra variables
  belong to at least two cost functions (see paragraph below
  Definition \ref{def-netdecomp}), leaf variables necessarily belong
  to $S$. We denote by $\textit{leaf}(T)$ the set of leaf variables
  in the scope $T$. The first type of node aims at computing the value
  of the cost function $W_T$ combined with the unary cost functions on
  each leaf variable. This computation will be performed by a family
  of nodes $U_T^\ell$, where $\ell\in\L^{T-\textit{leaf}(T)}$ is an
  assignment of non-leaf variables. Therefore, for a given cost
  function $W_T$ and a given assignment $\ell$ of non-leaf variables,
  we define a DAG node with scope $leaf(T)$:

  $$U_{T}^\ell(\ell') = \cf_{T}(\ell\cup \ell') \bigoplus_{x_j\in \textit{leaf}(T)} \cf_{x_j}(\ell'[x_j])$$ 

  These nodes will be leaf nodes of the filtering DAG. Given that
  all cost functions in $I$ have bounded arity, these nodes
  have an overall polynomial size and can be computed in polynomial
  time in the size of the input global cost function.

  Nodes of the second and third types are associated to every non-leaf
  variable $x_i$ in $I$. For every value $a\in D(x_i)$, we will have a
  node $\omega_i^a$ with scope $I(x_i)\cap S$. $x_i$ may have
  different children cost functions in $I$ and we denote by $\C_i$ the
  set of all the children cost functions of $x_i$ in $I$. For each
  $\cf_T \in \C_i$, we will also have a DAG node $\omega_T^{i,a}$ with
  scope $S'_i=(I(\cf_T)\cup\{x_i\}) \cap S$. Notice that even if these
  scopes may be large (ultimately equal to $S$ for $\omega_i^a$ if
  $x_i$ is the root of $I$), these nodes are not leaf nodes of the
  filtering DAG and do not rely on an extensional definition, avoiding
  exponential space.

  The aim of all these nodes is to compute the cost of an optimal
  extension of the assignment $\ell$ to the subtree $I(\cf_T)$ (for
  $\omega_T^{i,a}$) or $I(x_i)$ (for $\omega_i^a$). We therefore
  define:

  $$\omega_i^a(\ell) = \mathop\bigoplus_{\cf_T\in\C_i} \omega_T^{i,a}(\ell[S'_i])$$

  Indeed, if $\omega_T^{i,a}$ computes the cost of an optimal
  extension to the subtree rooted in $\cf_T$, an optimal extension to
  $I(x_i)$ is just the $\oplus$ of each optimal extension on each
  child, since the scopes $S'_i$ do not intersect ($I$ is a
  tree). The DAG node uses the $\oplus$ aggregation operator on non-intersecting scopes.

  The definition of the DAG nodes $\omega_T^{i,a}$ is more
  involved. It essentially requires:
\begin{enumerate}
\item to combine the cost of $\cf_T$ with the unary cost
  functions on leaf variables in $T$ (this is achieved by $U_T$ nodes)
  and costs of optimal extensions subtrees rooted in other non-leaf
  variables (this is achieved by $\omega_j^b$ nodes). 
\item to eliminate in this function all extra variables in the scope
  $T$ except $x_i$ if $x_i\in E$. In this case, $x_i$'s value will be
  set in $\ell$ and eliminated on higher levels.
\end{enumerate}

If $x_i\in E$ or else if $\ell[x_i]=a$, this leads to the following
definition of $\omega_T^{i,a}(\ell)$:
  \begin{equation}\label{eq:kitu}
  \min_{\myatop{\ell'\in\L^{T\cap E}}{(x_i\in S \lor \ell'[x_i]=a)}} \Big[U_T^{(\ell\cup\ell')[T-\textit{leaf}(T)]}(\ell[\textit{leaf}(T)]) \bigoplus_{x_j\in (T-\textit{leaf}(T)-\{x_i\})} \omega_j^{\ell[x_j]}(\ell[S_j])\Big]
  \end{equation}
  
  Otherwise ($x_i\in S$ and $\ell[x_i]\neq a$),
  $\omega_T^{i,a}(\ell)=\top$. This captures the fact that there is no
  optimal extension of $\ell$ that extends $(x_i,a)$ since $\ell$ is
  inconsistent with $x_i=a$.

  If we consider the root variable $x_i \in S$ of $I$, the
  $\omega_i^a$ nodes provide the cost of a best extension of any
  assignment $\ell$ (if $\ell[x_i]=a$) or $\top$ otherwise. An
  ultimate root DAG node using the aggregation operator $\min$ over
  all these $\omega_i^a$ will therefore return the optimal extension
  of $\ell\in\L^S$ to all variables in $I(x_i)$, including extra variables.
  
  From equation~\ref{eq:kitu}, one can see that nodes $\omega_T^{i,a}$
  use the aggregation operator $\min$ on intermediary nodes. These
  intermediary nodes combine the node $U_T$ and $\omega_j$ with
  $\oplus$ which have non-intersecting scopes.

  Overall all those nodes form a DAG (actually a tree). In this tree,
  every node with the aggregation operation $\oplus$ is applied to
  operands with non-intersecting scopes, as required in
  Property~\ref{thm:oplussafe}. Similarly, every node with the $\min$
  aggregation operation is applied to functions whose scope is always
  identical, as required by Property~\ref{thm:minsafe}. Note that the
  definitions of the $\omega_i^a$ and $\omega_T^{i,a}$ are linear
  respectively in the number of children of $\cf_T$ or $x_i$
  respectively. So, we have a filtering DAG satisfying Definition
  \ref{def:polydecom}.
\end{proof}



For a global cost function which is Berge-acyclic
network-decomposable, and therefore also polynomially DAG-filterable
(as Theorem~\ref{theo-decompo} shows), a natural question is which
approach should be preferred.  The main desired effect of enforcing
local consistencies is that it may increase the lower bound
$W_\varnothing$.  From this point of view,
Theorems~\ref{theo-dac} and~\ref{theo-vac} give a clear
answer for a single global cost function.  
\begin{itemize}
\item Since OSAC~\cite{Cooper2010}
is the strongest form of arc consistency (implying also
VAC), the strongest possible lower bound will be obtained by enforcing
OSAC on the network-decomposed global cost function.  The size of the OSAC
linear program being exponential in the arity of the cost functions,
the bounded arities of the network decomposed version will define a
polynomial-size linear program. This however requires an LP solver.
\item 
  If a network containing network-decomposed global cost
  functions is VAC, the underlying global cost functions are also 
  VAC. As a result, good quality lower bounds can be obtained by 
  enforcing  VAC. These lower bounds are not as good as those 
  obtained by OSAC, but VAC is usually much faster than OSAC.
\item T-DAC is otherwise extremely efficient, easy to implement,
  offering good lower bounds and incrementality for little
  effort. However, when several global cost functions co-exist in a
  problem, a variable order that is a topological sort of all these
  global cost functions may not exist. In this case, using a
  topological order for each scope independently would lead to the
  creation of cycles leading to possibly infinite
  propagation.
  It may then be more attractive to use 
  filtering DAGs to
  process these cost functions.
\end{itemize}

Finally, it should be noted that Theorem~\ref{theo-dac} only
guarantees that T-DAC on a global cost function or its topologically
sorted Berge-acyclic network-decomposition provide the same bound
contribution. If a consistency stronger than DAC* is enforced (such as FDAC* or EDAC*),
it may be more powerful when enforcedt on the global cost function itself
than on its network-decomposition, thus giving an advantage to
filtering DAGs.

In the end, the only truly informative answer will be provided by experimental results, as proposed in Section~\ref{sec-expe}.


\section{Experiments}
\label{sec-expe}

In this section, we put theory into practice and
demonstrate the practicality of the transformations 
described in the previous sections in solving
over-constrained and optimization problems. 
We implemented cost functions with our transformations 
in \texttt{\small toulbar2} v0.9.8\footnote{\url{http://www.inra.fr/mia/T/toulbar2/}}. 
For each cost function used in our benchmark problems,
we implemented {weak Existential Directional Generalized Arc Consistency (EDGAC*)~\cite{GHZL2005,LL2010,LL2012asa}, a local consistency combining AC, DAC and EAC, using DAG-filtering (called {\em DAG-based} approach in the sequel) with pre-computed tables (as described in Section~\ref{sec:proj-constr}).
When possible, we also implemented a  Berge-acyclic network-decomposition 
to be propagated using EDGAC* (called {\em network-based}
approach). We ignore weaker forms of local consistency such as Arc
Consistency or 0-inverse consistency~\cite{MCT2009} as previous
experiments with global cost functions have shown that these weak
local consistencies lead to much less efficient solving~\cite{LL2012asa}.

In the experiments, we used default options for \texttt{\small
  toulbar2}, including a new hybrid best-first search strategy introduced
in~\cite{Katsirelos15a}, which finds good solutions more rapidly compared to classical depth-first search.  The default variable ordering strategy is
dom/wdeg~\cite{boussemart2004} with Last Conflict~\cite{Lecoutre09},
while the default value ordering consists, for each variable, in
choosing first its {\em fully supported value} as defined by EDGAC*.  At each node during search, including the root node, we
eliminate dominated values using Dead End Elimination
pruning~\cite{Desmet1992,lecoutre2012,Givry13a} and we eliminate all
variables having degree less than two using variable
elimination~\cite{Bertele72,Larrosa00}. At the root node only, this is
improved by pairwise decomposition~\cite{Favier11a} and we also
eliminate all variables having a functional or bijective binary
relation ({\em e.g.,} an equality constraint) with another variable.
The tests are conducted on a single core of an Intel Xeon E5-2680
(2.9GHz) machine with 256GB RAM.

We performed our experiments on four different benchmark problems.  For the
two first benchmarks (car sequencing and nonogram), we have a model
with Berge-acyclic network-decompositions, whereas for the two others
(well-formed parentheses and market split), we do not. Each
benchmark has a 5-minute timeout. We randomly generate 30 instances
for each parameter setting of each benchmark. We first compare the
number of solved instances, \ie finding the optimum and proving its
optimality (no initial upper bound).  We report the average run-time 
in seconds and consider that an unsolved problem requires the maximum
available time (timeout). When all instances are solved, we also report 
the average number of backtracks (or '--' otherwise).
The best results are marked in bold (taking first into account the number of solved instances in less than 5 minutes and secondly CPU time). 

\subsection{The Car Sequencing Problem} 

The car sequencing problem (prob001 in CSPLib, \cite{BWL1986}) requires 
sequencing $n$ cars of different types specified by a set of
options. For any subsequence of $c_i$ consecutive cars on the assembly
line, the option $o_i$ can be installed on at most $m_i$ of them. 
This is called the \emph{capacity} constraint. 
The problem
is to find a production sequence on the assembly line such that each
car can be installed with all the required options without violating the capacity
constraint.  We use $n$ variables with domain $1$ to $n$ to model this
problem. The variable $x_i$ denotes the type of the $i^{th }$ car in
the sequence.  One \gc{GCC} (global cardinality~\cite{gcc}) constraint ensures all cars are scheduled
on the assembly line.  We post $n-c_i+1$ \gc{Among} constraints~\cite{BC1994} for
each option $o_i$ to ensure the capacity constraint is not
violated. We randomly generate $30$ over-constrained instances, 
each of which has $5$ possible
options, and for each option $o_i$, $m_i$ and $c_i$ 
are randomly generated in such a way that 
$1\leq m_i < c_i \leq 7$.  Each car in each instance is randomly
assigned to one type, and each type is randomly assigned to a set of
options in such a way that each option has $1/2$ chance to be included
in each type. To introduce costs, we randomly assign unary costs
(between 0 to 9) to each variable.

The problem is then modeled in three different ways. 
The first model is obtained by replacing each 
\gc{Among} constraint by the \gc{W\_Among}$^{var}$ cost function and 
the  \gc{GCC} constraint by the \gc{W\_GCC}$^{var}$ cost 
function. 
\gc{W\_Among}$^{var}$ returns a cost equal to the number of variables 
that need to be re-assigned to satisfy the \gc{Among} constraint. 
\gc{W\_GCC}$^{var}$ is used as a global constraint and returns $\top$ on violation~\cite{LL2012asa}. 
This model is called ``flow\&DAG-based'' approach in Table~\ref{tab:car}. 

The second model, identified as ``DAG-based'' in  Table~\ref{tab:car}, 
uses a set of \gc{W\_Among}$^{var}$ cost functions to encode \gc{GCC}, 
\ie replacing the single global cost function exploiting a flow network
by a set of DAG-based global cost functions~\cite{TR-Lee2014}. 

In the third model, identified as ``network-based'' in
Table~\ref{tab:car}, each of the \gc{W\_Among}$^{var}$ in the previous DAG-based
model is decomposed into a set of ternary cost functions with extra
variables as described in Section \ref{s:dec}.

Table~\ref{tab:car} gives the experimental results. 
Column $n'$ indicates the sum of the number of original variables ($n$) and the number of extra variables added in the network-based approach. Column $n''$ gives the total number of unassigned variables after pre-processing.
We observe that the network-based approach performed the worst among the three approaches. The 
DAG-based approach is up to six times faster than the flow\&DAG-based approach on completely solved instances ($n \leq 13$) and solves more instances within the 5-minute time limit.
Surprisingly, it also develops the least number of backtracks on completely solved instances. We found that the initial lower bound produced by weak EDGAC on the flow\&DAG-based approach can be lower than the one produced by the DAG-based approach. This is due to different orders of EPTs 
done by the two approaches resulting in different lower bounds. Finding an optimal order of integer arc-EPTs is NP-hard~\cite{MT2004}. Recall that EDGAC has a chaotic behavior compared to OSAC or VAC and encoding \gc{GCC} into a set of \gc{W\_Among}$^{var}$ will produce more EPTs (each \gc{W\_Among}$^{var}$ moving unary costs differently) creating new opportunities for the overlapping \gc{W\_Among}s$^{var}$ to deduce a better lower bound.

\ignore{
Table~\ref{tab:car} gives the experimental results. The results show
that enforcing FDGAC* reduces the number of backtracks at least $25$
times more than strong $\varnothing$IC, and $10$ times more than
GAC*. Besides, enforcing FDGAC* runs at least $1.2$ times faster than
GAC*, and $4$ times faster than strong $\varnothing$IC.
}

\begin{table}[htb]
    \caption{Car sequencing problem (timeout=5min). For each approach, we give the number of instances solved (\#), the mean number of backtracks only if all the instances have been completely solved (bt.), and the mean CPU time over all the instances (in seconds).}
    \label{tab:car}
    \centering
{\small
    \begin{tabular}{|c|c|r|r|c|r|r|r|r|c|r|r|}
        \hline $n$
            & \multicolumn{3}{|c|}{flow\&DAG-based}
            & \multicolumn{3}{|c|}{DAG-based}
            & \multicolumn{5}{|c|}{network-based} \\
        \cline{2-12}
            & \# & bt.\hspace*{3mm} & time
            & \# & bt.\hspace*{3mm} & time
            & $n'$ & $n''$ & \# & bt.\hspace*{3mm} & time  \\
        \hline
8 & 30 & 19.7 & 0.10 & \bf 30 & 13.6 & 0.03 & 154 & 102 & 30 & 210.9 & 0.11\\
9 & 30 & 58.1 & 0.31 & \bf 30 & 36.4 & 0.09 & 198 & 135 & 30 & 798.5 & 0.41\\
10 & 30 & 109.9 & 0.88 & \bf 30 & 82.1 & 0.21 & 245 & 170 & 30 & 3,372 & 2.0\\
11 & 30 & 193.2 & 2.1 & \bf 30 & 156.7 & 0.50 & 293 & 206 & 30 & 17,286 & 12.2\\
12 & 30 & 522.0 & 8.0 & \bf 30 & 306.1 & 1.4 & 344 & 245 & 29 & --~~~ & 90.5\\
13 & 30 & 1,251 & 22.6 & \bf 30 & 963.1 & 4.9 & 396 & 285 & 10 & --~~~ & 233.5\\
14 & 26 & --~~~ & 86.4 & \bf 30 & 3,227 & 20.4 & 451 & 328 & 2 & --~~~ & 280.3\\
15 & 17 & --~~~ & 160.4 & \bf 29 & --~~~ & 72.1 & 507 & 372 & 2 & --~~~ & 283.8\\
16 & 12 & --~~~ & 204.9 & \bf 23 & --~~~ & 111.8 & 566 & 419 & 1 & --~~~ & 297.3\\
        \hline
    \end{tabular}
}
\end{table}

\subsection{The Nonogram Problem} 

The nonogram problem (prob012 in CSPLIB~\cite{N1994}) is a typical
board puzzle on a board of size $p \times p$. Each row and column has
a specified sequence of shaded cells.  For example, a row specified
$(2,3)$ contains two segments of shaded cells, one with length $2$ and
another with length $3$. The problem is to find out which cells need
to be shaded such that every row and every column contain the specific
sequence of shaded cells.  We model the problem by $n=p^2$ variables,
in which $x_{ij}$ denotes whether the cell at the $i^{th}$ row and
$j^{th}$ column needs to be shaded. In the experiments, we generate
random instances from perturbed white noise images.  A random
solution grid, with each cell colored with probability 0.5, is
generated.  A feasible nonogram problem instance is created from the
lengths of the segments observed in this random grid. To make it infeasible, 
for each row and each column, the list of segment lengths is randomly
permuted, \ie, its elements are shuffled randomly. If a list is empty, then a segment of random length $l$ is added ($0 < l < p$).   
We model and
soften the restrictions on each row and column by
\gc{W\_Regular}$^{var}$, resulting in three models: flow-based, DAG-based,  
and network-based. 
The flow-based model uses the \gc{W\_Regular}$^{var}$ implementation based on minimum cost flows described in~\cite{LL2012asa}, the DAG-based version uses the filtering DAG (see~\cite{TR-Lee2014} for implementation details), and the network-based version uses the decomposition presented in Example~\ref{ex:wregular}.

Table~\ref{tab:logigraphe} shows the results of the experiments. 
For medium-size problems ($p \leq 9$, $n \leq 81$), the network-based approach develops 
the least number of backtracks on average compared to the two other approaches. Value and variable elimination at pre-processing reduces the number of variables by a factor greater than two. The flow-based and DAG-based approaches develop the same number of backtracks, producing the same EPTs, but the dynamic programming algorithm implemented in the DAG-based approach is about one order-of-magnitude faster than the minimum cost flow algorithm used in the flow-based approach. Moreover, the network-based approach is at least one order-of-magnitude faster than the DAG-based approach.
On the largest instances, because of an exponential increase of the number of backtracks, the network-based approach becomes unable to solve all the instances in less than five minutes, but still outperforms the other two approaches.
}
\ignore{
Table \ref{tab:logigraphe} shows the results of the experiment. In a
time limit of $5$ minutes, enforcing strong $\varnothing$IC could only
solve relatively small instances ($n=6$).  Enforcing GAC* could solve
relatively larger ones ($n=8$). For $n=10$, all instances can only be
solved when FDGAC* is enforced, and each instance is solved within a
minute.

Table \ref{tab:logigraphe} also gives the comparison on the
\gc{W\_Regular}$^{var}$ function when it is enforced by polynomially
decomposable approach and flow-based projection-safe approach
\cite{LL2009}.  The two approaches result in the same search tree when
we enforce the same consistency, but the run-time varies. Results show
that using the polynomially decomposable approach speeds up searching
by at least $3$ times than using the flow-based projection-safe
approach, due to the large constant factor behind the flow algorithm.
}

\begin{table}[htb]
    \caption{Nonogram (timeout=5min). For each approach, we give the number of instances solved (\#), the mean number of backtracks only if all the instances have been completely solved (bt.), and the mean CPU time over all the instances (in seconds).}
    \label{tab:logigraphe}
    \centering
{
\tabcolsep=0.11cm
    \begin{tabular}{|c|c|r|r|c|r|r|r|r|c|r|r|}
        \hline $n$
            & \multicolumn{3}{|c|}{flow-based}
            & \multicolumn{3}{|c|}{DAG-based}
            & \multicolumn{5}{|c|}{network-based} \\
        \cline{2-12}
            & \# & bt.\hspace*{3mm} & time
            & \# & bt.\hspace*{3mm} & time
            & $n'$ &  $n''$ & \# & bt.\hspace*{3mm} & time \\
        \hline
36 & 30 & 11.4 & 0.09 & 30 & 11.4 & 0.01 & 96 & 18 & \bf 30 & 4.4 & 0.00\\
49 & 30 & 41.8 & 0.29 & 30 & 41.8 & 0.05 & 133 & 42 & \bf 30 & 22.5 & 0.01\\
64 & 30 & 186.4 & 2.3 & 30 & 186.4 & 0.26 & 176 & 64 & \bf 30 & 90.3 & 0.01\\
81 & 30 & 254.4 & 4.5 & 30 & 254.4 & 0.50 & 225 & 97 & \bf 30 & 248.9 & 0.04\\
100 & 25 & --~~~ & 86.0 & 30 & 3,581 & 10.8 & 280 & 131 & \bf 30 & 3,861 & 0.47\\
121 & 19 & --~~~ & 166.8 & 26 & --~~~ & 72.8 & 341 & 171 & \bf 30 & 12,919 & 1.6\\
144 & 3 & --~~~ & 279.6 & 9 & --~~~ & 233.2 & 408 & 224 & \bf 28 & --~~~ & 44.1\\
169 & 0 & --~~~ & 300.0 & 5 & --~~~ & 266.5 & 481 & 267 & \bf 23 & --~~~ & 116.4\\
196 & 0 & --~~~ & 300.0 & 1 & --~~~ & 297.1 & 560 & 330 & \bf 7 & --~~~ & 257.1\\
        \hline
    \end{tabular}
}
\end{table}

\subsection{The Well-formed Parentheses problem}

In this experiment, we use a network-decomposition of the
\gc{W\_Grammar} constraint whose structure is depicted in
Figure~\ref{fig:grammarwcsp}. It is obviously not Berge-acyclic.  This
experiment will allow us to see the behavior of network-decompositions
when they are not Berge-acyclic.

Given a set of $2p$ even length intervals within $[1, \ldots, 2p]$, the
well-formed parentheses problem is to find a string of parentheses with length $2p$ such that
substrings in each of the intervals are well-formed parentheses. We
model this problem by a set of $n = 2p$ variables. Domains of size 6 are composed of three different parenthesis types: $()[]\{\}$.
We post a \gc{W\_Grammar}$^{var}$ cost function on each interval to represent
the requirement of well-formed parentheses. We generate $2p-1$ even
length intervals by randomly picking their end points in $[1, \ldots,
2p]$, and add an interval covering the whole range to ensure that all
variables are constrained. 
We also randomly assign unary costs (between 0 and 10) to each
variable.

\begin{figure}[htbp]
\begin{center}
\includegraphics{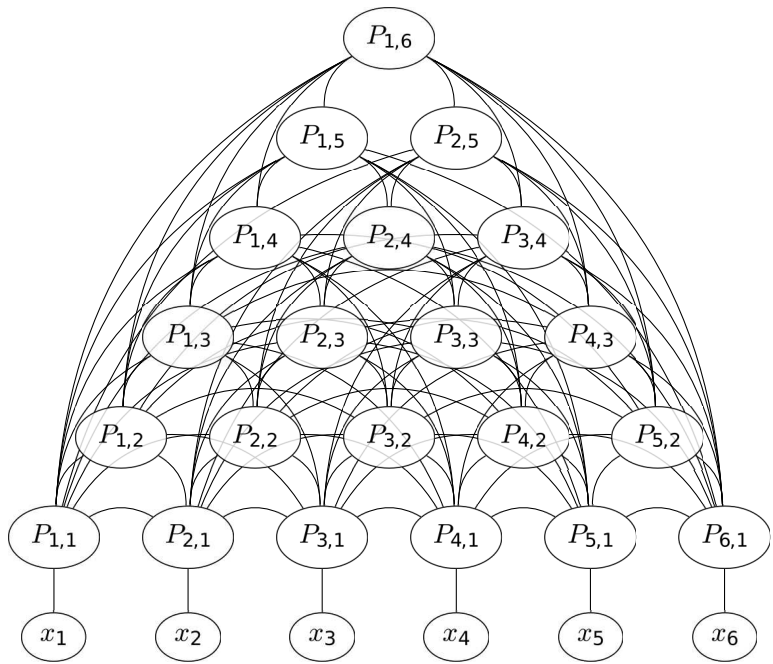}
\end{center}
\caption{Network associated to the decomposition of \gc{W\_Grammar}$^{var}(x_1,x_2,x_3,x_4,x_5,x_6)$.}\label{fig:grammarwcsp}
\end{figure}

We compare two models.
The first model, the DAG-based approach, is obtained by modeling each
\gc{W\_Grammar}$^{var}$ cost function using a filtering DAG approach.

In the second network-based model,  
we decompose each \gc{W\_Grammar}$^{var}$ cost function involving $m$ variables using $m(m+1)/2$ extra variables $P_{i,j}$ ($1 \leq j \leq m, 1 \leq i \leq m-j+1$) whose value corresponds to either a symbol value (for $j=1$) or a pair of a symbol value $S$ and a string length $k$ ($1 \leq k < j$, for $j \geq 2$) associated to the substring $(i,i+j-1)$, starting from $i$ of length $j$. Ternary cost functions link every triplet $P_{i,j}$, $P_{i,k}$, $P_{i+k,j-k}$ so that there exists a compatible rule \verb|S->AB| in order to get the substring $(i,i+j-1)$ from the two substrings $(i,i+k-1)$ and $(i+k,i+j-1)$ when $P_{i,j}=(S,k)$, $P_{i,k}=(A,u)$, $P_{i+k,j-k}=(B,v)$ with $u<k$, $v<j-k$. Binary cost functions are used to encode the terminal rules between $P_{i,1}$ ($i \in [1,m]$) and the original variables. 

\begin{table}[htb]
    \caption{Soft well-formed parentheses (timeout=5min). For each approach, we give the number of instances solved (\#), the mean number of backtracks only if all the instances have been completely solved (bt.), and the mean CPU time over all the instances (in seconds).}
    \label{tab:well-formed}
    \centering
    \begin{tabular}{|c|c|r|r|r|r|c|r|r|}
        \hline $n$
            & \multicolumn{3}{|c|}{DAG-based}
            & \multicolumn{5}{|c|}{network-based} \\
        \cline{2-9}
            & \# & bt.\hspace*{3mm} & time
            & $n'$ & $n''$ & \# & bt.\hspace*{3mm} & time \\
        \hline
8 & \bf 30 & 3.5 & 0.06 & 145 & 131 & 30 & 676.8 & 0.21\\
10 & \bf 30 & 6.6 & 0.60 & 250 & 228 & 30 & 63,084 & 12.5\\
12 & \bf 30 & 9.1 & 3.8 & 392 & 361 & 7 & --~~ & 260.3\\
14 & \bf 30 & 21.1 & 8.3 & 580 & 538 & 0 & --~~ & 300\\
16 & \bf 29 & --~~ & 48.9 & 841 & 785 & 0 & --~~ & 300\\
18 & \bf 23 & --~~ & 115.6 & 1,146 & 1,075 & 0 & --~~ & 300\\
        \hline
    \end{tabular}
\end{table}

Results are shown in Table \ref{tab:well-formed}.
The network-based approach is clearly inefficient. It has $n'=1,146$ variables on average for $p=9$ ($n=18$). The number of backtracks increases very rapidly due to the poor propagation on a non Berge-acyclic network. The DAG-based approach clearly dominates here. Notice that the DAG-based propagation of \gc{W\_Grammar}$^{var}$ can be very slow with around 1 backtrack per second for $p=9$.

As a second experiment on well-formed parentheses, 
we generate new instances using only one hard global grammar constraint and a set of $p(2p-1)$ binary cost functions corresponding to a complete graph. For each possible pair of positions, if a parentheses pair ($()$, $[]$, or $\{\}$) is placed at these specific positions, then it incurs a randomly-generated cost (between 0 to 10). A single \gc{W\_Grammar}$^{var}$ cost function is placed on all the $n=2p$ variables, which
returns $\top$  on violation (a Grammar constraint), ensuring that the whole string has well-formed parentheses.
As in the experiments of Table \ref{tab:well-formed}, the two models 
are characterized by how the consistency is enforced on the \gc{W\_Grammar}$^{var}$ cost function: a filtering DAG for the DAG-based approach, 
a network-decomposition for the network-based approach. 

Results are shown in Table \ref{tab:well-formed-hard}.
The network-based approach still develops more backtracks on average for $p \geq 6$ ($n \geq 12$) than the DAG-based approach but the difference is less important than in the previous experiment because there is a single grammar constraint. Surprisingly, for $p \leq 5$, the network-based approach develops less backtracks than the DAG-based approach. The network-based approach benefits from variable elimination that exploits bijective binary relations occurring in the decomposed hard grammar cost function. Moreover, having only one global constraint implies less extra variables for the network-based approach than in the previous experiment ($n'=189$ for $p=9$ instead of $n'=1,146$). The propagation speed of the network-based approach is much better than the DAG-based approach, with $\mathord{\sim} 4,100~bt./sec$ instead of $\mathord{\sim} 23~bt./sec$ for $p=9$, resulting in better overall time efficiency compared to the DAG-based approach, being up to 8 times faster for $p=7$ to solve all the thirty instances.

\begin{table}[htb]
    \caption{Well-formed parentheses (single hard global constraint) with additional binary cost functions (timeout=5min). For each approach, we give the number of instances solved (\#), the mean number of backtracks if available (bt.), and the mean CPU time (in seconds).}
    \label{tab:well-formed-hard}
    \centering
    \begin{tabular}{|c|c|r|r|r|r|c|r|r|}
        \hline $n$
            & \multicolumn{3}{|c|}{DAG-based}
            & \multicolumn{5}{|c|}{network-based} \\
        \cline{2-9}
            & \# & bt.\hspace*{3mm} & time
            & $n'$ &  $n''$ & \# & bt.\hspace*{3mm} & time \\
        \hline
8 & 30 & 37.4 & 0.11 & 44 & 33 & \bf 30 & 19.5 & 0.04\\
10 & 30 & 105.5 & 0.51 & 65 & 51 & \bf 30 & 100.5 & 0.12\\
12 & 30 & 265.4 & 2.4 & 90 & 73 & \bf 30 & 916.5 & 0.38\\
14 & 30 & 887.7 & 14.0 & 119 & 99 & \bf 30 & 6,623 & 1.7\\
16 & 30 & 3,037 & 80.6 & 152 & 129 & \bf 30 & 54,544 & 12.0\\
18 & 13 & --~~~ & 257.9 & 189 & 163 & \bf 30 & 394,391 & 95.7\\
        \hline
    \end{tabular}
\end{table}

\subsection{The Market split problem}

In some cases, problems may contain global cost functions which are not network-decomposable because the bounded arity cost function decomposition is not polynomial in size. However, if the network is
Berge-acyclic, Theorem~\ref{theo-dac} still applies. With exponential size networks, filtering will take exponential time, but may yield strong lower bounds. 
The 
global constraint $\sum_{i=1}^n a_i x_i = b$ ($a$ and $b$ being integer
coefficients) can be easily decomposed by introducing $n-3$ intermediate
sum variables $q_i$ and ternary sum constraints of the form $q_{i-1} + a_i x_i = q_i$ with $i \in [3,n-2]$ and $a_1 x_1 + a_2 x_2 = q_2$, $q_{n-2} + a_{n-1} x_{n-1} + a_n x_n = b$.  More generally, ternary decompositions can be built for the more general case where the right hand side of the constraint uses any relational operator, including any Knapsack constraint. In this representation, the extra variables $q_i$ have $b$ values in their domain, which is exponential in the size of the representation of $b$ (in $\log(b)$). As for the pseudo-polynomial Knapsack problem, if $b$ is polynomially bounded by the size of the global constraint, propagation will be efficient. It may otherwise be exponential in it. 

As an example, we
consider a generalized version of the Knapsack problem, the Market Split problem defined in~\cite{CornuejolsD98,Trick2003}. The goal is to minimize $\sum_{i=1}^n o_i x_i$ such that $\sum_{i=1}^n a_{i,j} x_i = b_j$ for
each $j \in [1,m]$ and $x_i$ are Boolean variables in $\{0,1\}$ ($o$, $a$ and $b$ being positive integer coefficients). We compared the Berge-acyclic decomposition in \texttt{\small toulbar2} 
(version 0.9.8)
with a direct application of the Integer Linear Programming solver \texttt{\small cplex} (version 12.6.3.0). We used a depth-first search with a static variable ordering (in decreasing $\frac{o_i}{\sum_{j=1}^m a_{i,j}}$ order) and no pre-processing  (options {\em -hbfs: -svo -o -nopre}) for \texttt{\small toulbar2}. We generated random instances with random integer coefficients in $[0,99]$ for $o$ and $a$, and $b_j =
\lfloor \frac{1}{2} \sum_{i=1}^n a_{i,j} \rfloor$. We used a sample of $30$ problems with $m=4, n=30$ leading to $\max b_j = 918$. The mean number of nodes developed in \texttt{\small toulbar2} was 29\%\ higher than in \texttt{\small cplex}, which was on average 4.5 times faster than \texttt{\small toulbar2} on these problems. The 0/1 knapsack problem probably represents a worst case situation for \texttt{\small toulbar2}, given that \texttt{\small cplex} embeds much of what is known about 0/1 knapsacks (and only part of these extend to more complicated domains).
Possible avenues to improve \texttt{\small toulbar2} results in this unfavorable situation would be to use a
combination of the $m$ knapsack constraints into one as suggested in~\cite{Trick2003}.


\section{Conclusion}
\label{conclusion}

Existing tools for solving optimization on graphical models are usually restricted to cost functions involving a reasonably small set of variables, often using an associated cost table. But problem modeling may require to express complex conditions on a non-bounded set of variables. This has been solved in Constraint Programming by using Global Constraints. 
Our results contribute to lift this approach to the more general 
framework of cost function networks, 
allowing to express and efficiently process both global constraints 
and global cost functions, using dedicated soft arc consistency filtering.

Our contributions are four-fold. First, we define the
\emph{tractability} of a global cost function, and study its behavior
with respect to projections/extensions with different arities of cost
functions. We show that tractable $r$-projection-safety is always
possible for projections/extension to/from the nullary cost function,
while it is always impossible for projections/extensions to/from
$r$-ary cost functions for $r \geq 2$. When $r = 1$, we show that a
tractable cost function may or may not be tractable
$1$-projection-safe.  Second, we define {\emph{polynomially
  DAG-filterable cost functions}} and show them to be tractable
$1$-projection-safe.  We give also a polytime dynamic programming
based algorithm to compute the minimum of this class of global cost
functions.  We also show that the cost function
\gc{W\_Grammar}$^{var}$is {polynomially
DAG-filterable} and tractable $1$-projection-safe. The same results applies to \gc{W\_Among}$^{var}$, \gc{W\_Regular}$^{var}$, and \gc{W\_Max}/\gc{W\_Min} as shown in the associated technical report~\cite{TR-Lee2014}.
Third, we show that dynamic programming can be emulated by soft consistencies such as DAC and VAC if a suitable network decomposition of the global cost function into a Berge-acyclic network of bounded arity cost functions exists. In this case, local consistency on the decomposed network is essentially as strong as on the global cost function. This approach is shown to be a specific case of the previous approach in the sense that any Berge-acyclic network-decomposable cost function is also {polynomially DAG-filterable}.
Finally, we perform
experiments and compare the DAG-based and network-based approaches, in
terms of run-time and search space. 
The DAG-based approach dominates when there are several overlapping global cost functions.
On the contrary, the network-based approach performs better if there are few global cost functions resulting in a reasonable number of extra variables.
This is complexified by additional techniques such as boosting search by variable elimination \cite{Larrosa00}, Weighted Degree heuristics~\cite{boussemart2004}, and Dead-End Elimination~\cite{Givry13a} which work better with the low-arity cost functions of the network-based approach. 
We also compare against the flow-based approach~\cite{LL2012asa} and
show that our approaches are usually more competitive.
On Berge acyclic network-decomposable cost function just as \gc{W\_Regular}$^{var}$, this is not unexpected as the dynamic programming based propagation or its emulation by T-DAC essentially solves a shortest path problem, which can easily be reduced to the more general min-cost flow problem used in~\cite{LL2012asa} which can itself be reduced to LP~\cite{NetFlo}. As problems become more specific, algorithmic efficiency can increase.

An immediate possible future work is to investigate other sufficient
conditions for {polynomially DAG-filterable} and also tractable
$1$-projection-safety.  Our results only provide a partial
answer. Whether there exists necessary conditions for  {polynomially DAG-filterable} is unknown.  
Besides  {polynomially DAG-filterable}, we
would like to investigate other form of tractable $1$-projection-safety and techniques for enforcing typical consistency
notions efficiently.  

\subsubsection*{Acknowledgements} This work has been partly funded by the ``Agence Nationale de la Recherche'' (ANR-10-BLA-0214) and a PHC PROCORE project number 28680VH.

\bibliographystyle{elsarticle-harv}
\bibliography{journal}

\end{document}


%
\title{A Catalog of Polynomially DAG-Decomposable Global Cost Functions}
\author{J.H.M.\ Lee, K.L.\ Leung and Y.\ Wu\\
Department of Computer Science and Engineering\\
The Chinese University of Hong Kong\\
Shatin, N.T., Hong Kong SAR\\
{\tt \{jlee,klleung,ywu\}@cse.cuhk.edu.hk}
}

\maketitle

\section{Introduction}

This companion manuscript should be read in conjunction with the paper of  \CiteInLine{All15}, which contains the necessary definitions, notations and theorems.

\section{Polynomial DAG-Decomposable Global Cost Functions}

In the following, we show that
\gc{W\_Among}$^{var}$,\gc{W\_Regular}$^{var}$, 
\gc{W\_Max}, and \gc{W\_Min} are Polynomially DAG-Decomposable and thus
Tractable Projection-Safe.
For simplicity, we
assume the scope of each global cost function to be $S = \{x_1, \ldots,
x_n\}$.

\subsection{The \gc{W\_Among}$^{var}$ Cost Function}

\gc{W\_Among}$^{var}$ is the cost function variant of the softened version of \gc{Among} using the corresponding variable-based violation measure~\cite{ACNA2008}.

\begin{definition}\cite{ACNA2008} 
  Given a set of values $V$, a lower bound $lb$ and an upper bound  $ub$ such that $0 \leq lb \leq ub \leq |S|$.  \emph{\gc{W\_Among}$^{var}(S, lb, ub, V)$} returns $max\{0,
  lb-t(\ell,V), t(\ell,V)-ub\}$, where $t(\ell, V) = |\{i \mid
  \ell[x_i] \in V\}|$ for each tuple $\ell \in \L(S)$.
\end{definition}

\begin{example}
  \label{wamongex1}
  Consider $S = \{x_1,x_2,x_3\}$, where $D(x_1) = D(x_2) = D(x_3) = \{a,b,c,d\}$.  
  The cost returned by \gc{W\_Among}$^{var}(S, 1, 2, \{a,b\})(\ell)$ is:
  \begin{itemize}
  \item{} $0$ if $\ell = (a,b,c,d)$;
  \item{} $1$ if $\ell = (c,d,c,d)$;
  \item{} $2$ if $\ell = (a,b,a,b)$;
  \end{itemize}	 
\end{example}

\begin{theorem}
  \label{among}
  \gc{W\_Among}$^{var}(S,lb, ub, V)$ is polynomially DAG-decomposable and  thus tractable projection-safe.
\end{theorem}
\begin{proof}
We first define two base cases $U^V_i$ and $\overline{U}^V_i$. The  function $U^V_i$ is the cost function on $x_i$ defined as:  \[
  U^V_i(v) = \left\{ \begin{array}{ll}
      0, & \mbox{ if $v \in V$;} \\
      1, & \mbox{ otherwise} \\
    \end{array} \right.
  \]
  \noindent
  and $\overline{U}^V_i(v) = 1-U^V_i$ is its negation.
  
We construct \gc{W\_Among}$^{var}$ based on $U^V_i$ and $\overline{U}^V_i$.  Define $\smallcf^j_{S_i} =
  \mbox{\gc{W\_Among}}^{var}(S_i,j, j, V)$, where $S_i = \{x_1,
  \ldots, x_i\} \subseteq S$. By definition, $S_i = S_{i-1} \cup
  \{x_i\}$ and $S_0 = \varnothing$. \gc{W\_Among}$^{var}(S, lb, ub,
  V)$ can be represented by the sub-cost functions $\smallcf^j_{S_i}$  as: \[
  \mbox{\gc{W\_Among}}^{var}(S, lb, ub, V)(\ell) = \min_{lb \leq j \leq
    ub}\{\smallcf^j_{S_{n}}(\ell)\}
  \]
  \noindent
  and each $\smallcf^j_{S_i}$ can be represented as:
\[\begin{array}{llll}
    \smallcf^j_{S_0}(\ell) &=& j \\
    \smallcf^0_{S_i}(\ell) &=& \smallcf^0_{S_{i-1}}(\ell[S_{i-1}]) \oplus \overline{U}^V_i(\ell[x_i]) & \mbox{ for $i > 0$}\\
    \smallcf^j_{S_i}(\ell) &=& \min \left\{ \begin{array}{l}
        \smallcf^{j-1}_{S_{i-1}}(\ell[S_{i-1}]) \oplus U^V_i(\ell[x_i]) \\
        \smallcf^{j}_{S_{i-1}}(\ell[S_{i-1}]) \oplus \overline{U}^{V}_i(\ell[x_i]) \\
    	\end{array}\right. &  \mbox{ for $j > 0$ and $i>0$}\\ 
    \end{array}
    \]
    The equations form a DAG $(V,E)$, as illustrated in Figure~\ref{amongDAG} using Example \ref{wamongex1}. In Figure   \ref{amongDAG}, leaves are indicated by double-lined  circles. Vertices with $\min$ or $\oplus$ aggregators are    indicated by rectangles and circles respectively. The DAG has a  number of vertices $|V| = O(ub \cdot n) = O(n^2)$ and uses only using $\oplus$ or $\min$ as aggregations with proper scopes properties. 
		By Theorem 6 and 7 by \CiteInLine{All15}, \gc{W\_Among}$^{var}$ is safely DAG- decomposable. 		
		Moreover, each leaf of the DAG is a  unary cost functions. 
		By Definition 21 from \CiteInLine{All15}, \gc{W\_Among}$^{var}$ is therefore polynomially DAG-decomposable. The result  follows by Theorem 9 by \CiteInLine{All15}.
	
    \begin{figure}[htp]
      \centering \psfrag{wamongvar}[cc][cc]{\tiny{$W\_Among^{var}$}}
      \psfrag{w2s3}[cc][cc]{\tiny{$\omega^2_{S_3}$}}
      \psfrag{w1s3}[cc][cc]{\tiny{$\omega^1_{S_3}$}}
      \psfrag{w1s2}[cc][cc]{\tiny{$\omega^1_{S_2}$}}
      \psfrag{w0s2}[cc][cc]{\tiny{$\omega^0_{S_2}$}}
      \psfrag{w1s1}[cc][cc]{\tiny{$\omega^1_{S_1}$}}
      \psfrag{w1s0}[cc][cc]{\tiny{$\omega^1_{S_0}$}}
      \psfrag{w0s1}[cc][cc]{\tiny{$\omega^0_{S_1}$}}
      \psfrag{w0s0}[cc][cc]{\tiny{$\omega^0_{S_0}$}}
      \psfrag{uv3}[cc][cc]{\tiny{$U^V_3$}}
      \psfrag{baruv3}[cc][cc]{\tiny{$\overline{U}^V_3$}}
      \psfrag{uv2}[cc][cc]{\tiny{$U^V_2$}}
      \psfrag{baruv2}[cc][cc]{\tiny{$\overline{U}^V_2$}}
      \psfrag{uv1}[cc][cc]{\tiny{$U^V_1$}}
      \psfrag{baruv1}[cc][cc]{\tiny{$\overline{U}^V_1$}}
	\includegraphics[width=0.7\textwidth]{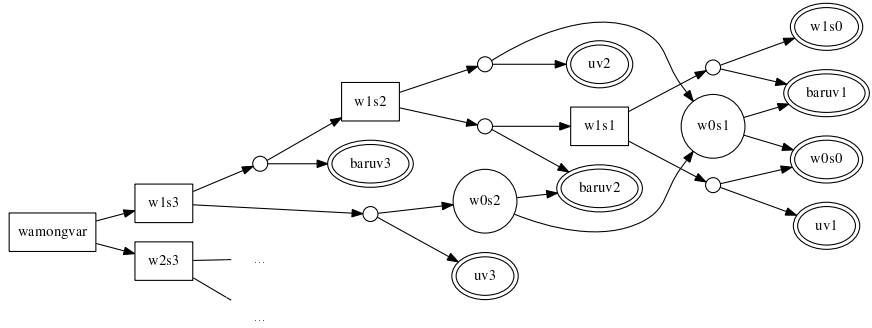}
	\caption{The DAG corresponding to \gc{W\_Among}$^{var}$ \label{amongDAG} }
	\end{figure}
\end{proof}

Function \opSty{AmongMin} in Algorithm \ref{algo:among} computes the minimum of the \gc{W\_Among}$^{var}(S, lb, ub, V)$ cost function according to Theorem \ref{among}. Lines \ref{amongL1} to \ref{amongL7} compute the minimum costs returned by each additional unary cost functions and store at the arrays $\overline{u}$ and $u$. Lines \ref{amongL8} to \ref{amongL13} builds up the results by filling the table $f$ of size $n \times ub$ according to the formulation stated in the proof of Theorem~\ref{among}, and return
the result at line \ref{amongL14}. The complexity is stated in Theorem \ref{amongMinTime} as follows.

\begin{algorithm}[htb]
  \caption{Finding the minimum of \gc{W\_Among}$^{var}$}
  \label{algo:among}
  \SetKwBlock{Start}{}{}
  \SetKw{Func}{Function}
  \SetKw{Proc}{Procedure}    
  
  \Func{} \opSty{AmongMin}($S,lb, ub, V$)
  \Start{	 
    \lnl{amongL1}						    		    		
    \For{$i = 1$ \KwTo $n$}{    		  
      \lnl{amongL2}	
      $\overline{u}[i] := \min\{\overline{U}^V_i\}$\;    		  
      \lnl{amongL7}	
      $u[i] := \min\{U^V_i\}$\;    		  
    }    
    \lnl{amongL8}						
    \lFor{$j = 0$ \KwTo $ub$}{
      $f[0,j] := j$ \;
    }
    \lnl{amongL10}						
    \For{$i = 1$ \KwTo $n$}{
      \lnl{amongL11}						
      $f[i,0] := f[i-1,0] \oplus \overline{u}[i]$ \;
      \lnl{amongL12}						
      \For{$j = 1$ \KwTo $ub$}{
        \lnl{amongL13}						 	
        $f[i,j] := \min\{f[i-1,j-1] \oplus u[i],
        f[i-1,j] \oplus \overline{u}[i]\}$ \;
      }
    }
    \lnl{amongL14}						 	
    \Return{$\min_{lb \leq j \leq ub}\{f[n,j]\}$} \;
  }
\end{algorithm}

\begin{theorem}\label{amongMinTime}
  Function \opSty{AmongMin} in Algorithm \ref{algo:among} computes the  minimum of \gc{W\_Among}$^{var}(S,lb, ub, V)$ and requires  $O(n(n+d))$, where $n = |S|$ and $d$ is the maximum domain size.
\end{theorem}
\begin{proof}
  Lines \ref{amongL1} to \ref{amongL7} in Algorithm \ref{algo:among}  take $O(nd)$. Lines \ref{amongL8} to \ref{amongL14} requires $O(n
  \cdot ub)$. Since $ub$ is bounded by $n$, the result follows.
\end{proof}

\begin{corollary}  \label{amongTime}
Given $\cf_S = \mbox{\gc{W\_Among}}^{var}(S,lb, ub, V)$. Function  \opSty{findSupport}$()$ in Algorithm 3 from \CiteInLine{All15} requires $O(nd(n+d))$.
\end{corollary}
\begin{proof}
	Follow directly from Property 2 in \CiteInLine{All15} and Theorem~\ref{amongMinTime} .
  \end{proof}

\subsection{The \gc{W\_Regular}$^{var}$ Cost Function}

\gc{W\_Regular}$^{var}$ is the cost function variant of the softened version of the hard constraint \gc{Regular}~\cite{GP2004} based on a regular language.

\begin{definition}
  A regular language $L(M)$ is represented by \emph{a deterministic finite state  automaton} (DFA) $M=(Q,\Sigma, \delta, q_0, F)$, where:
  \begin{itemize}
  \item{} $Q$ is a set of states;
  \item{} $\Sigma$ is a set of characters;
  \item{} The transition function $\delta$ is defined as: $\delta: Q
    \times \Sigma \mapsto Q$;
  \item{} $q_0 \in Q$ is the initial state, and;
  \item{} $F \subseteq Q$ is the set of final states.
\end{itemize}
A string $\tau$ lies in $L(M)$, written as $\tau \in L(M)$, iff $\tau$ can lead the transitions from $q_0$ to $q_f \in F$ in $M$
\end{definition}

The hard constraint \gc{Regular}$(S,M)$ authorizes a tuple $\ell \in
\L(S)$ if $\tau_{\ell} \in L(M)$, where $\tau_{\ell}$ is the string formed from $\ell$ \cite{GP2004}. A \gc{W\_Regular}$^{var}$ cost function is defined as follows, derived from the variable-based violation measure given by \CiteInLine{NMT2004} and \CiteInLine{WGL2006}.

\begin{definition} \cite{NMT2004,WGL2006}\label{def:regular}
  Given a DFA $M=(Q,\Sigma, \delta, q_0,
  F)$. The cost function \emph{\gc{W\_Regular}$^{var}(S,M)$} returns  $\min\{H(\tau_{\ell},\tau_i) \mid \tau_i \in L(M)\}$ for each tuple  $\ell \in \L(S)$, where $H(\tau_1, \tau_2)$ returns the Hamming  distance between $\tau_1$ and $\tau_2$.
\end{definition}

\begin{figure}[htp]
\centering
\psfrag{a}{\small{$a$}}
\psfrag{b}{\small{$b$}}
\psfrag{q0}{\small{$q_0$}}
\psfrag{q1}{\small{$q_1$}}
\psfrag{q2}{\small{$q_2$}}
\includegraphics[width=0.3\textwidth]{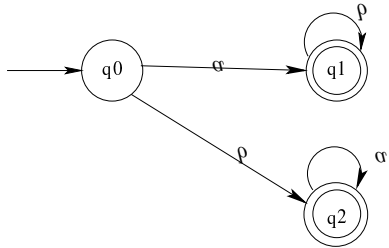}
\caption{The graphical representation of a DFA. \label{DFA}}
\end{figure}

\begin{example}
  \label{wregularex1}
  Consider $S = \{x_1,x_2,x_3\}$, where $D(x_1) = \{a\}$ and $D(x_2) =
  D(x_3) = \{a,b\}$.  Given the DFA $M$ shown in Figure \ref{DFA}. The  cost returned by \gc{W\_Regular}$^{var}(S, M)(\ell)$ is $1$ if $\ell
  = (a,b,a)$. The assignment of $x_3$ need changed in the tuple  $(a,b,a)$ so that $L(M)$ accepts the corresponding string $aba$.
\end{example}

\begin{theorem}
  \label{regular}
  \gc{W\_Regular}$^{var}(S,M)$ is polynomially DAG-decomposable and thus  tractable projection-safe.
\end{theorem}
\begin{proof}

  \gc{W\_Regular}$^{var}$ can be represented as a  DAG~\cite{NMT2004,WGL2006,LL2012asa}, which directly gives a  polynomial DAG-decomposition. In the following, we reuse  the symbols $S_i$ and $U^V_i(v)$, which are defined in the proof of  Theorem \ref{among}.
	
  Define $\smallcf^j_{S_i}$ to be the cost function  \gc{W\_Regular}$^{var}(S_i,M_j)$, where $M_j$ is the DFA $(Q,
  \Sigma, \delta, q_0, \{q_j\}))$. \gc{W\_Regular}$^{var}(S,M)$ can be  represented as: \[
  \mbox{\gc{W\_Regular}}^{var}(S,M)(\ell) = \min_{q_j \in
    F}\{\smallcf^j_{S_n}(\ell)\}
  \]
  The base cases $\smallcf^j_{S_0}$ are defined as:	 \[
  \smallcf^j_{S_0}(\ell) = \left\{ \begin{array}{ll}
      0, & \mbox{ if $j=0$} \\
      \top, & \mbox{ otherwise} \\			
    \end{array}\right. 
  \]
  Other sub-cost functions $\smallcf^j_{S_i}$, where $i > 0$, are
  defined as follows. Define $\delta_{q_j} = \{(q_i,v) \mid
  \delta(q_i,v) = q_j\}$. If $\delta_{q_j} = \varnothing$, no  transition can lead to $q_j$. \[
  \smallcf^j_{S_i}(\ell) = \left\{ \begin{array}{ll}
      \displaystyle \min_{\delta(q_k,v) = q_j}\{\smallcf^k_{S_{i-1}}(\ell[S_{i-1}]) \oplus U^{\{v\}}_i(\ell[x_i])\}, & \mbox{ if $\delta_{q_j}  \neq \varnothing$} \\
      \top, & \mbox{ otherwise} \\
    \end{array}\right. 
  \]
  \ignore{
    \[\begin{array}{lll}
      \smallcf^j_{S_0}(\ell) &=& \left\{ \begin{array}{ll}
          0, & \mbox{ if $j=0$} \\
          \top, & \mbox{ otherwise} \\			
        \end{array}\right. \\	
      \smallcf^j_{S_i}(\ell) &=& \left\{ \begin{array}{ll}
          \top, & \mbox{ if there does not exist $q_k$ and $v$ such that $\delta(q_k,v) = q_j$} \\		
          \min & 
          \{\smallcf^k_{S_{i-1}}(\ell[S_{i-1}]) \oplus U^{\{v\}}_i(\ell[x_i]) \mid \delta(q_k,v) = q_j\}, \mbox{ otherwise} \\
        \end{array}\right. \\		
    \end{array}
    \]
  }
		
  The corresponding DAG is shown in Figure \ref{regularDAG}, based on  Example \ref{wregularex1}. Again, the same notation as in Figure~\ref{amongDAG} is used. The DAG has a number of vertices $|V| =
  O(|S| \cdot |Q|)$, and its leaves are unary functions. The   DAG-decomposition is thus polynomial, and, 
	by Theorem 9 by \CiteInLine{All15}, the result follows.
\end{proof}
	
Together with Theorem $6.11$ by \CiteInLine{LL2012asa} which showed that this global cost function was flow-based tractable, this result gives another reasoning for its tractability. Theorem \ref{regular} also gives another proof of the tractable projection-safety of \gc{W\_Among}$^{var}$ \cite{ACNA2008}. The \gc{Among} global constraint \cite{BC1994} can be modeled by the \gc{Regular} global constraint \cite{GP2004}, which the size of the corresponding DFA is polynomial in the size of the scope.

\begin{figure}[htp]
  \centering
  \psfrag{wregularvar}[cc][cc]{\tiny{$W\_Regular^{var}$}}
  \psfrag{w2s3}[cc][cc]{\tiny{$\omega^2_{S_3}$}}
  \psfrag{w1s3}[cc][cc]{\tiny{$\omega^1_{S_3}$}}
  \psfrag{w1s2}[cc][cc]{\tiny{$\omega^1_{S_2}$}}
  \psfrag{w0s2}[cc][cc]{\tiny{$\omega^0_{S_2}$}}
  \psfrag{w1s1}[cc][cc]{\tiny{$\omega^1_{S_1}$}}
  \psfrag{w1s0}[cc][cc]{\tiny{$\omega^1_{S_0}$}}
  \psfrag{w0s1}[cc][cc]{\tiny{$\omega^0_{S_1}$}}
  \psfrag{w0s0}[cc][cc]{\tiny{$\omega^0_{S_0}$}}
  \psfrag{ua3}[cc][cc]{\tiny{$U^{\{a\}}_3$}}
  \psfrag{ub3}[cc][cc]{\tiny{$U^{\{b\}}_3$}}
  \psfrag{ua2}[cc][cc]{\tiny{$U^{\{a\}}_2$}}
  \psfrag{ub2}[cc][cc]{\tiny{$U^{\{b\}}_2$}}
  \psfrag{ua1}[cc][cc]{\tiny{$U^{\{a\}}_1$}}
  \psfrag{ub1}[cc][cc]{\tiny{$U^{\{b\}}_1$}}		
  \includegraphics[width=0.7\textwidth]{regular-eps-converted-to}
  \caption{The DAG corresponding to \gc{W\_Regular}$^{var}$ \label{regularDAG} }
\end{figure}

Function \opSty{RegularMin} in Algorithm~\ref{algo:regular} computes the minimum of a \gc{W\_Regular}$^{var}(S,M)$ cost function. The algorithm first initializes the table $u$ by assigning $\min\{U_i^c\}$ to $u[i,c]$ at lines \ref{regularL1} and \ref{regularL2}.  Lines \ref{regularL7} to \ref{regularL12} fills up the table $f$ of the size $n \times |Q|$.  Each entry $f[i,j]$ in $f$ holds the value
$\min\{\smallcf^j_{S_i}\}$, which is computed according to the formulation stated in Theorem \ref{regular}, and returns the result at
line \ref{regularL13}.

The time complexity of \opSty{RegularMin} in Algorithm~\ref{algo:regular} can be stated as follows.
\begin{theorem}\label{regularMinTime}
  Function \opSty{RegularMin} in Algorithm \ref{algo:regular} computes  the minimum of \gc{W\_Regular}$^{var}(S,M)$, and it requires $O(nd
  \cdot |Q|)$, where $n$ and $d$ are defined in Theorem  \ref{amongMinTime}.
\end{theorem}
\begin{proof}
  Lines \ref{regularL1} and \ref{regularL2} in  Algorithm~\ref{algo:regular} requires $O(n \cdot |\Sigma|)$.   Because $|\Sigma|$ is bounded by $d$, the time complexity is $O(n
  \cdot d)$. Lines \ref{regularL7} to \ref{regularL12} require $O(nd
  \cdot |Q|)$, according to the table size. Line \ref{regularL13}  requires $O(|Q|)$. The overall time complexity is $O(nd + nd \cdot
  |Q| + |Q|) = O(nd \cdot |Q|)$.
\end{proof}

\begin{algorithm}
  \caption{Finding the minimum of \gc{W\_Regular}$^{var}$}
  \label{algo:regular}
  \SetKwBlock{Start}{}{}
  \SetKw{Func}{Function}
  \SetKw{Proc}{Procedure}    
  
  \Func{} \opSty{RegularMin}($S,M$)
  \Start{	      
    \lnl{regularL1}			
    \For{$i := 1$ \KwTo $n$}{
      \lnl{regularL2}			
      \lFor{$c \in \Sigma$}{                	
        $u[i,c] := \min\{U_i^{\{c\}}\}$ \;           
      }
    }
    \lnl{regularL7}			
    $f[0,0] := 0$ \;
    \lnl{regularL8}			
    \lFor{$q_j \in Q \setminus \{q_0\}$}{
      $f[0,j] := \top$ \;
    }
    \lnl{regularL9}			
    \For{$i := 1$ \KwTo $n$}{
      \lnl{regularL10}
      $f[i,j] := \top$\;
      \lnl{regularL11}	
      \ForEach{$(q_k,q_j,c)$ such that $\delta(q_k,c) = q_j$}{
        \lnl{regularL12}	        
        $f[i,j] = \min\{f[i,j], f[i,k] \oplus u[i,c]$\}\;        
      }
    }
    \lnl{regularL13}			
    \Return{$\min_{q_j \in F}\{f[n,j]\}$} \;
  }
\end{algorithm}

We state the time complexity of enforcing GAC* with respect to \gc{W\_Regular}$^{var}$ as follows.
\begin{corollary}\label{regularTime}	  
	Given $\cf_S = \mbox{\gc{W\_Regular}}^{var}(S,M)$. Function  \opSty{findSupport}$()$ in Algorithm 3 from \CiteInLine{All15} requires $O(nd^2 \cdot |Q|)$, where $n$ and $d$ is defined in Theorem  \ref{amongMinTime}.
\end{corollary}
\begin{proof}
	Follow directly from Property 2 in \CiteInLine{All15} and Theorem \ref{regularMinTime}. 
\end{proof}

The time complexity is polynomial in the size of the input $(S,M)$ which includes the scope of size $n$ (with associated domains of size at most $d$) and the finite automata $M$, including $Q$.

\subsection{The \gc{W\_Max}/\gc{W\_Min} Cost Functions}

\begin{definition}
\label{def:maxweight}
Given a function $f(x_i,v)$ that maps every variable-value pair $(x_i,v)$, where $v \in D(x_i)$, to a cost in $\{0 \ldots \top\}$.
\begin{itemize}
\item{} The \emph{\gc{W\_Max}$(S,f)$}$(\ell)$, where $\ell \in \L(S)$,  returns $\max \{f(x_i, \ell[x_i]) \mid x_i \in S\}$;
\item{} The \emph{\gc{W\_Min}$(S,f)$}$(\ell)$, where $\ell \in \L(S)$, returns $\min \{f(x_i, \ell[x_i]) \mid x_i \in S\}$. 
\end{itemize}
\end{definition}

Note that the \gc{W\_Max} and \gc{W\_Min} cost functions are not a direct generalization of any global constraints. Therefore, their name does not follow the traditional format. However, they can be used to model the \gc{Maximum} and \gc{Minimum} hard constraints \cite{NB2001}. For examples, the \gc{Maximum}$(x_{max}, S)$ can be represented as $x_{max} = \mbox{\gc{W\_Max}}(S,f)$, where $f(x_i,v) =
v$.

\begin{example}
  \label{wmaxex1}
  Consider $S = \{x_1, x_2, x_3\}$, where $D(x_1) = \{1,3\}$, $D(x_2)  = \{2,4\}$, and $D(x_3) = \{2,3\}$.  Given $f(x_i,v) =
  3\times v$, the cost of the tuple $(1,2,3)$ given by \emph{\gc{W\_Max}$(S,f)$} is $9$, while that of  $(3,4,2)$ is $12$.
\end{example}

\begin{theorem}
  \label{maxweight}
  \gc{W\_Max}$(S,f)$ and \gc{W\_Min}$(S,f)$ are polynomially DAG-decomposable, and thus tractable projection-safe.
\end{theorem}

\begin{proof}
  Decomposing \gc{W\_Max}$(S,f)$ and \gc{W\_Min}$(S,f)$ directly using  Definition \ref{def:maxweight} does not lead to a tractable-safe DAG-decomposition. We give a  polynomial DAG-decomposition of \gc{W\_Max}$(S,f)$ only, since that of
  \gc{W\_Min}$(S,f)$ is similar. For the ease of explanation, we arrange all possible outputs of $f$ in a non-decreasing sequence
	$A =[\alpha_0, \alpha_1, \ldots, \alpha_k]$, where $\alpha_i \leq \alpha_j$ iff $i \leq j$. The value $\alpha_0 = 0$ is added into the sequence as a base case.

	We define three familys of unary cost functions $\{H_i^u\}$, $\{G_j^{\alpha}\}$ and $\{F_j^\alpha\}$. Cost functions $\{H_i^u \mid x_i \in S \wedge u 
    \in D(x_i)\}$ are unary functions on $x_i \in S$ defined as \begin{equation*}	 
        H_i^u(v) = \left\{\begin{array}{ll}
            f(x_i,v), & \mbox{ if $v=u$} \\
            \top, & \mbox{ if $v \neq u$}
        \end{array}\right.
    \end{equation*}
    \noindent
    
    The unary cost functions $F_j^{\alpha_{k-1}}$ are unary cost functions on $x_j\in S$ defined as:   \begin{equation*}	 
        F_j^{\alpha_k}(u) = \left\{\begin{array}{ll}
            0, & \mbox{ if $\alpha_k = f(x_j,u)$} \\
            \top, & \mbox{ otherwise}
        \end{array}\right.
    \end{equation*}

	Cost functions $\{G_j^{\alpha_k} \mid x_j \in S \wedge \alpha_k \in A\}$  are unary functions on $x_j \in S$, defined recursively as:    \begin{eqnarray*}
    	G_j^{\alpha_0}(u) &=& \top \\
      G_j^{\alpha_k}(u) &=& \left\{\begin{array}{ll}
      		  \top, & \mbox{ if  $\alpha_k > f(x_j,u)$ and $\forall v, f(x_i,v)\neq \alpha_k$} \\
            G_j^{\alpha_{k-1}}(u), & \mbox{ if  $\alpha_k  \leq f(x_j,u)$ and $\forall v, f(x_i,v)\neq \alpha_k$} \\
            \min\{ G_j^{\alpha_{k-1}}(u), F_j^{\alpha_k}(u) \}& \mbox{ if  $\alpha_k \leq f(x_j,u)$ and $\exists v, f(x_i,v) = \alpha_k$} \\
        \end{array}\right. 
    \end{eqnarray*}    
    
	They give a DAG-decomposition for \gc{W\_Max} as follows:	\begin{equation}
	\label{eqn:max}
			\mbox{\gc{W\_Max}}(S,f)(\ell) = \displaystyle\min_{\alpha_k \in A \wedge \alpha_k = f(x_i,v)} \{H_i^v(\ell[x_i]) \oplus 
										  \displaystyle\bigoplus_{x_j \in S \setminus \{x_i\}}G_j^{\alpha_k}(\ell[x_j])\}
    \end{equation}
	\noindent
    $H_i^v$ represents the choice of the maximum cost component in the  tuple, while $G_j^{\alpha}$ represents the  choice of each component other than the one with the  maximum weight.
    
    \begin{figure}[htp]
	\centering
	\psfrag{wmax}[cc][cc]{\tiny{$W\_Max$}}
	\psfrag{Hd2}[cc][cc]{\small{$H^d_2$}}
  \psfrag{Hb2}[cc][cc]{\small{$H^b_2$}}
  \psfrag{Ha1}[cc][cc]{\small{$H^a_1$}}
  
  \psfrag{G2d1}[cc][cc]{\small{$G^{\alpha_6}_1$}}
  \psfrag{G2d2}[cc][cc]{\small{$G^{\alpha_6}_2$}}
  \psfrag{G2d3}[cc][cc]{\small{$G^{\alpha_6}_3$}}
  \psfrag{F2d1}[cc][cc]{\small{$F^{\alpha_6}_1$}}
  \psfrag{F2d2}[cc][cc]{\small{$F^{\alpha_6}_2$}}
  \psfrag{F2d3}[cc][cc]{\small{$F^{\alpha_6}_3$}}
  
  \psfrag{G2b1}[cc][cc]{\small{$G^{\alpha_2}_1$}}
  \psfrag{G2b2}[cc][cc]{\small{$G^{\alpha_2}_2$}}
  \psfrag{G2b3}[cc][cc]{\small{$G^{\alpha_2}_3$}}
  \psfrag{F2b1}[cc][cc]{\small{$F^{\alpha_2}_1$}}
  \psfrag{F2b2}[cc][cc]{\small{$F^{\alpha_2}_2$}}
  \psfrag{F2b3}[cc][cc]{\small{$F^{\alpha_2}_3$}}
  
  \psfrag{G1a1}[cc][cc]{\small{$G^{\alpha_1}_1$}}
  \psfrag{G1a2}[cc][cc]{\small{$G^{\alpha_1}_2$}}
  \psfrag{G1a3}[cc][cc]{\small{$G^{\alpha_1}_3$}}
  \psfrag{F1a1}[cc][cc]{\small{$F^{\alpha_1}_1$}}
  \psfrag{F1a2}[cc][cc]{\small{$F^{\alpha_1}_2$}}
  \psfrag{F1a3}[cc][cc]{\small{$F^{\alpha_1}_3$}}
  
  \psfrag{Ga01}[cc][cc]{\small{$G^{\alpha_0}_1$}}
  \psfrag{Ga02}[cc][cc]{\small{$G^{\alpha_0}_2$}}
  \psfrag{Ga03}[cc][cc]{\small{$G^{\alpha_0}_3$}}

	\includegraphics[width=0.7\textwidth]{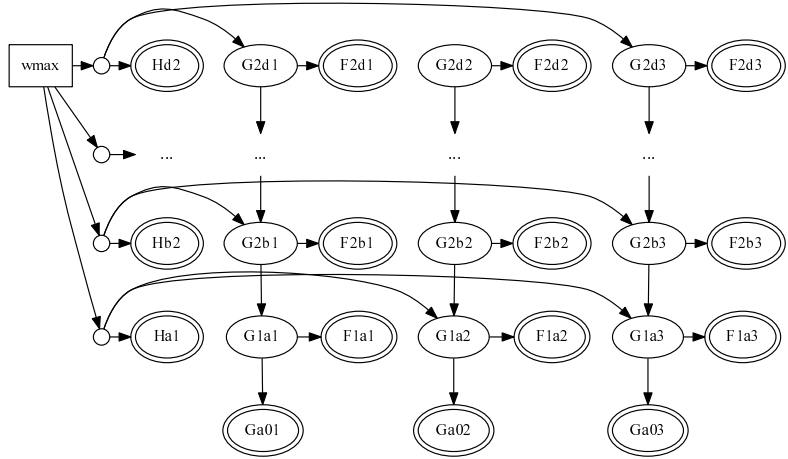}
	\caption{The DAG corresponding to \gc{W\_Max} \label{maxDAG} }
	\end{figure} 
        
    The corresponding DAG $(V,E)$ of the decomposition as shown in Figure \ref{maxDAG} based on Example \ref{wmaxex1}. The notation is the   same as Figure \ref{amongDAG}. The DAG contains $|V|$ vertices, where $|V| = O(nd \cdot n^2d) = O(n^3d^2)$.     
		By Theorem 6 and 7 by \CiteInLine{All15}, the decomposition is safely DAG-decomposition. Moreover, the leaves are unary cost
		functions. The DAG-decomposition is polynomial and, by Theorem 9 by \CiteInLine{All15}, the result follows.					
    \end{proof}

\begin{example}
  \label{wmaxex2}
  Following Example \ref{wmaxex1}, the sequence $A$ is defined as:
	\begin{eqnarray*}
		A &=& [\alpha_0, \alpha_1,\alpha_2,\alpha_3,\alpha_4,\alpha_5, \alpha_6] \\
			&=& [0, f(x_1,1),f(x_2,2),f(x_3,2),f(x_1,3),f(x_3,3), f(x_2,4)] \\
		  &=& [0, 3,6,6,9,9,12]
	\end{eqnarray*}
	\noindent
	The DAG-decomposition for \gc{W\_Max} can be represented as:	
	\[
	\mbox{\gc{W\_Max}}(S,c)(\ell) = \min \left\{ \begin{array}{l} 		
		H_2^4(\ell[x_2]) \oplus G^{\alpha_6}_1(\ell[x_1]) \oplus G^{\alpha_6}_3(\ell[x_3]), \\		
		H_3^3(\ell[x_3]) \oplus G^{\alpha_5}_1(\ell[x_1]) \oplus G^{\alpha_5}_2(\ell[x_2]), \\
		H_1^3(\ell[x_1]) \oplus G^{\alpha_4}_2(\ell[x_2]) \oplus G^{\alpha_4}_3(\ell[x_3]), \\
		H_3^2(\ell[x_3]) \oplus G^{\alpha_3}_1(\ell[x_1]) \oplus G^{\alpha_3}_2(\ell[x_2]), \\
		H_2^2(\ell[x_2]) \oplus G^{\alpha_2}_1(\ell[x_1]) \oplus G^{\alpha_2}_3(\ell[x_3]), \\
		H_1^1(\ell[x_1]) \oplus G^{\alpha_1}_2(\ell[x_2]) \oplus G^{\alpha_1}_3(\ell[x_3]), \\
		\end{array} \right\}
	\]	
	Assume $\ell = (1,2,3)$. We first compute the values of $\{H_i^u\}$ and $\{G^\alpha_j\}$, 
	incrementally starting from $\alpha_0$. The results are shown in Table \ref{tab:values}. 
	\begin{table}[htb]
    \caption{Computing the values of $\{H_i^u\}$ and $\{G^\alpha_j\}$}
    \label{tab:values}
	\begin{center}
    \begin{tabular}{|c|c|c|c|c|}
    		\hline
         $\alpha_j$ & $H_i^u$ & $G^{\alpha_j}_1$ & $G^{\alpha_j}_2$ & $G^{\alpha_j}_3$\\  
         \hline
         $\alpha_0               $ & $-$ & $\top$ & $\top$ & $\top$\\                    
         \hline
         $\alpha_1 = f(x_1,1) = 3$ & $3$ & $0$ & $\top$ & $\top$\\            
         \hline
         $\alpha_2 = f(x_2,2) = 6$ & $6$ & $0$ & $0$ & $\top$\\            
         \hline
         $\alpha_3 = f(x_3,2) = 6$ & $\top$ & $0$ & $0$ & $0$\\            
         \hline
         $\alpha_4 = f(x_1,3) = 9$ & $\top$ & $0$ & $0$ & $0$\\            
         \hline
         $\alpha_5 = f(x_3,3) = 9$ & $9$ & $0$ & $0$ & $0$\\            
         \hline
         $\alpha_6 = f(x_2,4) = 12$ & $\top$ & $0$ & $0$ & $0$\\            
        \hline        
    \end{tabular}
\end{center}
\end{table}
	
	\noindent
	\gc{W\_Max}($S$,$c$)($\ell$) can be computed using Table \ref{tab:values}, which gives the cost $9$.	
	\[
	\mbox{\gc{W\_Max}}(S,c)(\ell) = \min \left\{ \begin{array}{l} 		
		\top \oplus 0 \oplus 0, \\		
		\top  \oplus 0 \oplus 0, \\		
		9  \oplus 0 \oplus 0, \\
		\top  \oplus 0 \oplus 0, \\
		6  \oplus 0 \oplus \top, \\
		3  \oplus \top \oplus \top, \\		
		\end{array} \right\} = 9
	\]
\end{example}

Function \opSty{WMaxMin} in Algorithm \ref{algo:maximum} computes the minimum of a \gc{W\_Max}$(S,c)$ cost function, based on Equation \ref{eqn:max}. The one for \gc{W\_Min}$(S,c)$ is similar. The for-loop at line \ref{maxL3} tried every possible variable-value pair $(x_i,a)$ in the non-decreasing order of $f(x_i,a)$. At each iteration, it first computes the minimum among all tuple $\ell$ which $\ell[x_i] = v$ and it is the maximum weighted component in the tuple in line \ref{maxL5}, and update the global minimum in line \ref{maxL6}. The variables $\{g[x_i]\}$ is then updated in line \ref{maxL4}. They store the current minimum of $\{G_i^\alpha\}$, which is used for compute the minimum among tuples with $\ell[x_i]$ is not the maximum weighted component.

\begin{algorithm}
  \caption{Finding the minimum of \gc{W\_Max}}
  \label{algo:maximum}
  
  \SetKwBlock{Start}{}{}
  \SetKw{Func}{Function}
  \SetKw{Proc}{Procedure}    
  
  \Func{} \opSty{WMaxMin}($S,f$)
  \Start{    
    \lnl{maxL1}
    \lFor{$i := 1$ \KwTo $n$}{
      $g[x_i] := \top$ \;        
    }    	
    \lnl{maxL2}
    $\ArgSty{curMin} := \top$ \;
    \lnl{maxL2.1}
    $A := \{(x_i,v) \mid x_i \in S \wedge v \in D(x_i)\}$\;
    \lnl{maxL2.2}
    sort $A$ in the nondecreasing order of $f(x_i,v)$\;
    \lnl{maxL3}
    \ForEach{$(x_i,v)$ according to the sorted list $A$}{
      \lnl{maxL3.5}
      $\alpha := f(x_i,v)$\;    		
      \lnl{maxL5}
      $\ArgSty{curCost} := H^{v}_i(v) \oplus \bigoplus_{j = 1 \ldots n, j \neq i}g[x_j] $\;
      \lnl{maxL6}
      $\ArgSty{curMin} := \min\{\ArgSty{curMin}, \ArgSty{curCost}\}$\;        	        	
      \lnl{maxL4}      	        	      	        	      	  
      $g[x_i] := \min\{g[x_i], G^{\alpha}_i(v)\}$\;       	
    } 
    \lnl{maxL7}   	
    \Return{$\ArgSty{curMin}$} \;
  }
\end{algorithm}	
	
The time complexity is given by the theorem below.
\begin{theorem}\label{thm:maxmintime}
  Function \opSty{WMaxMin} in Algorithm \ref{algo:maximum} computes  the minimum of \mbox{\gc{W\_Max}}$(S,f)$, and it requires  $O(nd\cdot\log(nd))$, where $n$ and $d$ are defined in Theorem  \ref{amongMinTime}.
\end{theorem}

\begin{proof}
  \noindent
  Line \ref{maxL2.1} takes $O(nd \cdot \log(nd))$ to sort. The  for-loop at line \ref{maxL3} iterates $nd$ times. All operations in
  the iteration requires $O(1)$ except line \ref{maxL5}. As it is,  line \ref{maxL5} requires $O(n)$. By using special data structure
  like segment trees \cite{JL1977}, the time complexity can be reduced  to $O(\log(n))$.  The overall complexity becomes $O(nd \cdot \log(nd) +
  nd \cdot \log(n)) = O(nd \cdot \log(nd))$.
\end{proof}

\begin{corollary}
  \label{minmaxTime}
  Given $\cf_S = \mbox{\gc{W\_Max}}(S,f)$. 
	The function  \opSty{findSupport}$()$ in Algorithm 3 from \CiteInLine{All15} requires $O(nd^2\cdot\log(nd))$, where $n$ and $d$ is defined in Theorem  \ref{amongMinTime}.
\end{corollary}
\begin{proof}
	Follow directly from Property 2 in \CiteInLine{All15} and Theorem  \ref{thm:maxmintime}.
\end{proof}


\ignore{
In this section, we have seen how a variety of global cost
functions can be decomposed in a DAG structure that allows for
efficient minimization and therefore usual soft local consistency
enforcing. However, each newly implemented cost function requires to buid a corresponding DAG structure with a dedicated dynamic
programming algorithm.

In the next section, we show that, in some cases, it is also possible to
avoid this by directly decomposing a global cost functions into a WCSP in such a way that local consistency enforcing will emulate dynamic programming, avoiding the need for programming dedicated enforcing algorithms.
}

\ignore{

\begin{table}[htb]
  \caption{Illustration of Algorithm \ref{algo:maximum} by Example \ref{wmaxex1}}
  \label{tab:maxEx}
  \begin{center}
    \begin{tabular}{|c|c|c|c|c|c|c|c|c|}
      \hline
      Iteration & $\alpha$ & $g[x_1]$ & $g[x_2]$ & $g[x_3]$ & curCost & curMin & Category\\
      \hline
      (initial) & - & $\top$ & $\top$ & $\top$ & $\top$ & $\top$ & - \\            
      \hline
      $1$ & $ 3 (= f(x_1,a))$ & $0$ & $\top$ & $\top$ & $\top$ & $\top$ & $\kappa_3$\\            
      \hline
      $2$  & $ 6 (= f(x_2,b))$ & $0$ & $0$ & $\top$ & $\top$ & $\top$ & \multirow{2}{*}{$\kappa_6$}\\            
      \cline{1-7}
      $3$ & $ 6 (=f(x_3,b))$ & $0$ & $0$ & $0$ & $6$ & $6$ & \\            
      \hline
      $4$ & $ 9 (=f(x_1,c))$ & $0$ & $0$ & $0$ & $9$ & $6$ & \multirow{2}{*}{$\kappa_9$} \\            
      \cline{1-7}
      $5$ & $9 (=f(x_3,c))$ & $0$ & $0$ & $0$ & $9$ & $6$ & \\            
      \hline
      $6$ & $12 (=f(x_2,d))$ & $0$ & $0$ & $0$ & $12$ & $6$ & $\kappa_{12}$\\            
      \hline
    \end{tabular}
  \end{center}
\end{table}

In the following, we reuse Example \ref{wmaxex1} to illustrate the
algorithm. Table \ref{tab:maxEx} shows the values of each variables in
Algorithm \ref{algo:maximum} after each iteration. At the beginning of
Algorithm \ref{algo:maximum}, $g[x_1]$, $g[x_2]$, $g[x_3]$,
$\ArgSty{curCost}$, and $\ArgSty{curMin}$ are all initialized to
$\top$.  At the $1^{st}$ iteration, the variable-value pair $(x_1,a)$,
which $f(x_1,a) = 3$, is taken out for processing.  The variable
$\ArgSty{curCost}$ is computed in Line \ref{maxL5} by $H^a_1(a) \oplus
g[x_2] \oplus g[x_3]$, which gives $\top$. The result implies that
there does not exists any tuples such that $(x_1,a)$ is the maximum
weighted component.  The minimum is not yet discovered, and
$\ArgSty{curMin}$ is not updated. However, $g[x_1]$ is still updated
for further processing. With similar reasoning, $\ArgSty{curMin}$ is
still not updated after the $2^{nd}$ iteration. At the $3^{rd}$
iteration, $(x_3,b)$ is taken out for processing. The value of
$\ArgSty{curCost}$ is $6$ after computing line \ref{maxL5}, implying
that the cost function gives a minimum of $6$ if $x_3$ is assigned to
$b$, and $(x_3,b)$ is the maximum weighted component. Thus,
$\ArgSty{curMin}$ is updated to $6$. The variable $g[x_3]$ is updated
afterwards. The iteration continues until all variable-value pairs are
processed.

To proof the correctness of Algorithm \ref{algo:maximum}, we first
group the results of the iterations into categories
$\kappa_{\alpha_k}$ according to the cost $\alpha_k$, where $\alpha_j
\geq \alpha_i$ iff $j \geq i$. Two sets of results from the $i^{th}$
and $j^{th}$ iterations are grouped into the same category, \ie $i \in
\kappa_{\alpha_k}$ and $j \in \kappa_{\alpha_k}$ iff they have the
same value of $\alpha$ equal to $\alpha_k$. For examples, in Example
$\ref{wmaxex1}$ can by grouped into four categories ,as shown in Table
\ref{tab:maxEx}. In the following, we show that the value of
$\ArgSty{curCost}$ at the last iteration in each category gives the
``local minimum'' within the category.

For each category $\alpha_k$, we define the set $Q_{\alpha_k}$ to be
the variable-value pairs processed within the category. For examples,
in Example \ref{tab:maxEx}, $Q_{\alpha_k}$ is shown as below.
\begin{itemize}
\item{} $Q_3 = \{(x_1,a)\}$;
\item{} $Q_6 = \{(x_2,b), (x_3,b)\}$;
\item{} $Q_9 = \{(x_1,c), (x_3,c)\}$;
\item{} $Q_{12} = \{(x_2,d)\}$;
\end{itemize}
\noindent
We also define $\ArgSty{curCost}_i$ to be the value of
$\ArgSty{curCost}$ at the $i^{th}$ iteration.

\begin{lemma}
  \label{lem:localmin}
  Within the category $\kappa_{\alpha_k}$, at the $K^{th}$ iteration, where $K = \max\{i \mid i \in \kappa_{\alpha_k}\}$:
  \[
  \ArgSty{curCost}_K = \left\{ 
    \begin{array}{ll}
      \min_{(x_i,v) \in Q_{\alpha_k}}& \{\gc{W\_Max}(S,f)(\ell) \mid \ell \in \L(S) \wedge \\
      & \hspace{1in} \ell[x_i] = v \wedge \max_{x_j \in S}\{f(x_j,\ell[x_j])\} = \alpha_k\} \\
      \top, & \mbox{ if not such tuple $\ell$ exists} \\
    \end{array}
  \right.					
  \]
\end{lemma}
\begin{proof}
  Assume there do not exist any tuples $\ell$ such that $\max_{x_j \in S}\{f(x_j,\ell[x_j])\} = \alpha_k$. 
  There must exist a variable $y$ such that $f(y,u) > \alpha_k$ for every $u \in D(y)$. As observed from the algorithm,  
  $g[y]$ is not yet updated, \ie $g[y] = \top$. Thus, $\ArgSty{curCost}_K = \top$. 
  
  \noindent
  Assume such tuple exists. It is suffice to show that at the $K^{th}$ iteration, 
  $g[x_j] = \min\{G^{\alpha_k}_j\}$ for each $j \neq i$. By the definition of 
  $\{G^\alpha_j\}$, $G^\beta_j(v) = G^\alpha_j(v)$ iff $\beta \leq \alpha$. 
  As observed from the algorithm, after each iteration:
  \begin{eqnarray*}
    g[x_j] &=& \min\{G^\beta_j(v) \mid \beta = f(x_j,v) \leq \alpha_k\} \\
    &=& \min\{G^{\alpha_k}_j(v) \mid f(x_j,v) \leq \alpha_k\} \\
    &=& \min\{G^{\alpha_k}_j\}
  \end{eqnarray*}
  By the definition of $\{H^u_i\}$, within the category
  $\kappa_{\alpha_k}$, $\min\{H^u_i\} = \alpha_k$.  Result follows by
  the safe decomposability of $\gc{W\_Max}(S,f)$.
\end{proof}
}

\section{Conclusion}

In this manuscript, we have shown that \gc{W\_Among}$^{var}$,
\gc{W\_Regular}$^{var}$ ,\gc{W\_Max} and \gc{W\_Min}, are polynomially
DAG-decomposable. We also give the respective polytime dynamic programming
based algorithms to compute the minimum of this class of global cost
functions.